\documentclass[twoside]{article}

\usepackage[preprint]{aistats2027}
\usepackage{float}

\usepackage[round]{natbib}

\let\cite\citep

\def\safedef#1{%
   \ifx#1\undefined
      \expandafter\def\expandafter#1%
   \else
      \errmessage{The \string#1 is defined already}%
      \expandafter\def\expandafter\tmp
   \fi
}

\usepackage{amsthm,amsmath,bbm,amsfonts,amssymb}
\usepackage{dsfont} 
\usepackage{mathtools}
\usepackage{commath}
\usepackage{eqnarray}
\mathtoolsset{showonlyrefs=false}
\allowdisplaybreaks 

\usepackage{xspace}
\usepackage[T1]{fontenc}
\usepackage[utf8]{inputenc}
\usepackage[usenames,dvipsnames]{xcolor} 

\usepackage{hyperref,url}

\usepackage{wrapfig}
\usepackage[export]{adjustbox}
\usepackage{graphicx,tabularx}
\usepackage{multirow,hhline}
\usepackage{booktabs}
\usepackage{pbox} 
\usepackage{algorithm,algorithmic}

\usepackage[export]{adjustbox}

\usepackage{enumitem}
\usepackage[normalem]{ulem}

{\color{kjsaved}}%

\usepackage{framed}
\usepackage[most]{tcolorbox}
\definecolor{kjgray}{rgb}{.7,.7,.7}

\newtheoremstyle{kjstyle}
{1ex} 
{\topsep} 
{\itshape} 
{} 
{\bfseries} 
{.} 
{.5em} 
{} 

\newtheoremstyle{kjstyle2}
{.0em} 
{.0em} 
{\itshape} 
{} 
{\bfseries} 
{.} 
{.5em} 
{} 

\newtheoremstyle{kjstylenoitalic}
{1ex} 
{\topsep} 
{} 
{} 
{\bfseries} 
{.} 
{.5em} 
{} 

\tcbset{kjboxstyle/.style={title={},breakable,colback=white,enhanced jigsaw,boxrule=1.3pt,sharp corners,colframe=kjgray,boxsep=0pt,coltitle={black},attach title to upper={},left=.8ex,bottom=.4em}}    

\tcbset{kjboxstylec/.style={title={},breakable,colback=white,enhanced jigsaw,boxrule=1.3pt,sharp corners,colframe=kjgray,boxsep=0pt,coltitle={black},attach title to upper={},before skip=1.5ex,left=.8ex,bottom=.4em,top=0.4em,enlarge top by=-0.3em,enlarge bottom by=-0.3em}}    

\newtheorem{theorem}{Theorem}

\theoremstyle{kjstyle}
\AtBeginEnvironment{thm}{\begin{tcolorbox}[kjboxstyle]}%
  \AtEndEnvironment{thm}{\end{tcolorbox}}%
\theoremstyle{kjstyle}
\AtBeginEnvironment{lem}{\begin{tcolorbox}[kjboxstyle]}%
  \AtEndEnvironment{lem}{\end{tcolorbox}}%
\theoremstyle{kjstyle}
\AtBeginEnvironment{prop}{\begin{tcolorbox}[kjboxstyle]}%
  \AtEndEnvironment{prop}{\end{tcolorbox}}%
\theoremstyle{kjstyle}
\AtBeginEnvironment{cor}{\begin{tcolorbox}[kjboxstyle]}%
  \AtEndEnvironment{cor}{\end{tcolorbox}}%

\theoremstyle{kjstyle}
\AtBeginEnvironment{defn}{\begin{tcolorbox}[kjboxstyle]}%
  \AtEndEnvironment{defn}{\end{tcolorbox}}%

\theoremstyle{kjstyle}
\AtBeginEnvironment{ass}{\begin{tcolorbox}[kjboxstyle]}%
  \AtEndEnvironment{ass}{\end{tcolorbox}}%

\theoremstyle{kjstyle}
\AtBeginEnvironment{conj}{\begin{tcolorbox}[kjboxstyle]}%
  \AtEndEnvironment{conj}{\end{tcolorbox}}%

\theoremstyle{kjstyle}
\AtBeginEnvironment{question}{\begin{tcolorbox}[kjboxstyle]}%
  \AtEndEnvironment{question}{\end{tcolorbox}}%

\theoremstyle{kjstylenoitalic}
\AtBeginEnvironment{remark}{\begin{tcolorbox}[kjboxstyle]}%
  \AtEndEnvironment{remark}{\end{tcolorbox}}%
\theoremstyle{kjstylenoitalic}
\AtBeginEnvironment{openp}{\begin{tcolorbox}[kjboxstyle]}%
  \AtEndEnvironment{openp}{\end{tcolorbox}}%

\theoremstyle{kjstylenoitalic}
\AtBeginEnvironment{ex}{\begin{tcolorbox}[kjboxstyle]}%
  \AtEndEnvironment{ex}{\end{tcolorbox}}%

\theoremstyle{kjstylenoitalic}

\usepackage{mdframed}
\usepackage{lipsum}
\definecolor{kjgray}{rgb}{.7,.7,.7}
\makeatletter



\makeatletter
\renewcommand{\paragraph}{%
  \@startsection{paragraph}{4}%
  {\z@}{0.50ex \@plus 1ex \@minus .2ex}{-1em}%
  {\normalfont\normalsize\bfseries}%
}
\makeatother

\usepackage[customcolors,shade]{hf-tikz} 

\usepackage{tabularx} 
\newcolumntype{P}[1]{>{\centering\arraybackslash}p{#1}}
\newcolumntype{M}[1]{>{\centering\arraybackslash}m{#1}}

\def\ddefloop#1{\ifx\ddefloop#1\else\ddef{#1}\expandafter\ddefloop\fi}

\def\ddef#1{\expandafter\def\csname #1#1\endcsname{\ensuremath{\mathbb{#1}}}}
\ddefloop ABCDFGHIJKLMNORSTUWXYZ\ddefloop 

\def\ddef#1{\expandafter\def\csname c#1\endcsname{\ensuremath{\mathcal{#1}}}}
\ddefloop ABCDEFGHIJKLMNOPQRSTUVWXYZ\ddefloop

\def\ddef#1{\expandafter\def\csname b#1\endcsname{\ensuremath{{\mathbf{#1}}}}}
\ddefloop ABCDEFGHIJKLMNOPQRSTUVWXYZ\ddefloop  
\def\ddef#1{\expandafter\def\csname b#1\endcsname{\ensuremath{{\boldsymbol{#1}}}}}
\ddefloop abcdeghijklmnopqrtsuvwxyz\ddefloop  

\def\ddef#1{\expandafter\def\csname h#1\endcsname{\ensuremath{\hat{#1}}}}
\ddefloop ABCDEFGHIJKLMNOPQRSTUVWXYZabcdefghijklmnopqrsuvwxyz\ddefloop 
\def\ddef#1{\expandafter\def\csname hc#1\endcsname{\ensuremath{\hat{\mathcal{#1}}}}}
\ddefloop ABCDEFGHIJKLMNOPQRSTUVWXYZ\ddefloop
\def\ddef#1{\expandafter\def\csname hb#1\endcsname{\ensuremath{\hat{\mathbf{#1}}}}}
\ddefloop ABCDEFGHIJKLMNOPQRSTUVWXYZ\ddefloop %
\def\ddef#1{\expandafter\def\csname hb#1\endcsname{\ensuremath{\hat{\boldsymbol{#1}}}}}
\ddefloop abcdefghijklmnopqrstuvwxyz\ddefloop %

\def\ddef#1{\expandafter\def\csname t#1\endcsname{\ensuremath{\tilde{#1}}}}
\ddefloop ABCDEFGHIJKLMNOPQRSTUVWXYZabcdefgijklmnpqtsuvwxyz\ddefloop 
\def\ddef#1{\expandafter\def\csname tc#1\endcsname{\ensuremath{\tilde{\mathcal{#1}}}}}
\ddefloop ABCDEFGHIJKLMNOPQRSTUVWXYZ\ddefloop
\def\ddef#1{\expandafter\def\csname tb#1\endcsname{\ensuremath{\tilde{\mathbf{#1}}}}}
\ddefloop ABCDEFGHIJKLMNOPQRSTUVWXYZ\ddefloop
\def\ddef#1{\expandafter\def\csname tb#1\endcsname{\ensuremath{\tilde{\boldsymbol{#1}}}}}
\ddefloop abcdefghijklmnopqrstuvwxyz\ddefloop %

\def\ddef#1{\expandafter\def\csname bar#1\endcsname{\ensuremath{\bar{#1}}}}
\ddefloop ABCDEFGHIJKLMNOPQRSTUVWXYZabcdefghijklmnopqrtsuvwxyz\ddefloop
\def\ddef#1{\expandafter\def\csname barc#1\endcsname{\ensuremath{\bar{\mathcal{#1}}}}}
\ddefloop ABCDEFGHIJKLMNOPQRSTUVWXYZ\ddefloop
\def\ddef#1{\expandafter\def\csname barb#1\endcsname{\ensuremath{\bar{\mathbf{#1}}}}}
\ddefloop ABCDEFGHIJKLMNOPQRSTUVWXYZ\ddefloop
\def\ddef#1{\expandafter\def\csname barb#1\endcsname{\ensuremath{\bar{\boldsymbol{#1}}}}}
\ddefloop abcdefghijklmnopqrstuvwxyz\ddefloop %

\def\ddef#1{\expandafter\def\csname war#1\endcsname{\ensuremath{\overline{#1}}}}
\ddefloop ABCDEFGHIJKLMNOPQRSTUVWXYZabcdefghijklmnopqrtsuvwxyz\ddefloop
\def\ddef#1{\expandafter\def\csname warc#1\endcsname{\ensuremath{\overline{\mathcal{#1}}}}}
\ddefloop ABCDEFGHIJKLMNOPQRSTUVWXYZ\ddefloop
\def\ddef#1{\expandafter\def\csname warb#1\endcsname{\ensuremath{\overline{\mathbf{#1}}}}}
\ddefloop ABCDEFGHIJKLMNOPQRSTUVWXYZ\ddefloop
\def\ddef#1{\expandafter\def\csname warb#1\endcsname{\ensuremath{\overline{\boldsymbol{#1}}}}}
\ddefloop abcdefghijklmnopqrstuvwxyz\ddefloop %

\def\epsilon{\varepsilon}

\usepackage{pgffor}
\def\greeksymbols{alpha,beta,gamma,gam,delta,dt,eps,epsilon,zeta,eta,theta,th,iota,kappa,kap,lambda,lam,mu,nu,xi,pi,rho,sigma,sig,tau,phi,chi,psi,omega,om,Gamma,Gam,Delta,Dt,Theta,Th,Lambda,Lam,Pi,Sigma,Sig,Phi,Psi,Omega,Om}
\def\greeksymbolsnoeta{alpha,beta,gamma,gam,delta,dt,eps,epsilon,zeta,theta,th,iota,kappa,kap,lambda,lam,mu,nu,xi,pi,rho,sigma,sig,tau,phi,chi,psi,omega,om,Gamma,Gam,Delta,Dt,Theta,Th,Lambda,Lam,Pi,Sigma,Sig,Phi,Psi,Omega,Om} 

\foreach \x in \greeksymbolsnoeta{\expandafter\xdef\csname b\x\endcsname{\noexpand\ensuremath{\noexpand\boldsymbol{\csname \x\endcsname}}}}

\foreach \x in \greeksymbols{\expandafter\xdef\csname h\x\endcsname{\noexpand\ensuremath{\noexpand\hat{\csname \x\endcsname}}}}
\foreach \x in \greeksymbolsnoeta{\expandafter\xdef\csname hb\x\endcsname{\noexpand\ensuremath{\noexpand\hat{\noexpand\boldsymbol{ \csname \x\endcsname}}}}}

\foreach \x in \greeksymbols{\expandafter\xdef\csname bar\x\endcsname{\noexpand\ensuremath{\noexpand\bar{\csname \x\endcsname}}}}
\foreach \x in \greeksymbolsnoeta{%
\expandafter\xdef\csname barb\x\endcsname{\noexpand\ensuremath{\noexpand\bar{\noexpand\boldsymbol{ \csname \x\endcsname}}}}
}

\foreach \x in \greeksymbols{\expandafter\xdef\csname t\x\endcsname{\noexpand\ensuremath{\noexpand\tilde{\csname \x\endcsname}}}}
\foreach \x in \greeksymbolsnoeta{\expandafter\xdef\csname tb\x\endcsname{\noexpand\ensuremath{\noexpand\tilde{\noexpand\boldsymbol{ \csname \x\endcsname}}}}}

\providecommand{\normz}[2][-1]{
\ensuremath{\mathinner{
\ifthenelse{\equal{#1}{-1}}{ 
\!\left\|#2\right\|}{}
\ifthenelse{\equal{#1}{0}}{ 
\|#2\|}{}
\ifthenelse{\equal{#1}{1}}{ 
\bigl\|#2\bigr\|}{}
\ifthenelse{\equal{#1}{2}}{ 
\Bigl\|#2\Bigr\|}{}
\ifthenelse{\equal{#1}{3}}{ 
\biggl\|#2\biggr\|}{}
\ifthenelse{\equal{#1}{4}}{ 
\Biggl\|#2\Biggr\|}{}
}} 
}  

\providecommand{\floor}[2][-1]{
\ensuremath{\mathinner{
\ifthenelse{\equal{#1}{-1}}{ 
\!\left\lfloor#2\right\rfloor}{}
\ifthenelse{\equal{#1}{0}}{ 
\lfloor#2\rfloor}{}
\ifthenelse{\equal{#1}{1}}{ 
\!\bigl\lfloor#2\bigr\rfloor}{}
\ifthenelse{\equal{#1}{2}}{ 
\!\Bigl\lfloor#2\Bigr\rfloor}{}
\ifthenelse{\equal{#1}{3}}{ 
\!\biggl\lfloor#2\biggr\rfloor}{}
\ifthenelse{\equal{#1}{4}}{ 
\!\Biggl\lfloor#2\Biggr\rfloor}{}
}} 
}

\providecommand{\ceil}[2][-1]{
\ensuremath{\mathinner{
\ifthenelse{\equal{#1}{-1}}{ 
\!\left\lceil#2\right\rceil}{}
\ifthenelse{\equal{#1}{0}}{ 
\lceil#2\rceil}{}
\ifthenelse{\equal{#1}{1}}{ 
\!\bigl\lceil#2\bigr\rceil}{}
\ifthenelse{\equal{#1}{2}}{ 
\!\Bigl\lceil#2\Bigr\rceil}{}
\ifthenelse{\equal{#1}{3}}{ 
\!\biggl\lceil#2\biggr\rceil}{}
\ifthenelse{\equal{#1}{4}}{ 
\!\Biggl\lceil#2\Biggr\rceil}{}
}} 
}

\definecolor{mygrn}{rgb}{0,.8,0}
\definecolor{myred}{rgb}{.8,0,0}

\DeclareMathOperator{\PP}{\mathbb{P}}

\DeclareMathOperator*{\argmax}{arg~max}

\DeclarePairedDelimiterX{\inp}[2]{\langle}{\rangle}{#1, #2}

\newcommand\declareop[3]{%
  \newcommand#1{%
    \mskip\muexpr\medmuskip*#2\relax
    {#3}%
    \mskip\muexpr\medmuskip*#2\relax
}}
\declareop\capprox{1}{{\sr{\const}{\approx}}} 
\declareop\logapprox{1}{{\sr{\mathsf{log}}{\approx}}} 

\def\const{\mathsf{const}}

\usepackage{pifont}
\newcommand{\sr}{\stackrel}

\makeatletter
\newcommand{\vast}{\bBigg@{3}}
\newcommand{\Vast}{\bBigg@{4}}
\makeatother

\newenvironment{talign*}
 {\csname align*\endcsname}
 {\endalign}

\def\chrulefill{\leavevmode\leaders\hrule height 0.7ex depth \dimexpr0.4pt-0.7ex\hfill\kern0pt}

\newcommand{\kj}[1]{{\color{RedOrange}[KJ: #1]}}
\newcommand{\guide}[1]{{\color{Violet}[#1]~}}
{\color{kjsavedrevised}}%

\newcommand{\stkout}[1]{\ifmmode\text{\sout{\ensuremath{#1}}}\else\sout{#1}\fi}

\usepackage[toc,page,header]{appendix}
\usepackage{minitoc}
\usepackage{silence}
\theoremstyle{definition}

\newtheorem{corollary}{Corollary}
\newtheorem{proposition}{Proposition}
\newtheorem{lemma}{Lemma}

\newif\ifFINAL

\FINALtrue 

\ifFINAL
   
  \def\guide#1{}
  \def\kj#1{} 
  \def\kjnew#1{}
\else
  \usepackage{transparent}
  \usepackage[inline]{showlabels}
  \usepackage{rotating}
  \renewcommand{\showlabelfont}%
  {\transparent{0.8}\scriptsize\bf\slshape\color{Lavender}}
\fi

\begin{document}

\doparttoc 
\faketableofcontents 

\setlength{\abovedisplayskip}{4pt}%
\setlength{\belowdisplayskip}{4pt}%
\setlength{\abovedisplayshortskip}{4pt}%
\setlength{\belowdisplayshortskip}{4pt}


%

%

\runningauthor{Li et al.}
\twocolumn[

\aistatstitle{Halve-Then-Greedy: Adaptive Sampling for Posterior Success \\ with Near-Optimal Instance Dependence}



\aistatsauthor{
Yinan Li
\And
Kwang-Sung Jun
\And
Kapilan Balagopalan\textsuperscript{$\dagger$}
}

\aistatsaddress{
University of Arizona
\And
POSTECH CSE/GSAI
\And
University of Arizona
}

\aistatsauthor{
Tuan Nguyen\textsuperscript{$\dagger$} 
\And
Yao Zhao\textsuperscript{$\dagger$}
\And
Abbas Zaidi
\And
Houssam Nassif}

\aistatsaddress{University of Arizona
\And
University of Arizona
\And
Meta Inc.
\And
Meta Inc.}

\begin{center}
\small
$^\dagger$ Authors 3--5 are listed in alphabetical order.\par
\vspace{3pt}
Corresponding author: Yinan Li (\texttt{yinanli@arizona.edu}).
\end{center}
]

\begin{abstract}
  We study fixed-budget adaptive experimentation aimed at maximizing the
probability of obtaining sufficient posterior evidence that a recommended design's mean performance exceeds a target. Unlike best-arm identification, this objective does not require finding the best design.
For Gaussian bandits with a Gaussian prior, we propose \emph{Halve-Then-Greedy} (HTG), which combines a large initial balanced
allocation, successive halving, and posterior-greedy sampling.
We prove instance-dependent, finite-budget exponential bounds on
posterior-failure probability that capture both separation from unfavorable arms and the presence of multiple favorable arms.
On one-positive-arm families, we establish an
asymptotic sample-complexity lower bound that HTG matches up to a logarithmic factor. We further analyze uniform sampling and sequential halving, revealing regimes where they require more samples than HTG by linear and logarithmic factors, respectively.
Simulations illustrate HTG's sample-efficiency advantages and their dependence
on instance structure.
\end{abstract}

\section{Introduction}
\label{sec:introduction}


We consider an experiment comparing several candidate designs under a limited sampling budget. Instead of identifying the design with the largest mean performance, the
objective is to identify a design for which there is sufficient evidence that its mean performance exceeds a prespecified target. Posterior-probability criteria
provide a natural way to formalize this evidence requirement, and Bayesian
experimental design often evaluates how likely such criteria are to be
satisfied across repeated experiments \citep{hagar2025economical}.
We say that \emph{posterior success} occurs if, at the end of the fixed
budget, the posterior probability that the recommended design's mean
performance exceeds the target is at least a prescribed credibility level.
Our goal is to allocate observations adaptively to make posterior success
likely. We model this problem as a Gaussian multi-armed bandit with a Gaussian
prior; the formal setup is given in Section~\ref{sec:problem}.


This objective differs from several standard frequentist bandit
pure-exploration objectives. Best-arm identification (BAI) seeks to identify
an arm with the largest true mean under a fixed sampling budget
(\emph{fixed-budget} setting) or a prescribed identification-error tolerance
(\emph{fixed-confidence} setting)
\citep{audibert10best,karnin13almost,garivier16optimal,kaufmann16onthecomplexity}.
Thresholding bandits aim to classify arms according to their means relative
to a prescribed threshold \citep{locatelli16anoptimal}.
Good-arm identification seeks to identify one or more arms whose true means
meet a prescribed performance threshold, without requiring them to be optimal
\citep{kano2019good,katzsamuels20true,jourdan2026anytime}.
Among these, good-arm identification is the closest frequentist formulation
to our problem. The key difference is the success criterion: good-arm
identification defines correctness in terms of the identified arms' true
means, whereas posterior success requires the terminal posterior probability
that the recommended arm's mean exceeds the target to be at least a
prescribed credibility level.

Bayesian best-arm identification has been studied from several perspectives.
\citet{russo2020simple} analyzes the asymptotic rate at which the posterior
becomes confident about the identity of the true best arm, while
\citet{shang20fixed} develop fixed-confidence stopping and sample-complexity
guarantees for Bayesian sampling rules.
In the fixed-budget setting, prior work studies the probability of
misidentifying the best arm after a prescribed number of samples, with
performance averaged over the prior distribution on bandit instances
\citep{atsidakou2022bayesian,pmlr-v258-nguyen25a}.

More closely related to our threshold-based objective is the Bayesian
satisficing formulation of \citet{reverdy2016satisficing}, which considers
whether the current belief assigns sufficiently high probability to an
arm's mean exceeding a prescribed threshold, but evaluates performance
through cumulative expected satisficing regret.
A related exploration--validation approach uses adaptive sampling to select
a treatment and then conducts a separate A/B test to validate the selected
treatment \citep{pmlr-v300-zhang26c}.
Our objective instead concerns the probability of obtaining a posterior
certificate at the end of a fixed budget, with all collected observations
contributing to the terminal posterior and no separate validation test.


Our posterior success problem has two complementary sources of favorable
instance structure. Strongly unfavorable arms can be screened with relatively
few observations, allowing sampling to concentrate on more promising
candidates, and the presence of multiple favorable arms provides multiple
opportunities to obtain the required certificate. Motivated by these two
features, we propose \emph{Halve-Then-Greedy} (HTG), which combines a large
initial balanced allocation with successive halving and a final
posterior-greedy phase.

We establish instance-dependent, finite-budget exponential bounds on HTG's
posterior-failure probability that capture both sources of favorable
instance structure (Section~\ref{sec:guarantees}).
For every fixed profile of one-positive-arm instances, we derive an
asymptotic sample-complexity lower bound and show that HTG matches it up to a
logarithmic factor in the number of arms
(Section~\ref{sec:lower-bounds}).
Finally, we compare HTG theoretically and empirically with uniform sampling
and sequential halving
(Sections~\ref{sec:connections}--\ref{sec:experiments}).


\vspace{-0.5em}

\section{Problem Formulation}
\label{sec:problem}

We consider $K$ arms with unknown means
$\mu=(\mu_1,\ldots,\mu_K)\in\mathbb{R}^K$.
Before observing any data, the algorithm has the prior belief that the arm means
$\mu_1,\ldots,\mu_K$ are independently drawn from a common Gaussian distribution,
\begin{equation}
    \mu_i \sim \mathcal{N}(\mu_0,\sigma_0^2),
    \qquad i\in[K].
    \label{eq:prior}
\end{equation}
Conditional on the true mean vector $\mu$, observations are independent across
arms and samples, with
\begin{equation}
    X_{i,s}\sim\mathcal{N}(\mu_i,\sigma^2),
    \qquad i\in[K],\ s\ge1,
    \label{eq:observation-model}
\end{equation}
where the common variance $\sigma^2>0$ is known. Although the algorithm uses
the Gaussian prior to form its posterior beliefs, our performance guarantees
treat the true mean vector $\mu$ as fixed and evaluate probability only over
the observation noise and any randomization of the algorithm.

We take the target threshold to be zero without loss of generality. A general
threshold $\tau$ can be reduced to this setting by replacing
$X_{i,s}$, $\mu_i$, and $\mu_0$ with
$X_{i,s}-\tau$, $\mu_i-\tau$, and $\mu_0-\tau$, respectively.

An algorithm $\mathcal{A}$ sequentially allocates a fixed budget of $T$
observations across the $K$ arms and, after the final observation, outputs an
arm $J_T\in[K]$. Let $\mathcal{F}_T$ denote the information available at time $T$.
For a prescribed posterior credibility
tolerance $\delta\in(0,1/2)$, we define the posterior-success event as
\begin{equation}
    \mathrm{Succ}_T
    :=
    \left\{
        \mathbb{P}\!\left(
            \mu_{J_T}>0
            \,\middle|\,
            \mathcal{F}_T
        \right)
        \ge 1-\delta
    \right\},
    \label{eq:success-event}
\end{equation}
where the conditional probability is the posterior probability induced by the
Gaussian prior and observation model. Thus, success means that the posterior
probability that the output arm has mean above the target threshold is at least
$1-\delta$; it does not necessarily require that its true mean be positive.

For a fixed mean vector $\mu$, we evaluate an algorithm $\mathcal{A}$ through its
finite-budget failure probability
\[
    \mathbb{P}_{\mu}^{\mathcal{A}}
    \bigl(\mathrm{Succ}_T^c\bigr),
\]
where $\mathbb{P}_{\mu}^{\mathcal{A}}$ denotes probability under the fixed
instance $\mu$, accounting for both observation noise and any randomization of
the algorithm. Our goal is to design adaptive allocation strategies for which
this failure probability decreases rapidly with the budget $T$, with rates that
adapt to the difficulty of the underlying instance. 

\vspace{-0.5em}

\section{Halve-Then-Greedy}
\label{sec:htg-algorithm}



We first introduce the posterior quantities used by HTG. Let $\Phi$ denote
the standard normal cumulative distribution function, and define
$\lambda:=\sigma^2/\sigma_0^2$ and $z:=\Phi^{-1}(1-\delta)$. Write
$\overline X_{i,n}:=n^{-1}\sum_{s=1}^n X_{i,s}$. By Gaussian conjugacy,
after $n$ observations from arm $i$,
\[
    \mu_i\mid X_{i,1:n}
    \sim
    \cN(\widehat\mu_{i,n},\widehat\sigma_n^2),
\]
where
\[
    \widehat\mu_{i,n}
    =\frac{\lambda\mu_0+n\overline X_{i,n}}{n+\lambda},
    \qquad
    \widehat\sigma_n=\frac{\sigma}{\sqrt{n+\lambda}}.
\]
The corresponding posterior $Z$-score is
\[
    Z_{i,n}
    :=\frac{\widehat\mu_{i,n}}{\widehat\sigma_n}
    =\frac{\lambda\mu_0+n\overline X_{i,n}}
           {\sigma\sqrt{n+\lambda}},
\]
so that $\PP(\mu_i>0\mid X_{i,1:n})=\Phi(Z_{i,n})$.
A derivation of these posterior quantities is given in
Appendix~\ref{app:posterior}.

Let $N_i(t)$ denote the number of observations collected from arm $i$
after $t$ total observations. Define
$Q:=1+\ceil{\log_2K}$,
$n_\lambda:=\max\{1,\ceil{\lambda}\}$, and
$T_{\mathrm{cut}}:=3K\max\{n_\lambda,Q\}$.
HTG is given in Algorithm~\ref{alg:htg}. We use an arbitrary fixed deterministic rule to break ties.

\begin{algorithm}[H]
\caption{Halve-Then-Greedy (HTG)}
\label{alg:htg}
\begin{algorithmic}[1]
\REQUIRE Budget $T\ge K$

\IF{$T<T_{\mathrm{cut}}$}
    \STATE Sample arms round-robin until $T$ observations are collected.
\ELSE
    \STATE Set $S_1=[K]$ and $t=0$.
    \FOR{$r=1,\ldots,Q$}
        \STATE Set $k_r=|S_r|$ and
        \[
        q_r=
        \begin{cases}
        \floor{T/(3K)}, & r=1,\\
        \floor{T/(3Qk_r)}, & r\ge2.
        \end{cases}
        \]
        \STATE Sample each $i\in S_r$ an additional $q_r$ times.
        \STATE Update $t\leftarrow t+k_rq_r$.
        \IF{$r<Q$}
            \STATE Retain the $\ceil{k_r/2}$ arms in $S_r$
            with largest $Z_{i,N_i(t)}$; call the set $S_{r+1}$.
        \ENDIF
    \ENDFOR

    \WHILE{$t<T$}
        \STATE Choose
        $A_{t+1}\in
        \argmax_{i\in[K]} Z_{i,N_i(t)}$.
        \STATE Observe arm $A_{t+1}$ and set $t\leftarrow t+1$.
    \ENDWHILE
\ENDIF

\STATE \textbf{Output}
$J_T\in\argmax_{i\in[K]} Z_{i,N_i(T)}$. 
\label{line:recommendation}
\end{algorithmic}
\end{algorithm}

HTG consists of two allocation phases. The first phase uses successive
halving to progressively screen out less promising arms, with a deliberately
large initial batch. The second phase allocates the remaining budget greedily
according to each arm's current posterior probability of having positive mean.
Throughout both phases, the posterior distribution of each arm mean is updated
using all observations collected from that arm.

The active sets $S_r$ are used only during the halving phase. An arm removed
from an active set remains eligible for both the subsequent greedy phase and
the final recommendation, and all observations previously collected from that
arm continue to be used in its posterior. 
At the end
of the budget, HTG recommends the arm with the largest posterior $Z$-score
(line~\ref{line:recommendation} of Algorithm~\ref{alg:htg}).

Since $\PP(\mu_i>0\mid X_{i,1:n})=\Phi(Z_{i,n})$ and $\Phi$ is increasing,
this recommendation maximizes the posterior probability that the selected
arm has positive mean. Moreover, since $z=\Phi^{-1}(1-\delta)$, the success
event in~\eqref{eq:success-event} can be written as
\begin{equation}
    \mathrm{Succ}_T
    =
    \left\{
        Z_{J_T,N_{J_T}(T)}\ge z
    \right\}
    =
    \left\{
        \max_{i\in[K]} Z_{i,N_i(T)}\ge z
    \right\}.
    \label{eq:success-zmax}
\end{equation}

\vspace{-0.5em}


\textbf{Why this structure?}
HTG combines three complementary design choices. First, the initial halving stage samples every arm $q_1=\left\lfloor\frac{T}{3K}\right\rfloor$
times. This substantial balanced allocation provides broad coverage of the
arms and ensures that potentially favorable arms receive a nontrivial amount
of evidence before the adaptive allocation begins. As we show in
Section~\ref{sec:guarantees}, this feature allows HTG to benefit when several
arms have positive means.

Second, the subsequent halving stages direct observations toward increasingly
promising arms instead of continuing to allocate uniformly to clearly
unfavorable competitors. At the end of each halving stage, all arms in the
active set have the same total number of observations. Since the prior
parameters are common across arms, ranking these arms by their posterior
$Z$-scores is therefore equivalent to ranking them by their empirical means.
Thus, no separate statistics or posterior states are required for the
halving phase.

Finally, HTG reserves a constant fraction of the budget for posterior-greedy
sampling. The first halving batch uses at most $T/3$ observations, while each
of the remaining $Q-1$ stages uses at most $T/(3Q)$. Hence the total number
of observations used by the halving phase is at most 
$\frac{T}{3}
    +
    \frac{Q-1}{3Q}T
    <
    \frac{2T}{3}$, 
so at least $T/3$ observations remain for the greedy phase. These observations
are allocated to the arm with the largest current posterior probability of
having positive mean, and all $K$ original arms are eligible. This allows the
procedure to concentrate evidence after the initial screening phase rather
than spending the entire budget across elimination stages.

These three components play complementary roles: the initial balanced
allocation allows HTG to exploit the presence of multiple positive arms,
halving progressively removes clearly unfavorable arms from the active set,
and greedy sampling concentrates evidence on promising arms. Together, they
allow HTG to adapt both to instances with many positive arms and to instances
where unfavorable arms can be separated quickly. The next section quantifies this adaptivity
through instance-dependent finite-budget guarantees.

\section{Instance-Dependent Guarantees}
\label{sec:guarantees}

We now establish finite-budget guarantees for HTG. Throughout this paper,
assume that
\[
    \mu_\star:=\max_{i\in[K]}\mu_i>0.
\]
The bounds reveal two complementary ways in which an instance can be favorable.
First, arms with means far below the best positive mean require relatively few samples to reveal that they are substantially suboptimal, allowing the algorithm to shift its sampling effort elsewhere. Second, when several arms have positive means, there are several candidates that can yield the required posterior certificate, since posterior success requires sufficient evidence for only one of them.
We first introduce
complexities that capture these two forms of favorable instance structure and
then relate them to the quantity appearing in the HTG guarantee.

\subsection{Certification complexities}
\label{sec:certification-complexities}

For $x\in\mathbb R$, write $x_+:=\max\{x,0\}$. We begin with two
interpretable instance-dependent quantities. Define the \emph{capped-gap
complexity}
\begin{equation}
    H_{\mathrm{cap}}(\mu)
    :=
    \sum_{i=1}^{K}
    \frac{\sigma^2}
    {\max\{\mu_\star,\mu_\star-\mu_i\}^2},
    \label{eq:H-cap}
\end{equation}
and the \emph{positive-arm complexity}
\begin{equation}
    H_{\mathrm{pos}}(\mu)
    :=
    \frac{K\sigma^2}
    {\sum_{i=1}^{K}(\mu_i)_+^2}.
    \label{eq:H-pos}
\end{equation}


To see the meaning of the capped-gap complexity, let
$\Delta_i:=\mu_\star-\mu_i$ denote the usual gap between arm $i$ and the best
mean. Then $H_{\mathrm{cap}}(\mu)
    =
    \sum_{i=1}^{K}
    \frac{\sigma^2}{\max\{\mu_\star,\Delta_i\}^2}$.
Thus, the effective gap associated with arm $i$ is
$\gamma_i:=\max\{\mu_\star,\Delta_i\}$, meaning that the usual best-arm gap is capped
from below by $\mu_\star$. 
This is the reason for the name \emph{capped-gap complexity}. 


We use these to define the \emph{certification complexity}
\begin{equation}
    H_{\mathrm{cert}}(\mu)
    :=
    \min\bigl\{
        H_{\mathrm{cap}}(\mu),
        H_{\mathrm{pos}}(\mu)
    \bigr\}.
    \label{eq:H-cert}
\end{equation}

The two quantities capture different sources of favorable instance structure.
For an arm with $\mu_i\ge0$, we have $\Delta_i\le\mu_\star$, so its
contribution to $H_{\mathrm{cap}}$ is $\sigma^2/\mu_\star^2$. For a negative
arm, $\Delta_i=\mu_\star+|\mu_i|>\mu_\star$, and its contribution is
$\frac{\sigma^2}{(\mu_\star+|\mu_i|)^2}$.
Thus, arms far below the best positive mean incur a smaller cost: relatively
few samples are needed to reveal that they are substantially suboptimal. In
particular, an arm lying close to the target threshold does not incur the
potentially much larger cost $\sigma^2/\mu_i^2$. In this sense,
$H_{\mathrm{cap}}$ captures the difficulty of distinguishing promising arms
from suboptimal ones.

In contrast, $H_{\mathrm{pos}}$ captures the benefit of having multiple
positive arms. Since posterior success requires sufficient posterior evidence
for only one arm, having several positive arms provides several candidates
that can yield the required posterior certificate. For example, if exactly
$m$ arms have common positive mean $\Delta$ and the remaining arms are
nonpositive, then 
$H_{\mathrm{pos}}
=
\frac{K\sigma^2}{m\Delta^2}$. 
Hence, in this special case, $H_{\mathrm{pos}}$ decreases as the number of positive arms increases.

Neither $H_{\mathrm{cap}}$ nor $H_{\mathrm{pos}}$ uniformly dominates the
other. Their minimum in~\eqref{eq:H-cert} captures the two distinct sources
of favorable instance structure described above.
Section~\ref{sec:htg-interpretation} illustrates how these affect the
finite-budget guarantee for HTG.


The finite-budget analysis of HTG involves a closely related complexity.
Let
\[
    \mu_{(1)}\ge\mu_{(2)}\ge\cdots\ge\mu_{(K)}
\]
denote the ordered arm means, and recall
$Q=1+\lceil\log_2 K\rceil$ from Section~\ref{sec:htg-algorithm}. Define
\begin{equation}
\label{eq:H-HTG}
H_{\mathrm{HTG}}(\mu)
:=
\min\left\{
\begin{aligned}
&\sum_{i=1}^{K}\frac{\sigma^2}{\mu_i^2},
\qquad
\frac{K\sigma^2}{\sum_{i=1}^{K}(\mu_i)_+^2},
\\
&Q\max_{1\le r\le K}
\frac{r\sigma^2}
{\max\{\mu_\star,\mu_\star-\mu_{(r)}\}^2}
\end{aligned}
\right\}.
\end{equation}
Here and throughout, we use the convention
$\sigma^2/0^2:=+\infty$.
The three terms in~\eqref{eq:H-HTG} correspond to complementary guarantees
provided by the three structural components of HTG. 
The first is supplied by the final posterior-greedy phase.
The second term is exactly $H_{\mathrm{pos}}$ and captures the benefit of having several positive arms, since HTG allocates a nontrivial number of samples $q_1=\left\lfloor\frac{T}{3K}\right\rfloor$ to every arm before concentrating its sampling effort. The third term comes from the
successive-halving phase and adapts to the gaps of unfavorable arms from
the best positive mean.


Although $H_{\mathrm{HTG}}$ is the quantity that appears directly in our
finite-budget theorem, it is within a factor $Q$ of the simpler certification
complexity $H_{\mathrm{cert}}$. 


\begin{lemma}[Comparison of complexities]
\label{lem:complexity-comparison}
For every mean vector $\mu$ satisfying $\mu_\star>0$,
\begin{equation}
        H_{\mathrm{cert}}(\mu)
        \le
        H_{\mathrm{HTG}}(\mu)
        \le
        Q\,H_{\mathrm{cert}}(\mu).
    \label{eq:complexity-comparison}
\end{equation}
\end{lemma}

The proof is deferred to
Appendix~\ref{app:complexity-comparison}. 

\subsection{Finite-budget posterior certification}
\label{sec:finite-budget-guarantee}

We now state the main finite-budget guarantee for HTG. Recall from
Section~\ref{sec:htg-algorithm} that
$\lambda=\sigma^2/\sigma_0^2$,
$z=\Phi^{-1}(1-\delta)$, and
$T_{\mathrm{cut}}=3K\max\{n_\lambda,Q\}$.
Let
\[
    m:=\bigl|\{i\in[K]:\mu_i>0\}\bigr|,
    \qquad
    c:=\sqrt{2}\,z+\frac{\sqrt{\lambda}\,|\mu_0|}{\sigma},
\]
and define the instance-dependent offset
\begin{equation}
    T_0
    :=
    T_{\mathrm{cut}}
    +
    96m c^2 H_{\mathrm{HTG}}(\mu).
    \label{eq:T0}
\end{equation}
The quantities $m$, $H_{\mathrm{HTG}}(\mu)$, and $T_0$ are used only in the
analysis and are not inputs to HTG.

\begin{theorem}[Finite-budget posterior certification]
\label{thm:htg-main}
Under the Gaussian observation model~\eqref{eq:observation-model} with the common Gaussian prior~\eqref{eq:prior}, for every fixed mean vector $\mu$ with $\mu_\star>0$ and integer $T\ge K$, HTG satisfies
\begin{equation}
\label{eq:htg-main-bound}
    \mathbb{P}_{\mu}^{\mathrm{HTG}}
    \bigl(\mathrm{Succ}_T^c\bigr)
    \le
    \min\left\{
        1,\;
        2KT
        \exp\left(
            -\frac{(T-T_0)_+}
            {192H_{\mathrm{HTG}}(\mu)}
        \right)
    \right\}.
\end{equation}
\end{theorem}

The proof is deferred to
Appendix~\ref{app:htg-main}. 
Theorem~\ref{thm:htg-main} shows that, after the finite-budget offset $T_0$,
the failure probability decreases exponentially at a rate governed by
$H_{\mathrm{HTG}}(\mu)$. By
Lemma~\ref{lem:complexity-comparison}, this rate complexity is within a factor
$Q=O(\log K)$ of the certification complexity $H_{\mathrm{cert}}(\mu)$.
The offset $T_0$ accounts for the budget required by the initial allocation through $T_{\mathrm{cut}}$, as well as the effect of the prior mean and the additional samples needed to reach the prescribed posterior credibility level.

As an immediate consequence, for any target failure probability
$\eta\in(0,1)$,
\begin{equation}
\label{eq:htg-sufficient-budget}
    T
    \ge
    T_0
    +
    192H_{\mathrm{HTG}}(\mu)
    \log\frac{2KT}{\eta}
    \quad\Longrightarrow\quad
    \mathbb{P}_{\mu}^{\mathrm{HTG}}
    \bigl(\mathrm{Succ}_T^c\bigr)
    \le \eta.
\end{equation}
Here $\delta$ determines the posterior credibility required for success,
whereas $\eta$ is the desired probability of failing to obtain such a
certificate under the fixed instance $\mu$. A fully explicit sufficient
budget with no occurrence of $T$ inside the logarithm is given in
Appendix~\ref{app:explicit-budget}.

\subsection{Interpretation and proof ideas}
\label{sec:htg-interpretation}

We illustrate how the two sources of favorable instance structure translate
into the guarantee of Theorem~\ref{thm:htg-main}, and then summarize the
main proof ideas. 
Detailed calculations for these examples, together with an additional
example involving a competing arm whose mean is close to zero, are given
in Appendix~\ref{app:htg-examples}.


\textbf{One positive arm with other arms strongly negative.}
Consider an instance with $\mu_1=a>0$ and $\mu_2=\cdots=\mu_K=-b$,
where $b\ge a\sqrt K$. Then
\[
    H_{\mathrm{cert}}(\mu)\asymp H_{\mathrm{HTG}}(\mu)
    \asymp \frac{\sigma^2}{a^2},
    \qquad
    H_{\mathrm{pos}}(\mu)=\frac{K\sigma^2}{a^2}.
\]
More precisely, both $H_{\mathrm{cert}}(\mu)$ and
$H_{\mathrm{HTG}}(\mu)$ lie between $\sigma^2/a^2$ and
$2\sigma^2/a^2$. Thus, the complexity $H_{\mathrm{HTG}}$ appearing in
Theorem~\ref{thm:htg-main} can be a factor of order $K$ smaller than
$H_{\mathrm{pos}}$ on this family. Moreover, the first term
in~\eqref{eq:H-HTG}, which arises from the final posterior-greedy phase, is
at most $2\sigma^2/a^2$, whereas the third term, which arises from the
successive-halving phase, is $Q\sigma^2/a^2$. Hence the final
posterior-greedy phase removes the additional factor $Q$ that would arise
from the halving-based term alone.

\textbf{Many positive arms.}
At the other extreme, suppose $\mu_1=\cdots=\mu_K=a>0$. Then
\[
    H_{\mathrm{cap}}(\mu)=\frac{K\sigma^2}{a^2},
    \qquad
    H_{\mathrm{HTG}}(\mu)=H_{\mathrm{pos}}(\mu)
    =\frac{\sigma^2}{a^2}.
\]
Thus, the complexity $H_{\mathrm{HTG}}$ appearing in
Theorem~\ref{thm:htg-main} is a factor $K$ smaller than
$H_{\mathrm{cap}}$ on this family. This reflects the fact that when many arms are positive, any of them can provide the posterior certificate required
for success. Although $H_{\mathrm{HTG}}$ itself is independent of $K$ in
this example, the total sufficient budget is not: here $m=K$, so the offset is
$T_0=T_{\mathrm{cut}}+96Kc^2\sigma^2/a^2$. 



\textbf{Proof ideas.}
The proof, given in Appendix~\ref{app:htg-main}, begins with a concentration
event that holds simultaneously over all arms and sample counts, and hence
also at the adaptive sample counts generated by HTG. Failure at the end of
the budget means that no arm is certified, that is,
$Z_{i,N_i(T)}<z$ for every $i\in[K]$. On the concentration event, this
imposes an upper bound on the total number of samples collected from each
positive arm. The final posterior-greedy phase and the successive-halving
phase then give separate budget-counting arguments, yielding the first and
third terms in~\eqref{eq:H-HTG}, respectively.

The second term follows from a separate argument. The initial allocation
ensures that every arm receives at least
$q_1=\lfloor T/(3K)\rfloor$ samples. Hence, if the algorithm fails, then
for every positive arm $i$ there must exist some $n\ge q_1$ such that
$Z_{i,n}<z$---namely, $n=N_i(T)$. We bound the probability of this event
uniformly over all $n\ge q_1$ for each positive arm. Since these events
depend only on the independent observation sequences of the individual
arms, their probabilities can be multiplied across positive arms. The
resulting exponent involves
$\sum_{i:\mu_i>0}\mu_i^2/\sigma^2=K/H_{\mathrm{pos}}(\mu)$, which yields
the second term in~\eqref{eq:H-HTG}. Combining the three bounds gives
Theorem~\ref{thm:htg-main}.

\section{Lower Bounds and Near-Optimality}
\label{sec:lower-bounds}

We next establish lower bounds on the probability of posterior failure
and compare them with the guarantee of Theorem~\ref{thm:htg-main}.
We consider families of instances with exactly one positive arm,
whose location is unknown, while allowing the negative means to differ
across arms.
As we will show, on each such family, the lower-bound complexity
agrees with $H_{\mathrm{cert}}$ up to universal constants.
By Lemma~\ref{lem:complexity-comparison}, $H_{\mathrm{HTG}}$ is
within a factor $Q=O(\log K)$ of $H_{\mathrm{cert}}$.
Thus, HTG's asymptotic sample complexity is within an $O(\log K)$
factor of the lower bound on these one-positive-arm families.

\subsection{Families with one positive arm}
\label{sec:one-positive-families}

Fix $\Delta>0$ and
$\varepsilon=(\varepsilon_1,\ldots,\varepsilon_K)\in(0,\infty)^K$.
For each $a\in[K]$, define the mean vector $\mu^{(a)}$ by
\begin{equation}
    \mu_i^{(a)}
    :=
    \begin{cases}
        \Delta, & i=a,\\
        -\varepsilon_i, & i\ne a,
    \end{cases}
    \qquad i\in[K].
    \label{eq:one-positive-family}
\end{equation}
Every instance in this family has one positive arm with mean $\Delta$.
The family varies its location $a$, while each remaining arm $i$
has mean $-\varepsilon_i$. The negative means need not be equal.

Two quantities will govern the lower bound:
\begin{equation}
    h
    :=
    \frac{\sigma^2}{\Delta^2},
    \qquad
    S
    :=
    \sum_{i=1}^{K}
    \frac{\sigma^2}{(\Delta+\varepsilon_i)^2}.
    \label{eq:lower-bound-complexities}
\end{equation}
Here $h$ is the single-arm inverse-square term, corresponding to the
complexity of obtaining evidence that a mean $\Delta$ exceeds zero,
even when the positive arm's location is known. The quantity $S$
is the sum of inverse-square terms over all possible locations of
the positive arm. Each term uses the difference
$\Delta+\varepsilon_i$ between the two possible means $\Delta$ and
$-\varepsilon_i$ for arm $i$. In this sense, $S$ reflects the
complexity associated with uncertainty about which arm is positive.

For each budget $T\ge K$, we consider adaptive algorithms that
collect exactly $T$ observations and sample every arm at least once.
The algorithms may be randomized and may depend on $T$.
At each round, an algorithm selects an arm using only previously
collected observations and its internal randomization.
All algorithms use the Gaussian prior~\eqref{eq:prior}, the prescribed
credibility tolerance $\delta$, and all observations collected from
each arm when forming its posterior.

We use the posterior-maximizing terminal recommendation from
Section~\ref{sec:htg-algorithm} (line~\ref{line:recommendation} of Algorithm~\ref{alg:htg}),
\[
    J_T\in\arg\max_{i\in[K]}Z_{i,N_i(T)}.
\]
This recommendation entails no loss for the posterior-success objective.
Indeed, since the posterior probability that arm $i$ has positive mean is
$\Phi(Z_{i,N_i(T)})$ and $\Phi$ is increasing, for any realized observation
history this recommendation maximizes the posterior probability that the
selected arm has positive mean. In particular, it succeeds whenever any
terminal recommendation could succeed, as characterized by
\eqref{eq:success-zmax}.

We call any algorithm satisfying the sampling, posterior, and
terminal-recommendation conditions described above an
\emph{admissible algorithm}.
An admissible algorithm does not know
the true mean vector $\mu$ and its sampling decisions may depend only on
the observations collected so far, the known problem parameters, and
its internal randomization.

Define the minimax failure probability over this family, 
\begin{equation}
    r_T^\star(\Delta,\varepsilon)
    :=
    \inf_{\mathcal A}
    \max_{a\in[K]}
    \mathbb{P}_{\mu^{(a)}}^{\mathcal A}
    \bigl(\mathrm{Succ}_T^c\bigr),
    \label{eq:one-positive-minimax-risk}
\end{equation}
where the infimum is over admissible algorithms. In words,
$r_T^\star(\Delta,\varepsilon)$ is the smallest achievable
worst-case failure probability over the $K$ possible locations
of the positive arm, for the fixed profile $(\Delta,\varepsilon)$.

Throughout this section, the family parameters $K$, $\Delta$, and
$\varepsilon$, the observation variance $\sigma^2$, the prior parameters
$(\mu_0,\sigma_0^2)$, and the posterior credibility tolerance $\delta$
are fixed as $T\to\infty$.

\subsection{An asymptotic lower bound}
\label{sec:asymptotic-lower-bound}

We now show that both $h$ and $S$ impose necessary sample-complexity
scales for posterior success.

\begin{theorem}[Search and certification lower bound]
\label{thm:one-positive-lower-bound}
For every $\Delta>0$ and $\varepsilon\in(0,\infty)^K$, consider the
family~\eqref{eq:one-positive-family} and the admissible algorithms
described in Section~\ref{sec:one-positive-families}.
With $h$ and $S$ defined in~\eqref{eq:lower-bound-complexities},
the minimax failure probability in
\eqref{eq:one-positive-minimax-risk} satisfies
\begin{equation}
    \limsup_{T\to\infty}
    -\frac{1}{T}\log r_T^\star(\Delta,\varepsilon)
    \le
    \frac{1}{2\max\{h,S\}}.
    \label{eq:one-positive-lower-bound}
\end{equation}
\end{theorem}

The proof, including finite-budget lower bounds from which
\eqref{eq:one-positive-lower-bound} follows, is deferred to
Appendix~\ref{app:lower-bounds}.

Theorem~\ref{thm:one-positive-lower-bound} applies to every fixed
profile $(\Delta,\varepsilon)$. Here, failure refers to the event
$\mathrm{Succ}_T^c$. For any admissible algorithm, the asymptotic
exponent of its worst-case failure probability over the family
cannot exceed the bound in~\eqref{eq:one-positive-lower-bound}.
At each budget $T$, the instance attaining the worst-case failure
probability may depend on both the algorithm and $T$.

In sample-complexity terms, achieving failure probability at most
$\eta$ on every instance in the family requires
\begin{equation}
    T
    \ge
    \bigl(2-o_\eta(1)\bigr)\max\{h,S\}
    \log\frac{1}{\eta},
    \qquad \eta\downarrow0,
    \label{eq:one-positive-sample-lower-bound}
\end{equation}
where $o_\eta(1)\to0$ as $\eta\downarrow0$, with all family,
observation, and prior parameters, as well as $\delta$, held fixed.
The derivation of~\eqref{eq:one-positive-sample-lower-bound} from
Theorem~\ref{thm:one-positive-lower-bound} is given in
Appendix~\ref{app:lower-bound-sample-complexity}.

\textbf{Proof idea.}
Fix an admissible algorithm.
Under any auxiliary instance whose means are all negative, we
construct an event $E$ on which $Z_{i,n}<z$ for every arm $i$
and every $n\ge1$. Its probability is positive and does not depend
on the budget or algorithm. Since every arm is sampled at least once,
\eqref{eq:success-zmax} implies that $E$ guarantees posterior failure.

For the term $S$, take
$\nu=(-\varepsilon_1,\ldots,-\varepsilon_K)$ and consider the
observed history under $\nu$ conditional on $E$. Since the algorithm
still collects exactly $T$ observations, the conditional expected
sample counts sum to $T$. A weighted averaging argument shows that
there exists $a\in[K]$ satisfying
\[
    \mathbb E_\nu[N_a(T)\mid E]
    \le
    \frac{T}{S}\,
    \frac{\sigma^2}{(\Delta+\varepsilon_a)^2}.
\]
Passing from $\nu$ to $\mu^{(a)}$ changes only the mean of arm $a$,
from $-\varepsilon_a$ to $\Delta$.
A subsequent change-of-measure argument gives the upper bound $1/(2S)$ on the
asymptotic exponent of the worst-case posterior-failure probability
over the family.

For the term $h$, fix a location $a$ and replace its positive mean
$\Delta$ by $-\beta$, where $\beta>0$, leaving the other means
unchanged. For each fixed $\beta$, a similar change-of-measure
argument using $N_a(T)\le T$ gives, as $T\to\infty$, the upper bound
$(\Delta+\beta)^2/(2\sigma^2)$ on the failure exponent. Letting $\beta\downarrow0$ afterwards yields $1/(2h)$.
Combining the bounds $1/(2S)$ and $1/(2h)$ gives
\eqref{eq:one-positive-lower-bound}.

\subsection{Comparison with the upper bound}
\label{sec:lower-bound-comparison}

We now compare the lower-bound complexity $\max\{h,S\}$ with
$H_{\mathrm{cert}}$ on one-positive-arm families.

\begin{lemma}[Certification complexity on one-positive-arm families]
\label{lem:one-positive-complexity}
For the family~\eqref{eq:one-positive-family} and every $a\in[K]$,
\begin{equation}
    \max\{h,S\}
    \le
    H_{\mathrm{cert}}(\mu^{(a)})
    \le
    h+S
    \le
    2\max\{h,S\}.
    \label{eq:one-positive-complexity-comparison}
\end{equation}
\end{lemma}

The proof is deferred to
Appendix~\ref{app:one-positive-complexity}.
Lemma~\ref{lem:one-positive-complexity} shows that, on these
one-positive-arm families, the lower-bound complexity
$\max\{h,S\}$ and the certification complexity
$H_{\mathrm{cert}}$ agree up to universal constants.
By Lemma~\ref{lem:complexity-comparison},
$H_{\mathrm{HTG}}$ is within a factor $Q=O(\log K)$ of
$H_{\mathrm{cert}}$. Together with
Theorem~\ref{thm:htg-main}, this shows that HTG's asymptotic
sample complexity is within an $O(\log K)$ factor of the lower
bound on these one-positive-arm families. In this sense, HTG is near-optimal up to
a logarithmic factor in its sample complexity on this class of instances.


The one-positive-arm families considered above include the
``one positive arm with other arms strongly negative'' example
from Section~\ref{sec:htg-interpretation} as a special case,
while also allowing arbitrary negative means.
As a complement, Appendix~\ref{app:equal-positive-exponent}
analyzes the ``many positive arms'' example. When all arms have
the same positive mean, every admissible algorithm has the exact
asymptotic posterior-failure exponent
$1/(2H_{\mathrm{pos}}(\mu))$, with $K$ and all model parameters
held fixed.

\section{Connections and Comparisons}
\label{sec:connections}

We first compare HTG with two standard fixed-budget sampling strategies,
uniform sampling and sequential halving, and then discuss a frequentist
interpretation of the posterior quantities used by HTG.

\subsection{Uniform sampling and sequential halving}
\label{sec:baseline-comparisons}

We compare HTG with uniform sampling and sequential halving (SH)
\citep{karnin13almost}, changing only the sampling rule.
All methods use the same posterior updates from
Section~\ref{sec:htg-algorithm}, with all observations collected from an arm used
to update its posterior. At the end of the budget, every method recommends 
$J_T\in\arg\max_{i\in[K]} Z_{i,N_i(T)}$. 
As discussed in Section~\ref{sec:one-positive-families}, this recommendation is optimal for the posterior-success objective: whenever any arm satisfies the posterior-success criterion~\eqref{eq:success-event}, the recommended arm does as well.
For elimination-based sampling rules, eliminated arms remain eligible for the final recommendation. 

\textbf{Uniform sampling.}
Uniform sampling allocates the budget round-robin, so each arm receives either
$n=\lfloor T/K\rfloor$ or $n+1$ observations. Recall
$c=\sqrt{2}\,z+\sqrt{\lambda}|\mu_0|/\sigma$. If $n\ge\lambda$, then
\begin{equation}
\label{eq:uniform-finite-budget}
    \mathbb P_\mu^{\mathrm{Uniform}}
    \bigl(\mathrm{Succ}_T^c\bigr)
    \le
    \exp\left[
        -\frac12
        \left(
            \frac{\mu_\star}{\sigma}\sqrt n-c
        \right)_+^2
    \right].
\end{equation}
Moreover, uniform sampling has the exact asymptotic failure exponent
\begin{equation}
\label{eq:uniform-exponent}
    \lim_{T\to\infty}
    -\frac1T
    \log
    \mathbb P_\mu^{\mathrm{Uniform}}
    \bigl(\mathrm{Succ}_T^c\bigr)
    =
    \frac{1}{2H_{\mathrm{pos}}(\mu)}.
\end{equation}
\eqref{eq:uniform-exponent} shows that uniform sampling exploits the presence of multiple positive
arms through $H_{\mathrm{pos}}$, but cannot redirect observations away from
unfavorable arms. Exact formulas and proofs are given in
Appendix~\ref{app:uniform-sampling}. 

\textbf{Sequential halving.}
SH divides the budget approximately equally across
$Q-1=\lceil\log_2 K\rceil$ elimination stages
\citep{karnin13almost}. Starting with all $K$ arms, each stage samples the
active arms equally and retains the upper half according to their posterior
$Z$-scores. Since the active arms have the same cumulative sample count at
each elimination, this ranking is equivalent to ranking them by their
empirical means. 

For the posterior-success objective in~\eqref{eq:success-event}, SH satisfies a finite-budget guarantee whose leading instance-dependent complexity is
\begin{equation}
\label{eq:H-SH}
    H_{\mathrm{SH}}(\mu)
    :=
    (Q-1)
    \min\left\{
        \sum_{i=1}^{K}\frac{\sigma^2}{\mu_i^2},
        \frac{K\sigma^2}{\mu_\star^2}
    \right\}.
\end{equation}
The precise finite-budget guarantee, algorithm specification, and proof are
given in Appendix~\ref{app:sequential-halving}.

The leading complexities of the three methods can now be compared directly.
Uniform sampling has complexity $H_{\mathrm{pos}}(\mu)$ by
\eqref{eq:uniform-exponent}, and SH has complexity
$H_{\mathrm{SH}}(\mu)$. By the definition of $H_{\mathrm{HTG}}$ in
\eqref{eq:H-HTG}, we have, 
\begin{equation}
\label{eq:baseline-complexity-comparison}
    H_{\mathrm{HTG}}(\mu)
    \le H_{\mathrm{pos}}(\mu),
    \qquad
    H_{\mathrm{HTG}}(\mu)
    \le \frac{H_{\mathrm{SH}}(\mu)}{Q-1},
    \quad K\ge2.
\end{equation}

\eqref{eq:baseline-complexity-comparison} shows that, at the level of the leading instance-dependent complexities appearing
in our guarantees, HTG is no larger than either baseline and can be strictly
smaller by adapting its allocation to favorable instance structure.
Instance families on which HTG has strictly smaller sample complexity than
uniform sampling and SH are given in
Appendix~\ref{app:baseline-separation}, and we compare the three methods
empirically in Section~\ref{sec:experiments}.

\vspace{-0.5em}

\subsection{Frequentist interpretation}
\label{sec:frequentist-interpretation}

The posterior-based rules have a frequentist interpretation in terms of
one-sided Gaussian tests. Consider the flat-prior limit
$\sigma_0^2\to\infty$ with $\mu_0$ fixed, so that $\lambda\to0$. In this limit,
\[
    Z_{i,n}
    =
    \frac{\sqrt n\,\overline X_{i,n}}{\sigma},
    \qquad
    p_{i,n}:=1-\Phi(Z_{i,n}).
\]
For a fixed arm and deterministic sample count, $p_{i,n}$ is the usual
one-sided Gaussian $p$-value for testing $H_{0,i}:\mu_i\le0$ against
$H_{1,i}:\mu_i>0$. Moreover, in the flat-prior limit, the posterior
probability of positivity satisfies
$\mathbb P(\mu_i>0\mid X_{i,1:n})=1-p_{i,n}$.

Under this interpretation, HTG's halving phase retains the active arms with the smallest
current $p$-values, its greedy phase samples an arm with the smallest current
$p$-value, and its terminal recommendation selects an arm with the smallest
terminal $p$-value. The posterior-success criterion becomes
\[
    \mathrm{Succ}_T
    =
    \left\{
        \min_{i\in[K]}p_{i,N_i(T)}\le\delta
    \right\}.
\]
Equivalently, at least one one-sided Gaussian lower confidence bound
$\overline X_{i,N_i(T)}-z\sigma/\sqrt{N_i(T)}$ is nonnegative.
A derivation of these identities under a flat prior is given in
Appendix~\ref{app:frequentist-interpretation}.

\vspace{-0.5em}

\section{Experiments and Discussion}
\label{sec:experiments}

\textbf{Setup.}
We compare HTG, uniform sampling, and posterior-based sequential halving
(SH), discussed in Section~\ref{sec:baseline-comparisons}. All methods use the same posterior, incorporating all collected
observations, and recommend the arm with the largest terminal posterior
$Z$-score over all $K$ arms. We set $K=32$, $\sigma^2=1$, the working
prior to $\mathcal{N}(0,1)$, and $\delta=0.01$. We estimate
$\mathbb{P}^{\mathcal{A}}_\mu(\mathrm{Succ}_T)$ using $10{,}000$ 
repetitions per instance and budget. Within each repetition, the $n$th pull of arm $i$ yields the same
simulated observation for every method. Success is evaluated after exactly $T$ samples. We use $16$
approximately logarithmically spaced budgets from $576$ to $65{,}536$,
rounded to multiples of $K$.

\begin{figure}[h]
    \centering
    \includegraphics[width=\linewidth]{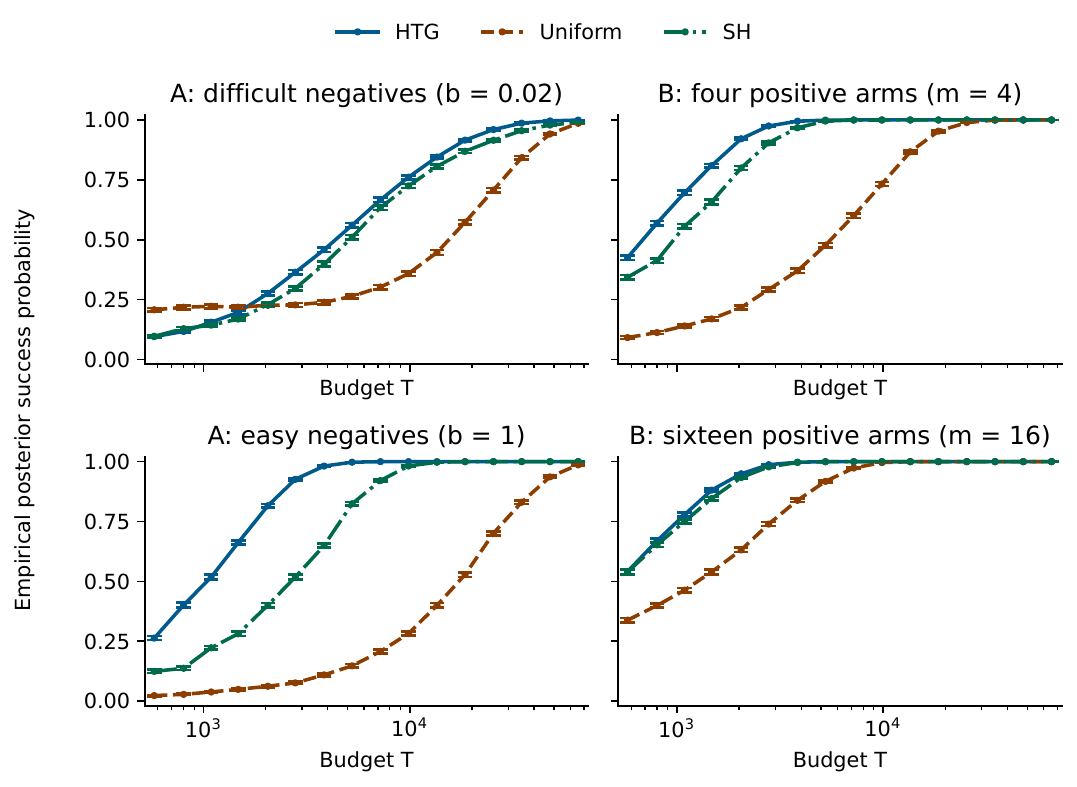}
    \caption{Posterior success probability versus budget.
    Left column: one arm has mean $0.1$, and the remaining arms have
    mean $-b$, with $b=0.02$ (top) and $b=1$ (bottom).
    Right column: $m$ arms have mean $0.1$, and the remaining arms have
    mean $-1$, with $m=4$ (top) and $m=16$ (bottom).
    Error bars show pointwise $95\%$ Wilson Monte Carlo confidence
    intervals based on $10{,}000$ repetitions.}
    \label{fig:ppos-experiments}
\end{figure}

\textbf{Negative competitors.}
The left column of Figure~\ref{fig:ppos-experiments} examines the
separation effect illustrated by the single-positive-arm example in
Section~\ref{sec:htg-interpretation}. Increasing $b$ reduces
$H_{\mathrm{cap}}$ while leaving $H_{\mathrm{pos}}$ unchanged.
HTG benefits substantially from easier negative competitors, and its
advantage over both baselines is more pronounced in the well-separated
setting. 

\textbf{Multiple positive arms.}
The right column examines the multiple-opportunity effect illustrated
by the many-positive-arm example in Section~\ref{sec:htg-interpretation}.
This behavior is consistent with the role of $H_{\mathrm{pos}}$,
which decreases in proportion to $1/m$ when the positive means are
equal and fixed: additional favorable arms provide more opportunities
to obtain a certificate.

Appendix~\ref{app:experiments} extends both comparisons over a broader
range of instance parameters. It also reports sample allocation
trajectories and comparisons across the number of arms. The latter use
the same instance family as Appendix~\ref{app:baseline-separation},
where uniform sampling and SH require
$\Omega(K)$ and $\Omega(\log K)$ times as many samples as HTG,
respectively. Empirical scalings with $K$ are consistent with the theoretical rates established on this family. 

\textbf{Conclusion.}
We proposed HTG for fixed-budget posterior success, exploiting the two sources of favorable instance structure summarized by
$H_{\mathrm{cert}}=\min\{H_{\mathrm{cap}},H_{\mathrm{pos}}\}$.
Its finite-budget exponential failure bound is governed by a complexity within
$O(\log K)$ of $H_{\mathrm{cert}}$, and our lower bounds establish
near-optimal asymptotic sample complexity on one-positive-arm families.
An interesting extension, motivated by experiments that compare multiple candidate treatments against a shared control \citep{villar2015multi}, is posterior certification relative to an unknown control mean, for example
requiring
$\mathbb P(\mu_{J_T}-\mu_c>\tau\mid\mathcal F_T)\ge1-\delta$.
This introduces a new allocation tradeoff between learning the common control
mean $\mu_c$ and gathering evidence for promising candidate arms.

\bibliographystyle{abbrvnat_lastname_first_overleaf}
\bibliography{library-overleaf}

\clearpage
\appendix
\thispagestyle{empty}

\onecolumn

\aistatstitle{
Supplementary Materials}

\vspace{-4em}
\part{APPENDIX} 
\parttoc 

\section{Related Work}
\label{app:related-work}

\textbf{Best-arm identification.}
Frequentist best-arm identification (BAI) seeks an arm with the largest
true mean. Fixed-budget methods minimize identification error under a
prescribed sampling budget
\citep{audibert10best,karnin13almost,carpentier16tight}, whereas
fixed-confidence methods minimize expected sample size subject to an
identification-error tolerance
\citep{kaufmann16onthecomplexity,garivier16optimal}.
Approximate identification and simple regret relax exact optimality. \citet{zhao23revisiting} analyze Sequential Halving for
returning approximately optimal arms and obtaining small simple regret.
Elimination methods are not restricted to exact BAI, and their
performance need be analyzed under the objective of interest.
Our comparisons evaluate uniform sampling and Sequential Halving directly
under posterior success, using the same posterior updates and terminal
recommendation rule as HTG, instead of transferring exact-BAI guarantees.

\medskip
\noindent
\textbf{Thresholding and good-arm identification.}
Thresholding bandits classify arms relative to a prescribed threshold,
allowing a specified precision around it \citep{locatelli16anoptimal}.
Good-arm identification seeks one or more acceptable arms without
requiring optimality
\citep{kano2019good,katzsamuels20true,jourdan2026anytime}.
The benefit of multiple acceptable arms is studied by
\citet{katzsamuels20true} and, more generally, in pure exploration with
multiple correct answers \citep{degenne19pure}.
Related sequential tests determine whether any arm lies on a specified
side of a threshold
\citep{kaufmann2018sequential,tabata2020bad}.
These problems define correctness through the true means. Our success
event instead concerns terminal posterior evidence: a truly favorable
recommendation can fail the credibility criterion, and an unfavorable
recommendation can satisfy it. 

\medskip
\noindent
\textbf{The connection to APGAI.}
APGAI is an anytime good-arm identification algorithm with fixed-budget
and fixed-confidence guarantees \citep{jourdan2026anytime}.
Its sampling index is
$\sqrt{N_i(t)}(\overline X_{i,N_i(t)}-\theta)$, where $\theta$ is the target
threshold; the authors attribute this index to the
$\mathrm{APT}_{\mathrm P}$ rule of \citet{tabata2020bad}.
At threshold zero, HTG's flat-prior score is
$Z_{i,N_i(t)}=\sqrt{N_i(t)}\,\overline X_{i,N_i(t)}/\sigma$.
With common observation variance, these indices have the same maximizers
for a given history.
The complete algorithms nevertheless differ: HTG uses a horizon-dependent
balanced allocation and halving phase before posterior-greedy sampling,
and evaluates a posterior certificate after exactly the prescribed budget.
All original arms remain eligible during its greedy phase and final
recommendation. 
Our contribution is HTG
sampling strategy that combines balanced allocation, successive halving,
and posterior-greedy sampling, together with its guarantees for terminal
posterior success.

\medskip
\noindent
\textbf{Bayesian bandits and best-arm identification.}
Posterior sampling originates with \citet{thompson33onthelikelihood}
and has been extensively analyzed for cumulative regret
\citep{agrawal12analysis,russo14learning}.
For Bayesian BAI, \citet{russo2020simple} studies the asymptotic decay of
posterior probability assigned to incorrect best-arm identities, holding
the true instance fixed.
\citet{shang20fixed} develop fixed-confidence stopping and sample-complexity
guarantees for Bayesian sampling rules.
Bayesian fixed-budget BAI instead evaluates best-arm misidentification
probability averaged over a prior on instances
\citep{atsidakou2022bayesian,pmlr-v258-nguyen25a};
prior-averaged fixed-confidence guarantees are studied by
\citet{jang2024fixed}.
Recent work also permits abstention and controls incorrect non-abstaining
recommendations under an abstention budget \citep{huang2026bayesian}.
These criteria should be distinguished from our sampling probability of
terminal posterior failure: we neither require identifying the best arm
nor treat abstention as a separately budgeted outcome.

\medskip
\noindent
\textbf{Bayesian satisficing.}
The $(M,\delta)$-satisficing formulation of
\citet{reverdy2016satisficing} is a conceptual predecessor.
An arm is satisfactory if its mean exceeds a threshold $M$, and is regarded
as sufficiently satisfactory when the current belief assigns probability at
least $1-\delta$ to this event.
The posterior criterion is related to ours.
The difference is the objective: their performance measure is cumulative
satisficing regret over the sequence of decisions, whereas we ask only whether
the recommended arm satisfies the posterior criterion at the end of a fixed
sampling budget.

\medskip
\noindent
\textbf{Bayesian experimental design and probability of success.}
Bayesian experimental design selects experiments by maximizing expected
utility \citep{chaloner1995bayesian}. If the utility is the indicator that a
prescribed posterior criterion is satisfied, its expected utility is the
prior-averaged probability that the experiment achieves posterior success.
Related notions of assurance similarly evaluate, before an experiment, the
probability that a specified success criterion will ultimately be met
\citep{o2005assurance}.
\citet{hagar2025economical} study how posterior probabilities vary across repeated sampling under the assumed data-generating model and use these sampling distributions to choose sample sizes and posterior decision thresholds.
Our contribution is not only the use of posterior-probability success criteria, but the adaptive allocation of a fixed budget across multiple
arms and instance-dependent guarantees for achieving terminal posterior
success. A related fixed-budget exploration--validation design adaptively selects a
treatment and then uses a separate A/B test for validation
\citep{pmlr-v300-zhang26c}.
In contrast, all observations collected by HTG contribute to the terminal
posterior used to determine posterior success.


\section{Gaussian posterior calculations}
\label{app:posterior}

We derive the posterior quantities used in
Section~\ref{sec:htg-algorithm}. Recall that
\[
    \mu_i\sim\cN(\mu_0,\sigma_0^2),
    \qquad
    X_{i,s}\mid\mu_i\sim\cN(\mu_i,\sigma^2),
\]
and define $\lambda:=\sigma^2/\sigma_0^2$.

Consider any realized history in which arm $i$ has been observed $n$
times, with observations $X_{i,1},\ldots,X_{i,n}$. Since the sampling decisions depend only on previously observed data and
internal randomization, the posterior update for $\mu_i$ is proportional to
its prior density times the likelihood of its observed rewards. By Bayes' rule,
\[
    p(\mu_i\mid X_{i,1:n})
    \propto
    \exp\left\{
    -\frac{(\mu_i-\mu_0)^2}{2\sigma_0^2}
    -\frac{1}{2\sigma^2}
    \sum_{s=1}^n(X_{i,s}-\mu_i)^2
    \right\}.
\]
Collecting the terms that depend on $\mu_i$ gives
\[
    p(\mu_i\mid X_{i,1:n})
    \propto
    \exp\left\{
    -\frac{n+\lambda}{2\sigma^2}
    \left(\mu_i-\widehat\mu_{i,n}\right)^2
    \right\}.
\]
Therefore,
\[
    \mu_i\mid X_{i,1:n}
    \sim
    \cN\left(
    \widehat\mu_{i,n},
    \frac{\sigma^2}{n+\lambda}
    \right).
\]
Writing
$\widehat\sigma_n:=\sigma/\sqrt{n+\lambda}$, the posterior
$Z$-score is
\[
    Z_{i,n}
    :=
    \frac{\widehat\mu_{i,n}}{\widehat\sigma_n}
    =
    \frac{\lambda\mu_0+n\overline X_{i,n}}
         {\sigma\sqrt{n+\lambda}}.
\]
Hence,
\[
    \PP(\mu_i>0\mid X_{i,1:n})
    =
    \Phi(Z_{i,n}).
\]


\section{Proof of the Complexity Comparison}
\label{app:complexity-comparison}

\begin{proof}[Proof of Lemma~\ref{lem:complexity-comparison}]
For convenience, define
\[
    H_{\mathrm{thr}}
    :=
    \sum_{i=1}^K\frac{\sigma^2}{\mu_i^2},
    \qquad
    H_{2,\mathrm{cap}}
    :=
    \max_{1\le r\le K}
    \frac{r\sigma^2}
    {\max\{\mu_\star,\mu_\star-\mu_{(r)}\}^2},
\]
so that
\[
    H_{\mathrm{HTG}}
    =
    \min\{H_{\mathrm{thr}},H_{\mathrm{pos}},
    QH_{2,\mathrm{cap}}\}.
\]

We first compare $H_{\mathrm{cap}}$ with the first and third terms above.
Since
\[
    \max\{\mu_\star,\mu_\star-\mu_i\}\ge |\mu_i|
\]
for every $i$, with the convention $\sigma^2/0^2=+\infty$, we have
\begin{equation}
    H_{\mathrm{cap}}
    \le H_{\mathrm{thr}}.
    \label{eq:app-Hcap-Hthr}
\end{equation}

Next, let
\[
    a_r
    :=
    \frac{\sigma^2}
    {\max\{\mu_\star,\mu_\star-\mu_{(r)}\}^2},
    \qquad r\in[K].
\]
Since $\mu_{(1)}\ge\cdots\ge\mu_{(K)}$, the sequence $(a_r)_{r=1}^K$
is nonincreasing. Therefore,
\[
    H_{2,\mathrm{cap}}
    =
    \max_{1\le r\le K} r a_r
    \le
    \sum_{r=1}^K a_r
    =
    H_{\mathrm{cap}}.
\]
Conversely, $a_r\le H_{2,\mathrm{cap}}/r$ for every $r$, and hence
\[
    H_{\mathrm{cap}}
    =
    \sum_{r=1}^K a_r
    \le
    H_{2,\mathrm{cap}}
    \sum_{r=1}^K\frac1r.
\]
Using
\[
    \sum_{r=1}^K\frac1r
    \le 1+\log K
    \le 1+\log_2 K
    \le Q,
\]
where $Q=1+\lceil\log_2K\rceil$, we obtain
\begin{equation}
    H_{\mathrm{cap}}
    \le
    QH_{2,\mathrm{cap}}.
    \label{eq:app-Hcap-H2}
\end{equation}

Now,
\[
    H_{\mathrm{cert}}
    =
    \min\{H_{\mathrm{cap}},H_{\mathrm{pos}}\}.
\]
By~\eqref{eq:app-Hcap-Hthr} and~\eqref{eq:app-Hcap-H2},
\[
    H_{\mathrm{cert}}
    \le H_{\mathrm{thr}},
    \qquad
    H_{\mathrm{cert}}
    \le H_{\mathrm{pos}},
    \qquad
    H_{\mathrm{cert}}
    \le QH_{2,\mathrm{cap}}.
\]
Therefore,
\[
    H_{\mathrm{cert}}
    \le
    \min\{H_{\mathrm{thr}},H_{\mathrm{pos}},
    QH_{2,\mathrm{cap}}\}
    =
    H_{\mathrm{HTG}}.
\]

For the reverse inequality, since
$H_{2,\mathrm{cap}}\le H_{\mathrm{cap}}$ and $Q\ge1$,
\[
    H_{\mathrm{HTG}}
    \le QH_{2,\mathrm{cap}}
    \le QH_{\mathrm{cap}},
\]
while also
\[
    H_{\mathrm{HTG}}
    \le H_{\mathrm{pos}}
    \le QH_{\mathrm{pos}}.
\]
Hence
\[
    H_{\mathrm{HTG}}
    \le
    \min\{QH_{\mathrm{cap}},QH_{\mathrm{pos}}\}
    =
    QH_{\mathrm{cert}}.
\]

Combining the two bounds gives
\[
    H_{\mathrm{cert}}
    \le
    H_{\mathrm{HTG}}
    \le
    QH_{\mathrm{cert}}.
\]
\end{proof}

\section{Proof of the Finite-Budget Guarantee for HTG}
\label{app:htg-main}

\begin{proof}[Proof of Theorem~\ref{thm:htg-main}]
Fix a mean vector $\mu$ with $\mu_\star>0$. Throughout the proof, all
complexities are evaluated at this fixed instance, and we abbreviate
$\mathbb{P}_{\mu}^{\mathrm{HTG}}$ by $\mathbb{P}_{\mu}$. We work with the
independent potential-observation sequences $(X_{i,s})_{s\ge1}$ for
$i\in[K]$, of which HTG reveals prefixes through its adaptive allocation.

Let
\[
    \mathcal P:=\{i\in[K]:\mu_i>0\},
    \qquad
    \mathcal N:=\{i\in[K]:\mu_i<0\},
\]
and recall that
\[
    m=|\mathcal P|,
    \qquad
    b_0:=\frac{\sqrt{\lambda}\,|\mu_0|}{\sigma},
    \qquad
    c=\sqrt{2}\,z+b_0.
\]
For convenience, define
\[
    H_{\mathrm{thr}}
    :=
    \sum_{i=1}^K\frac{\sigma^2}{\mu_i^2},
    \qquad
    H_{2,\mathrm{cap}}
    :=
    \max_{1\le r\le K}
    \frac{r\sigma^2}
    {\max\{\mu_\star,\mu_\star-\mu_{(r)}\}^2},
\]
with the convention $\sigma^2/0^2=+\infty$, and write
\[
    h:=H_{\mathrm{HTG}}(\mu)
    =
    \min\{H_{\mathrm{thr}},H_{\mathrm{pos}},QH_{2,\mathrm{cap}}\}.
\]
Since $\mu_\star>0$, we have $m\ge1$ and $0<h<\infty$.
The offset in~\eqref{eq:T0} is
$T_0=T_{\mathrm{cut}}+96mc^2h$.

If $T\le T_0$, the right-hand side of~\eqref{eq:htg-main-bound} equals
one, so the claim is immediate. We therefore assume
$T>T_0\ge T_{\mathrm{cut}}$ throughout the remainder of the proof.

\medskip
\noindent
\textbf{Step 1: Allocation bounds and a common concentration event.}

Let $k_\ell=|S_\ell|$ and $q_\ell$ be the active-set size and the
additional per-arm sample count at halving stage $\ell$. Since
$T\ge T_{\mathrm{cut}}=3K\max\{n_\lambda,Q\}$, the first-stage count
satisfies
\begin{equation}
    n_0:=q_1
    =
    \left\lfloor\frac{T}{3K}\right\rfloor
    \ge n_\lambda
    =
    \max\{1,\lceil\lambda\rceil\}.
    \label{eq:htg-proof-initial-count}
\end{equation}
Consequently, every arm has at least $n_0\ge\lambda$ observations after
the first stage, and in particular $N_i(T)\ge n_0$ for every $i$.

The quantities inside the floors defining $q_\ell$ are at least one.
Using $\lfloor x\rfloor\ge x/2$ for $x\ge1$, we obtain
\begin{equation}
    q_\ell\ge\frac{T}{6Qk_\ell},
    \qquad \ell=1,\ldots,Q.
    \label{eq:htg-proof-stage-count}
\end{equation}
For $\ell=1$, this follows from the stronger inequality
$q_1\ge T/(6K)$.

The halving phase uses at most
$T/3+(Q-1)T/(3Q)<2T/3$ observations. Hence, if $G$ denotes the number
of observations allocated during the final greedy phase, then
\begin{equation}
    G\ge \frac{T}{3}.
    \label{eq:htg-proof-greedy-budget}
\end{equation}

For each arm and deterministic sample count, define
\[
    W_{i,n}
    :=
    \frac{1}{\sigma}\sum_{s=1}^n(X_{i,s}-\mu_i),
    \qquad n\ge1,
\]
and set $W_{i,0}=0$. Under $\mathbb P_\mu$, each $(W_{i,n})_{n\ge0}$
is a standard Gaussian random walk, and these walks are independent
across arms. For $R>0$, let
\[
    \mathcal E_T(R)
    :=
    \bigcap_{i=1}^K\bigcap_{n=1}^T
    \bigl\{|W_{i,n}|\le R\sqrt n\bigr\}.
\]
Since $W_{i,n}/\sqrt n$ is standard normal for each deterministic
pair $(i,n)$, a Gaussian tail bound and a union bound give
\begin{equation}
    \mathbb P_\mu\bigl(\mathcal E_T(R)^c\bigr)
    \le 2KT\exp(-R^2/2).
    \label{eq:htg-proof-concentration}
\end{equation}

Write
\[
    a_i(n)
    :=
    \frac{n\mu_i+\lambda\mu_0}{\sigma\sqrt{n+\lambda}},
    \qquad
    Z_{i,n}
    =
    a_i(n)+\frac{W_{i,n}}{\sqrt{n+\lambda}}.
\]
On $\mathcal E_T(R)$, simultaneously for all $i\in[K]$ and
$1\le n\le T$,
\begin{equation}
    |\overline X_{i,n}-\mu_i|
    \le \frac{\sigma R}{\sqrt n},
    \qquad
    |Z_{i,n}-a_i(n)|\le R.
    \label{eq:htg-proof-score-control}
\end{equation}
Because these inequalities hold at every deterministic count, they also
hold at every adaptive count encountered by HTG. No conditioning on an
adaptive sample count is needed.

\medskip
\noindent
\textbf{Step 2: Terminal failure bounds positive-arm sample counts.}

Fix $R>0$ and work on
$\mathcal E_T(R)\cap\mathrm{Succ}_T^c$ for this step and the next two
steps. By the terminal recommendation rule,
\[
    \mathrm{Succ}_T^c
    =
    \left\{
        \max_{i\in[K]}Z_{i,N_i(T)}<z
    \right\}.
\]
Let $i\in\mathcal P$ and put $n=N_i(T)$. Then $n\ge n_0\ge\lambda$
and $Z_{i,n}<z$. Using $W_{i,n}\ge-R\sqrt n$ gives
\[
    \frac{\mu_i}{\sigma}\sqrt n
    <
    R+z\sqrt{1+\frac{\lambda}{n}}
    -\frac{\lambda\mu_0}{\sigma\sqrt n}.
\]
Since $n\ge\lambda$,
\[
    z\sqrt{1+\frac{\lambda}{n}}\le\sqrt{2}\,z,
    \qquad
    \left|\frac{\lambda\mu_0}{\sigma\sqrt n}\right|\le b_0.
\]
Therefore,
\begin{equation}
    N_i(T)
    \le
    \frac{\sigma^2(R+c)^2}{\mu_i^2},
    \qquad i\in\mathcal P.
    \label{eq:htg-proof-positive-count}
\end{equation}
This implication uses failure at the terminal time. It does not assume
that a certificate obtained earlier remains valid after further
observations.

\medskip
\noindent
\textbf{Step 3: The greedy-phase bound.}

For this step, suppose $H_{\mathrm{thr}}<\infty$, so that no arm has
zero mean. Let $G_i$ denote the number of observations allocated to
arm $i$ during the final greedy phase. Thus,
$G=\sum_{i=1}^K G_i$. For each positive arm,
\begin{equation}
    G_i\le N_i(T)
    \le\frac{\sigma^2(R+c)^2}{\mu_i^2},
    \qquad i\in\mathcal P.
    \label{eq:htg-proof-positive-greedy}
\end{equation}

To bound the greedy selections of negative arms, choose any fixed
$p\in\mathcal P$. For every $n\ge1$,
\[
    a_p(n)
    =
    \frac{n\mu_p+\lambda\mu_0}{\sigma\sqrt{n+\lambda}}
    \ge -b_0.
\]
It follows from~\eqref{eq:htg-proof-score-control} that at every
greedy decision time $t$,
\[
    Z_{p,N_p(t)}\ge-b_0-R.
\]

Suppose that a negative arm $j$ is selected at such a time, and write
$n=N_j(t)$ for its count before the observation. Since the greedy
rule maximizes the posterior $Z$-score over all $K$ arms,
\[
    a_j(n)+R
    \ge Z_{j,n}
    \ge Z_{p,N_p(t)}
    \ge-b_0-R.
\]
Hence $a_j(n)\ge-b_0-2R$. On the other hand, $n\ge n_0\ge\lambda$
and $\mu_j<0$ imply
\[
    a_j(n)
    \le
    -\frac{|\mu_j|}{\sigma}\sqrt{\frac n2}+b_0.
\]
Combining these inequalities yields
\[
    n\le\frac{8\sigma^2(R+b_0)^2}{\mu_j^2}.
\]

Before its $g$-th greedy selection, arm $j$ has at least
$n_0+g-1\ge g$ observations. Applying the preceding bound at its
last greedy selection, and using $b_0\le c$, gives
\begin{equation}
    G_j
    \le
    \frac{8\sigma^2(R+c)^2}{\mu_j^2},
    \qquad j\in\mathcal N.
    \label{eq:htg-proof-negative-greedy}
\end{equation}
The same inequality holds trivially when $G_j=0$.

Summing~\eqref{eq:htg-proof-positive-greedy}
and~\eqref{eq:htg-proof-negative-greedy}, we obtain
\[
    G
    \le
    8H_{\mathrm{thr}}(R+c)^2.
\]
Together with~\eqref{eq:htg-proof-greedy-budget}, this proves
\begin{equation}
\begin{aligned}
    &\mathcal E_T(R)\cap\mathrm{Succ}_T^c\\
    &\qquad\Longrightarrow\quad
    T\le24H_{\mathrm{thr}}(R+c)^2.
\end{aligned}
    \label{eq:htg-proof-greedy-bound}
\end{equation}
When some arm has zero mean, $H_{\mathrm{thr}}=+\infty$, and this
branch is not needed.

\medskip
\noindent
\textbf{Step 4: The halving-phase bound.}

Fix a best arm $i_\star$ with $\mu_{i_\star}=\mu_\star$, assigning it
rank one when ordering arms with equal means. At the end of each
halving stage, all active arms have the same total sample count.
Since the prior parameters are common across arms, their ordering by
posterior $Z$-score is therefore the same as their ordering by empirical
mean.

First suppose that $i_\star$ is removed at some halving stage $\ell$.
Write $s=k_\ell$ for the number of active arms, and let
\[
    k:=\lceil s/2\rceil,
    \qquad
    r:=k+1.
\]
Let $n$ be the common sample count of the active arms at the end of
that stage. By~\eqref{eq:htg-proof-stage-count},
\begin{equation}
    n\ge q_\ell\ge\frac{T}{6Qs}.
    \label{eq:htg-proof-elimination-count}
\end{equation}

The $k$ retained arms all have empirical means at least as large as
that of $i_\star$. Since none of them is $i_\star$, at least one
retained arm $j$ has global rank at least $r$. In particular,
$\mu_j\le\mu_{(r)}$ and
$\overline X_{j,n}\ge\overline X_{i_\star,n}$. By
\eqref{eq:htg-proof-score-control},
\[
    \mu_\star-\mu_{(r)}
    \le
    \mu_\star-\mu_j
    \le
    \frac{2\sigma R}{\sqrt n}.
\]
Consequently,
\[
    n(\mu_\star-\mu_{(r)})^2\le4\sigma^2R^2.
\]
Also, $n\le N_{i_\star}(T)$, so
\eqref{eq:htg-proof-positive-count} gives
\[
    n\mu_\star^2\le\sigma^2(R+c)^2.
\]
Writing
$\gamma_{(r)}:=\max\{\mu_\star,\mu_\star-\mu_{(r)}\}$, we conclude
that
\[
    n\gamma_{(r)}^2\le4\sigma^2(R+c)^2.
\]
Combining this with~\eqref{eq:htg-proof-elimination-count} and
$s\le2r$ yields
\[
\begin{aligned}
    T
    &\le
    24Qs\frac{\sigma^2}{\gamma_{(r)}^2}(R+c)^2\\
    &\le
    48Q\frac{r\sigma^2}{\gamma_{(r)}^2}(R+c)^2\\
    &\le
    48QH_{2,\mathrm{cap}}(R+c)^2.
\end{aligned}
\]

Next suppose that $i_\star$ is never removed. After $Q-1$ halvings
the active set is a singleton, so $i_\star$ receives the final-stage
samples. By~\eqref{eq:htg-proof-stage-count},
\[
    N_{i_\star}(T)\ge q_Q\ge\frac{T}{6Q}.
\]
This also holds when $K=1$, in which case $Q=1$ and the first stage
is already a singleton stage. Applying
\eqref{eq:htg-proof-positive-count} gives
\[
    T
    \le
    6Q\frac{\sigma^2}{\mu_\star^2}(R+c)^2
    \le
    6QH_{2,\mathrm{cap}}(R+c)^2.
\]

Thus, in both cases,
\begin{equation}
\begin{aligned}
    &\mathcal E_T(R)\cap\mathrm{Succ}_T^c\\
    &\qquad\Longrightarrow\quad
    T\le48QH_{2,\mathrm{cap}}(R+c)^2.
\end{aligned}
    \label{eq:htg-proof-halving-bound}
\end{equation}
This argument allows zero means among the competing arms.

\medskip
\noindent
\textbf{Step 5: The positive-arm bound.}

This step does not condition on $\mathcal E_T(R)$. Instead, we use
the deterministic lower bound $N_i(T)\ge n_0$ from
\eqref{eq:htg-proof-initial-count} and an armwise crossing argument.

Fix $i\in\mathcal P$ and write $u_i:=\mu_i/\sigma>0$. For every
$n\ge n_0\ge\lambda$,
\[
\begin{aligned}
    Z_{i,n}<z
    &\Longrightarrow
    W_{i,n}
    <
    z\sqrt{n+\lambda}
    -u_i n-\frac{\lambda\mu_0}{\sigma}\\
    &\Longrightarrow
    W_{i,n}<c\sqrt n-u_i n.
\end{aligned}
\]
Define
\[
    M_{i,n}
    :=
    \exp\left(
        -\frac{u_i}{2}W_{i,n}
        -\frac{u_i^2n}{8}
    \right),
    \qquad n\ge0.
\]
The Gaussian increments of $W_{i,n}$ imply that $(M_{i,n})_{n\ge0}$
is a nonnegative mean-one martingale with respect to the armwise
observation filtration, with $M_{i,0}=1$.

Whenever $n\ge n_0$ and $Z_{i,n}<z$, we have
\[
\begin{aligned}
    \log M_{i,n}
    &>
    \frac{3u_i^2n}{8}-\frac{u_i c\sqrt n}{2}\\
    &\ge
    \frac{u_i^2n}{4}-\frac{c^2}{2}\\
    &\ge
    \frac{u_i^2n_0}{4}-\frac{c^2}{2}.
\end{aligned}
\]
The middle inequality follows from
$(u_i\sqrt n-2c)^2\ge0$.

For completeness, the maximal inequality used here follows by stopping
$M_{i,n}$ at
$\tau_a:=\inf\{n\ge0:M_{i,n}\ge a\}$, where $a>0$.
For every deterministic integer $L\ge0$, bounded stopping and
nonnegativity give
\[
    a\,\mathbb P_\mu(\tau_a\le L)
    \le
    \mathbb E_\mu[M_{i,\tau_a\wedge L}]
    =1.
\]
Letting $L\to\infty$ yields
$\mathbb P_\mu(\tau_a<\infty)\le\min\{1,a^{-1}\}$.
Applying this with
$a=\exp(u_i^2n_0/4-c^2/2)$ gives
\begin{equation}
\begin{aligned}
    &\mathbb P_\mu\bigl(
        \exists n\ge n_0:Z_{i,n}<z
    \bigr)\\
    &\qquad\le
    \min\left\{
        1,\,
        \exp\left(
            \frac{c^2}{2}-\frac{u_i^2n_0}{4}
        \right)
    \right\}.
\end{aligned}
    \label{eq:htg-proof-armwise-crossing}
\end{equation}

For each $i\in\mathcal P$, let
$B_i:=\{\exists n\ge n_0:Z_{i,n}<z\}$.
Since $N_i(T)\ge n_0$, terminal failure implies
\[
    \mathrm{Succ}_T^c
    \subseteq
    \bigcap_{i\in\mathcal P}B_i.
\]
Each event $B_i$ depends only on arm $i$'s potential-observation
sequence. These events are therefore independent across arms, even
though the terminal counts $N_i(T)$ need not be independent.
Using~\eqref{eq:htg-proof-armwise-crossing}, we obtain
\[
\begin{aligned}
    \mathbb P_\mu(\mathrm{Succ}_T^c)
    &\le
    \prod_{i\in\mathcal P}\mathbb P_\mu(B_i)\\
    &\le
    \exp\left(
        \frac{mc^2}{2}
        -\frac{n_0}{4}
        \sum_{i\in\mathcal P}\frac{\mu_i^2}{\sigma^2}
    \right).
\end{aligned}
\]
By the definition of $H_{\mathrm{pos}}$,
\[
    \sum_{i\in\mathcal P}\frac{\mu_i^2}{\sigma^2}
    =
    \frac{K}{H_{\mathrm{pos}}}.
\]
Since $n_0=\lfloor T/(3K)\rfloor\ge T/(3K)-1$, it follows that
\begin{equation}
\begin{aligned}
    &\mathbb P_\mu(\mathrm{Succ}_T^c)\\
    &\qquad\le
    \min\left\{
        1,\,
        \exp\left(
            \frac{mc^2}{2}
            -\frac{T-3K}{12H_{\mathrm{pos}}}
        \right)
    \right\}.
\end{aligned}
    \label{eq:htg-proof-positive-arm-bound}
\end{equation}

\medskip
\noindent
\textbf{Step 6: Combining the three bounds.}

Recall that
$h=\min\{H_{\mathrm{thr}},H_{\mathrm{pos}},QH_{2,\mathrm{cap}}\}$
and that we are considering $T>T_0$.

First suppose that $h=H_{\mathrm{thr}}$ or $h=QH_{2,\mathrm{cap}}$.
In the former case $H_{\mathrm{thr}}<\infty$, so the greedy-phase
argument applies. By~\eqref{eq:htg-proof-greedy-bound}
or~\eqref{eq:htg-proof-halving-bound}, respectively,
\[
    \mathcal E_T(R)\cap\mathrm{Succ}_T^c
    \quad\Longrightarrow\quad
    T\le48h(R+c)^2.
\]
Using $(R+c)^2\le2R^2+2c^2$ and $m\ge1$ gives
\[
    T
    \le96hR^2+96hc^2
    \le96hR^2+T_0.
\]
Therefore, for every radius satisfying
\[
    0<R<\sqrt{\frac{T-T_0}{96h}},
\]
the event $\mathcal E_T(R)\cap\mathrm{Succ}_T^c$ is empty.
By~\eqref{eq:htg-proof-concentration},
\[
    \mathbb P_\mu(\mathrm{Succ}_T^c)
    \le2KT\exp(-R^2/2).
\]
Taking the infimum over these radii yields
\begin{equation}
    \mathbb P_\mu(\mathrm{Succ}_T^c)
    \le
    2KT\exp\left(-\frac{T-T_0}{192h}\right).
    \label{eq:htg-proof-combined-rate}
\end{equation}

It remains to consider $h=H_{\mathrm{pos}}$. Since
$T_{\mathrm{cut}}\ge3K$,
\[
    T_0
    =
    T_{\mathrm{cut}}+96mc^2h
    \ge
    3K+6mc^2h.
\]
Consequently,~\eqref{eq:htg-proof-positive-arm-bound} gives
\[
\begin{aligned}
    \mathbb P_\mu(\mathrm{Succ}_T^c)
    &\le
    \exp\left(
        \frac{mc^2}{2}-\frac{T-3K}{12h}
    \right)\\
    &=
    \exp\left(
        -\frac{T-(3K+6mc^2h)}{12h}
    \right)\\
    &\le
    \exp\left(-\frac{T-T_0}{12h}\right)\\
    &\le
    2KT\exp\left(-\frac{T-T_0}{192h}\right).
\end{aligned}
\]
Thus~\eqref{eq:htg-proof-combined-rate} holds regardless of which
term attains the minimum defining $h$.

Finally, combining this bound with the trivial bound by one and the
case $T\le T_0$ treated at the start, we obtain
\[
    \mathbb P_\mu^{\mathrm{HTG}}(\mathrm{Succ}_T^c)
    \le
    \min\left\{
        1,\,
        2KT\exp\left(
            -\frac{(T-T_0)_+}
            {192H_{\mathrm{HTG}}(\mu)}
        \right)
    \right\},
\]
as claimed.
\end{proof}

\section{An Explicit Sufficient Budget}
\label{app:explicit-budget}

The sufficient-budget condition in~\eqref{eq:htg-sufficient-budget}
contains $T$ inside the logarithm. The following consequence of
Theorem~\ref{thm:htg-main} gives an explicit sufficient budget without
this dependence.

\begin{corollary}[Explicit sufficient budget]
\label{cor:htg-explicit-budget}
Under the assumptions of Theorem~\ref{thm:htg-main}, fix a target failure
probability $\eta\in(0,1)$ and write $h:=H_{\mathrm{HTG}}(\mu)$. Define
\begin{equation}
\begin{aligned}
    L_\eta
    &:=
    \log\left(
        \max\left\{
            e,\frac{768Kh}{\eta}
        \right\}
    \right),\\
    T_\eta
    &:=
    \left\lceil
        \max\left\{
            2T_0,\;768hL_\eta
        \right\}
    \right\rceil,
\end{aligned}
    \label{eq:htg-explicit-budget}
\end{equation}
where $T_0$ is defined in~\eqref{eq:T0} and $\log$ denotes the natural
logarithm. Then, for every integer $T\ge T_\eta$,
\[
    \mathbb P_\mu^{\mathrm{HTG}}
    \bigl(\mathrm{Succ}_T^c\bigr)
    \le \eta.
\]
\end{corollary}

\begin{proof}
Since $\mu_\star>0$, we have $0<h<\infty$. Moreover,
$T_0\ge T_{\mathrm{cut}}\ge3K$, so $T\ge T_\eta$ implies
$T\ge2T_0\ge K$. In particular, $T-T_0\ge T/2$, and
Theorem~\ref{thm:htg-main} gives
\begin{equation}
    \mathbb P_\mu^{\mathrm{HTG}}
    \bigl(\mathrm{Succ}_T^c\bigr)
    \le
    2KT\exp\left(-\frac{T}{384h}\right).
    \label{eq:htg-explicit-budget-tail}
\end{equation}

Set $x:=T/(384h)$. The second condition in the definition of $T_\eta$
ensures that
\[
    x\ge2L_\eta\ge2.
\]
For $x\ge2$, the elementary inequality $\log x\le x/2$ holds.
Also, by the definition of $L_\eta$,
$\log(768Kh/\eta)\le L_\eta$. Consequently,
\[
\begin{aligned}
    \log\frac{2KT}{\eta}
    &=
    \log\frac{768Kh}{\eta}+\log x\\
    &\le
    L_\eta+\frac{x}{2}\\
    &\le x.
\end{aligned}
\]
It follows that
\[
    2KT\exp\left(-\frac{T}{384h}\right)
    =
    2KT e^{-x}
    \le\eta.
\]
Combining this inequality with~\eqref{eq:htg-explicit-budget-tail}
proves the claim.
\end{proof}

\section{Illustrative Instance Calculations}
\label{app:htg-examples}

We give the calculations underlying the examples in
Section~\ref{sec:htg-interpretation}, followed by an additional example
illustrating the effect of a competing arm whose mean is close to zero.
Throughout this appendix, the best mean is $a>0$, and we write
$h:=\sigma^2/a^2$. We suppress the instance argument $\mu$ in the
complexities when there is no ambiguity.

For convenience, define
\[
    H_{\mathrm{thr}}
    :=
    \sum_{i=1}^K\frac{\sigma^2}{\mu_i^2},
    \qquad
    H_{2,\mathrm{cap}}
    :=
    \max_{1\le r\le K}
    \frac{r\sigma^2}
    {\max\{\mu_\star,\mu_\star-\mu_{(r)}\}^2}.
\]
Thus,
\[
    H_{\mathrm{HTG}}
    =
    \min\{H_{\mathrm{thr}},H_{\mathrm{pos}},
    QH_{2,\mathrm{cap}}\},
\]
where $Q=1+\lceil\log_2K\rceil$ and
$\sigma^2/0^2:=+\infty$.

\subsection{One positive arm with other arms strongly negative}
\label{app:htg-example-one-positive}

Let $K\ge2$ and consider
$\mu_1=a>0$ and $\mu_2=\cdots=\mu_K=-b$, where
$b\ge a\sqrt K$. Directly from the definitions,
\[
    H_{\mathrm{cap}}
    =
    h+\frac{(K-1)\sigma^2}{(a+b)^2},
    \qquad
    H_{\mathrm{pos}}=Kh.
\]
Every summand of $H_{\mathrm{cap}}$ is at most $h$, so
$H_{\mathrm{cap}}\le Kh=H_{\mathrm{pos}}$. Therefore,
\begin{equation}
    H_{\mathrm{cert}}
    =
    H_{\mathrm{cap}}
    =
    h+\frac{(K-1)\sigma^2}{(a+b)^2}.
    \label{eq:example-one-positive-cert}
\end{equation}

To evaluate $H_{\mathrm{HTG}}$, first note that
\begin{equation}
    H_{\mathrm{thr}}
    =
    h+\frac{(K-1)\sigma^2}{b^2}
    \le
    \left(2-\frac1K\right)h
    <2h.
    \label{eq:example-one-positive-thr}
\end{equation}
The inequality uses $b^2\ge Ka^2$.

The best arm contributes $h$ to the rank-weighted maximum defining
$H_{2,\mathrm{cap}}$. All other arms have the same denominator
$(a+b)^2$, so their largest rank-weighted contribution occurs at
rank $K$. Hence
\[
    H_{2,\mathrm{cap}}
    =
    \max\left\{
        h,\frac{K\sigma^2}{(a+b)^2}
    \right\}
    =
    h,
\]
where the last equality follows from $(a+b)^2\ge Ka^2$.

Since $K\ge2$ and $Q\ge2$, we have
$H_{\mathrm{pos}}=Kh\ge2h$ and
$QH_{2,\mathrm{cap}}=Qh\ge2h$. Combining this with
\eqref{eq:example-one-positive-thr} shows that the first term attains
the minimum defining $H_{\mathrm{HTG}}$:
\begin{equation}
    H_{\mathrm{HTG}}
    =
    H_{\mathrm{thr}}
    =
    h+\frac{(K-1)\sigma^2}{b^2}.
    \label{eq:example-one-positive-htg}
\end{equation}
In particular,
\begin{equation}
    h
    \le H_{\mathrm{cert}}
    \le H_{\mathrm{HTG}}
    <2h,
    \qquad
    H_{\mathrm{pos}}=Kh.
    \label{eq:example-one-positive-comparison}
\end{equation}
Thus both $H_{\mathrm{cert}}$ and $H_{\mathrm{HTG}}$ are within a
factor of two of $\sigma^2/a^2$, while
\[
    \frac K2
    \le
    \frac{H_{\mathrm{pos}}}{H_{\mathrm{HTG}}}
    \le K.
\]

This calculation also identifies the role of the final posterior-greedy
phase. Its analysis supplies the term $H_{\mathrm{thr}}<2h$, whereas
the term supplied by the halving analysis is $Qh$. Including the
posterior-greedy phase therefore avoids the additional factor $Q$
in the halving term on this family.

\subsection{Many positive arms}
\label{app:htg-example-many-positive}

Suppose $\mu_1=\cdots=\mu_K=a>0$. Every best-arm gap is zero, so
each denominator in $H_{\mathrm{cap}}$ equals $a^2$. Also,
$\sum_{i=1}^K(\mu_i)_+^2=Ka^2$. Consequently,
\[
    H_{\mathrm{cap}}=Kh,
    \qquad
    H_{\mathrm{pos}}=h,
    \qquad
    H_{\mathrm{cert}}=h.
\]
The other two terms in $H_{\mathrm{HTG}}$ satisfy
\[
    H_{\mathrm{thr}}=Kh,
    \qquad
    H_{2,\mathrm{cap}}
    =
    \max_{1\le r\le K}rh
    =
    Kh.
\]
It follows that
\begin{equation}
    H_{\mathrm{HTG}}
    =
    \min\{Kh,h,QKh\}
    =
    h
    =
    \frac{\sigma^2}{a^2}.
    \label{eq:example-many-positive-htg}
\end{equation}
Thus the positive-arm term makes $H_{\mathrm{HTG}}$ a factor $K$
smaller than $H_{\mathrm{cap}}$ on this family.

This comparison concerns the complexity in the denominator of the
exponential bound, rather than the entire sufficient budget. Here,
$m=K$, so the offset in~\eqref{eq:T0} is
\begin{equation}
    T_0
    =
    3K\max\{n_\lambda,Q\}
    +
    96Kc^2h.
    \label{eq:example-many-positive-offset}
\end{equation}
Although $H_{\mathrm{HTG}}=h$ does not depend on $K$ when $a$ and
$\sigma$ are fixed, the offset does. Hence this example does not give
a sufficient budget independent of $K$.

\subsection{A competing arm with mean close to zero}
\label{app:htg-example-near-zero}

Let $K\ge3$ and consider
\[
    \mu_1=a,\qquad
    \mu_2=-\varepsilon,\qquad
    \mu_3=\cdots=\mu_K=-b,
\]
where $0\le\varepsilon\le a/\sqrt K$ and $b\ge a\sqrt K$.
There is exactly one positive arm, so $H_{\mathrm{pos}}=Kh$ and
$H_{\mathrm{cert}}=H_{\mathrm{cap}}$. Moreover,
\begin{equation}
\begin{aligned}
    H_{\mathrm{cap}}
    &=
    h+\frac{\sigma^2}{(a+\varepsilon)^2}
      +\frac{(K-2)\sigma^2}{(a+b)^2}\\
    &\le
    h+h+\frac{K-2}{K}h
    <3h.
\end{aligned}
    \label{eq:example-near-zero-cap}
\end{equation}
Thus $H_{\mathrm{cert}}$ remains bounded as $\varepsilon$ approaches
zero.

In contrast,
\[
    H_{\mathrm{thr}}
    =
    h+\frac{\sigma^2}{\varepsilon^2}
      +\frac{(K-2)\sigma^2}{b^2}
\]
diverges as $\varepsilon\downarrow0$ and equals $+\infty$ when
$\varepsilon=0$. The assumption $\varepsilon\le a/\sqrt K$ gives
$H_{\mathrm{thr}}\ge Kh$, including at $\varepsilon=0$ under our
convention. Therefore,
\begin{equation}
    \min\{H_{\mathrm{thr}},H_{\mathrm{pos}}\}=Kh.
    \label{eq:example-near-zero-two-terms}
\end{equation}

The halving term gives a smaller bound when $K$ is sufficiently large.
Since the means are already ordered,
\[
\begin{aligned}
    H_{2,\mathrm{cap}}
    &=
    \max\left\{
        h,\frac{2\sigma^2}{(a+\varepsilon)^2},
        \frac{K\sigma^2}{(a+b)^2}
    \right\}\\
    &=
    \max\left\{
        h,\frac{2\sigma^2}{(a+\varepsilon)^2}
    \right\}.
\end{aligned}
\]
The second equality uses $K\sigma^2/(a+b)^2\le h$.
In particular, $h\le H_{2,\mathrm{cap}}\le2h$. Combining this
with~\eqref{eq:example-near-zero-two-terms} gives
\begin{equation}
    H_{\mathrm{HTG}}
    =
    \min\{Kh,QH_{2,\mathrm{cap}}\}
    \le
    \min\{K,2Q\}h.
    \label{eq:example-near-zero-htg}
\end{equation}
At $\varepsilon=0$, the rank-weighted maximum is exactly $2h$,
and hence
\[
    H_{\mathrm{HTG}}
    =
    \min\{K,2Q\}h.
\]

This example illustrates why the halving term is useful in addition to
the first two terms in~\eqref{eq:H-HTG}. A single negative mean close
to zero makes $H_{\mathrm{thr}}$ arbitrarily large, while
$H_{\mathrm{pos}}$ remains $Kh$. The halving term instead gives
$H_{\mathrm{HTG}}\le2Qh$, which remains bounded as
$\varepsilon\downarrow0$. Since $m=1$, its offset also satisfies
\[
    T_0
    \le
    T_{\mathrm{cut}}+192Qc^2h.
\]
Thus neither the complexity nor the offset in the HTG guarantee
diverges when this competing mean approaches zero.

\section{Proof of the lower bound}
\label{app:lower-bounds}

We prove Theorem~\ref{thm:one-positive-lower-bound} and its
sample-complexity consequence
\eqref{eq:one-positive-sample-lower-bound}.
Throughout this appendix, we consider the algorithms described in
Section~\ref{sec:one-positive-families}. Each algorithm collects exactly
$T$ observations, samples every arm at least once, and uses the
prescribed posterior and terminal recommendation. The algorithm may
be randomized and may depend on $T$.

At each round, the algorithm selects an arm using only previously
collected observations and its internal randomization. The same
sampling rule is used under the original and auxiliary instances.

Recall that
\[
    \lambda=\frac{\sigma^2}{\sigma_0^2},
    \qquad
    z=\Phi^{-1}(1-\delta)>0.
\]

\subsection{Finite-budget lower bounds}
\label{app:finite-budget-lower-bounds}

Define
\[
    b_{\mathrm{pr}}
    :=
    z\sqrt{1+\lambda}
    -
    \frac{\lambda\mu_0}{\sigma},
\]
and, for $x>0$, let
\[
    p(x)
    :=
    \Phi\left(b_{\mathrm{pr}}-1+\frac{x}{\sigma}\right)
    \left(1-\exp\left(-\frac{2x}{\sigma}\right)\right).
\]
Since $p(x)\in(0,1)$ for every $x>0$, the constants
\[
    C_0
    :=
    \sum_{i=1}^{K}\log\frac{1}{p(\varepsilon_i)}
\]
and
\[
    C_{a,\beta}
    :=
    \log\frac{1}{p(\beta)}
    +
    \sum_{i\ne a}\log\frac{1}{p(\varepsilon_i)},
    \qquad a\in[K],\ \beta>0,
\]
are finite. They depend on the fixed parameters, but not on $T$ or
the algorithm. We do not require $C_{a,\beta}$ to remain bounded as
$\beta\downarrow0$.

\begin{proposition}[Finite-budget search and certification bounds]
\label{prop:finite-budget-lower-bounds}
For every algorithm $\mathcal A$ described in
Section~\ref{sec:one-positive-families} and every integer $T\ge K$,
\begin{equation}
    \max_{a\in[K]}
    \mathbb P_{\mu^{(a)}}^{\mathcal A}
    \bigl(\mathrm{Succ}_T^c\bigr)
    \ge
    \exp\left[
        -\left(
            \sqrt{\frac{T}{2S}}+\sqrt{C_0}
        \right)^2
    \right].
    \label{eq:app-finite-search-bound}
\end{equation}
Moreover, for every $a\in[K]$ and every $\beta>0$,
\begin{equation}
    \mathbb P_{\mu^{(a)}}^{\mathcal A}
    \bigl(\mathrm{Succ}_T^c\bigr)
    \ge
    \exp\left[
        -\left(
            \frac{\Delta+\beta}{\sigma}\sqrt{\frac{T}{2}}
            +
            \sqrt{C_{a,\beta}}
        \right)^2
    \right].
    \label{eq:app-finite-certification-bound}
\end{equation}
\end{proposition}

\begin{proof}
Fix an algorithm $\mathcal A$ and a budget $T\ge K$.
We represent the observations by independent infinite sequences
$(X_{i,n})_{n\ge1}$, one for each arm. Sampling arm $i$ reveals the
next unused observation in its sequence. The algorithm's internal
randomization is represented by a random seed independent of these
sequences.

We write $\mathbb P_\nu$ for the probability measure under mean
vector $\nu$, and $P_\nu^T$ for the distribution of the observed
history up to time $T$, including the algorithm's internal
randomization. We suppress the dependence on $\mathcal A$ in this
notation.

\paragraph{Step 1: An event that guarantees posterior failure.}

Consider an auxiliary instance
\[
    \nu=(-v_1,\ldots,-v_K),
    \qquad v_i>0.
\]
For each arm $i$, define
\[
    Y_{i,m}
    :=
    \sum_{s=2}^{m+1}\frac{X_{i,s}}{\sigma},
    \qquad m\ge0,
\]
where $Y_{i,0}=0$.
Under $\mathbb P_\nu$, the increments of $Y_{i,m}$ are independent
Gaussian variables with mean $-v_i/\sigma$ and variance one.
Therefore,
\[
    \left(
        \exp\left(\frac{2v_i}{\sigma}Y_{i,m}\right)
    \right)_{m\ge0}
\]
is a nonnegative martingale with initial value one.
The maximal inequality for nonnegative martingales gives
\[
    \mathbb P_\nu\left(
        \exists m\ge0:\ Y_{i,m}\ge1
    \right)
    \le
    \exp\left(-\frac{2v_i}{\sigma}\right).
\]
To obtain the bound over all $m\ge0$, apply the inequality over
each finite range and then take an increasing union.

Define
\[
    E
    :=
    \bigcap_{i=1}^{K}
    \left(
        \left\{
            \frac{X_{i,1}}{\sigma}\le b_{\mathrm{pr}}-1
        \right\}
        \cap
        \bigcap_{m\ge0}\{Y_{i,m}<1\}
    \right).
\]
The first observation of each arm is independent of its subsequent
observations, and
\[
    \mathbb P_\nu\left(
        \frac{X_{i,1}}{\sigma}\le b_{\mathrm{pr}}-1
    \right)
    =
    \Phi\left(b_{\mathrm{pr}}-1+\frac{v_i}{\sigma}\right).
\]
Combining the preceding bounds and using independence across arms,
we obtain
\[
    \mathbb P_\nu(E)
    \ge
    \prod_{i=1}^{K}p(v_i)
    >0.
\]
The event $E$ is defined using the full observation sequences,
without reference to the algorithm or the budget. Its probability
therefore does not depend on either.

On $E$, for every arm $i$ and every $n\ge1$,
\[
    \sum_{s=1}^{n}\frac{X_{i,s}}{\sigma}
    =
    \frac{X_{i,1}}{\sigma}+Y_{i,n-1}
    <
    b_{\mathrm{pr}}
    \le
    z\sqrt{n+\lambda}-\frac{\lambda\mu_0}{\sigma}.
\]
The last inequality follows from $z>0$ and $n\ge1$.
Hence $Z_{i,n}<z$ for every $i$ and every $n\ge1$.
Since the algorithm samples every arm at least once,
\eqref{eq:success-zmax} implies that posterior failure occurs on $E$,
regardless of how the algorithm allocates its observations.

We now consider the distribution of the observed history under
$\nu$, conditional on $E$. Specifically, for any event $B$ determined
by the observed history, define
\[
    \mathsf Q_T(B)
    :=
    \mathbb P_\nu(B\mid E)
    =
    \frac{\mathbb P_\nu(B\cap E)}{\mathbb P_\nu(E)}.
\]
Thus, $\mathsf Q_T$ describes the observed history when we restrict
attention to realizations in which $E$ occurs. The sampling rule
remains unchanged, and the algorithm is not given any information
about $E$. Although $E$ may involve observations that the algorithm
never collects, this conditional distribution is well defined
because $E$ and the observed history are defined on the same
probability space.

Under this conditional distribution,
\[
    \mathsf Q_T(\mathrm{Succ}_T^c)=1.
\]
Moreover,
\[
    \frac{d\mathsf Q_T}{dP_\nu^T}
    =
    \frac{\mathbb P_\nu(E\mid\mathcal F_T)}
         {\mathbb P_\nu(E)}
    \le
    \frac{1}{\mathbb P_\nu(E)}.
\]
Writing $\operatorname{KL}$ for relative entropy, it follows that
\begin{equation}
    D_T
    :=
    \operatorname{KL}(\mathsf Q_T\Vert P_\nu^T)
    \le
    \log\frac{1}{\mathbb P_\nu(E)}
    \le
    \sum_{i=1}^{K}\log\frac{1}{p(v_i)}.
    \label{eq:app-conditioning-entropy}
\end{equation}
Although $\mathsf Q_T$ depends on the algorithm and the budget, the
upper bound in~\eqref{eq:app-conditioning-entropy} does not.

\paragraph{Step 2: Changing the mean of one arm.}

Suppose that $\nu'$ differs from $\nu$ only at arm $a$, with
\[
    \nu'_a-\nu_a=g\sigma,
    \qquad g>0.
\]
Define
\[
    M_a(t)
    :=
    \frac{1}{\sigma}
    \sum_{s=1}^{N_a(t)}(X_{a,s}-\nu_a),
    \qquad
    n_a:=\mathbb E_{\mathsf Q_T}[N_a(T)].
\]
Because the same algorithm is used under $\nu$ and $\nu'$, the
sampling probabilities cancel in the likelihood ratio. The
Gaussian observation densities therefore give
\begin{equation}
    \log\frac{dP_\nu^T}{dP_{\nu'}^T}
    =
    \frac{g^2}{2}N_a(T)-gM_a(T).
    \label{eq:app-gaussian-history-likelihood}
\end{equation}

To control the expectation of $M_a(T)$ under $\mathsf Q_T$, first
work under the original probability measure $\mathbb P_\nu$.
For every $\theta>0$, the process
\[
    \left(
        \exp\left(
            -\theta M_a(t)
            -\frac{\theta^2}{2}N_a(t)
        \right)
    \right)_{t=0}^{T}
\]
is a mean-one martingale. Indeed, when arm $a$ is selected, its
next centered, standardized observation is conditionally standard
Gaussian; otherwise, the process does not change. Consequently,
\[
    \mathbb E_{P_\nu^T}\left[
        \exp\left(
            -\theta M_a(T)
            -\frac{\theta^2}{2}N_a(T)
        \right)
    \right]
    =1.
\]
This identity is used under $P_\nu^T$, not under $\mathsf Q_T$.
After conditioning on $E$, the observations need not remain
independent or Gaussian.

The expectations used below are finite. Indeed, under
$\mathbb P_\nu$, the processes $M_a(t)$ and
$M_a(t)^2-N_a(t)$ are martingales, so
\[
    \mathbb E_{P_\nu^T}[M_a(T)^2]
    =
    \mathbb E_{P_\nu^T}[N_a(T)]
    \le T.
\]
Together with the density bound from Step 1, this gives
\[
    \mathbb E_{\mathsf Q_T}[M_a(T)^2]
    \le
    \frac{T}{\mathbb P_\nu(E)}
    <\infty.
\]
Since $N_a(T)\le T$, all expectations in the following entropy
calculations are therefore finite.

The entropy inequality gives
\[
    \mathbb E_{\mathsf Q_T}[W]
    \le
    D_T+\log\mathbb E_{P_\nu^T}[e^W].
\]
For completeness, this follows from nonnegativity of relative
entropy with respect to the probability measure having density
$e^W/\mathbb E_{P_\nu^T}[e^W]$ relative to $P_\nu^T$.
Applying the inequality with
\[
    W
    =
    -\theta M_a(T)-\frac{\theta^2}{2}N_a(T)
\]
yields
\[
    -\mathbb E_{\mathsf Q_T}[M_a(T)]
    \le
    \frac{D_T}{\theta}+\frac{\theta n_a}{2}.
\]
Optimizing over $\theta>0$, we obtain
\[
    -\mathbb E_{\mathsf Q_T}[M_a(T)]
    \le
    \sqrt{2n_aD_T}.
\]
When $D_T=0$, the same conclusion follows by letting
$\theta\downarrow0$.

Combining this bound with
\eqref{eq:app-gaussian-history-likelihood}, we have
\begin{align*}
    \operatorname{KL}(\mathsf Q_T\Vert P_{\nu'}^T)
    &=
    D_T+
    \mathbb E_{\mathsf Q_T}\left[
        \log\frac{dP_\nu^T}{dP_{\nu'}^T}
    \right]\\
    &=
    D_T+\frac{g^2n_a}{2}
    -g\mathbb E_{\mathsf Q_T}[M_a(T)]\\
    &\le
    D_T+\frac{g^2n_a}{2}+g\sqrt{2n_aD_T}\\
    &=
    \left(
        g\sqrt{\frac{n_a}{2}}+\sqrt{D_T}
    \right)^2.
\end{align*}

On the other hand, $\mathsf Q_T(\mathrm{Succ}_T^c)=1$ implies
\[
    \operatorname{KL}(\mathsf Q_T\Vert P_{\nu'}^T)
    \ge
    \log\frac{1}{P_{\nu'}^T(\mathrm{Succ}_T^c)}.
\]
To see this, note that the finite-history distributions $P_\nu^T$
and $P_{\nu'}^T$ are mutually absolutely continuous, so
$P_{\nu'}^T(\mathrm{Succ}_T^c)>0$. We can therefore write
\begin{align*}
    \operatorname{KL}(\mathsf Q_T\Vert P_{\nu'}^T)
    &=
    \operatorname{KL}\left(
        \mathsf Q_T
        \,\middle\Vert\,
        P_{\nu'}^T(\,\cdot\,\mid\mathrm{Succ}_T^c)
    \right)\\
    &\quad+
    \log\frac{1}{P_{\nu'}^T(\mathrm{Succ}_T^c)},
\end{align*}
and the first term is nonnegative.

Combining these inequalities gives
\begin{equation}
    \mathbb P_{\nu'}^{\mathcal A}(\mathrm{Succ}_T^c)
    \ge
    \exp\left[
        -\left(
            g\sqrt{\frac{n_a}{2}}+\sqrt{D_T}
        \right)^2
    \right].
    \label{eq:app-one-arm-change-of-measure}
\end{equation}

\paragraph{Step 3: The search bound.}

Take
\[
    \nu=(-\varepsilon_1,\ldots,-\varepsilon_K).
\]
By~\eqref{eq:app-conditioning-entropy}, the distribution
$\mathsf Q_T$ constructed in Step 1 satisfies $D_T\le C_0$.
For each $i\in[K]$, let
\[
    n_i
    :=
    \mathbb E_{\mathsf Q_T}[N_i(T)]
    =
    \mathbb E_\nu[N_i(T)\mid E],
    \qquad
    w_i:=\frac{\sigma^2}{(\Delta+\varepsilon_i)^2}.
\]
The algorithm collects exactly $T$ observations on every
realization, including those in $E$. Hence
\[
    \sum_{i=1}^{K}n_i
    =
    \mathbb E_\nu\left[
        \sum_{i=1}^{K}N_i(T)
        \,\middle|\,
        E
    \right]
    =T.
\]
Since $\sum_iw_i=S$, we obtain
\[
    \sum_{i=1}^{K}
    \frac{w_i}{S}\frac{n_i}{w_i}
    =
    \frac{T}{S}.
\]
Therefore, at least one location $a\in[K]$ satisfies
\[
    \frac{n_a}{w_a}\le\frac{T}{S}.
\]
This location may depend on $\mathcal A$ and $T$.

Passing from $\nu$ to $\mu^{(a)}$ changes only the mean of arm $a$,
from $-\varepsilon_a$ to $\Delta$. Its standardized mean difference
is
\[
    g_a
    =
    \frac{\Delta+\varepsilon_a}{\sigma}
    =
    \frac{1}{\sqrt{w_a}}.
\]
Applying~\eqref{eq:app-one-arm-change-of-measure}, we obtain
\begin{align*}
    \mathbb P_{\mu^{(a)}}^{\mathcal A}
    (\mathrm{Succ}_T^c)
    &\ge
    \exp\left[
        -\left(
            \sqrt{\frac{n_a}{2w_a}}+\sqrt{D_T}
        \right)^2
    \right]\\
    &\ge
    \exp\left[
        -\left(
            \sqrt{\frac{T}{2S}}+\sqrt{C_0}
        \right)^2
    \right].
\end{align*}
Taking the maximum over locations proves
\eqref{eq:app-finite-search-bound}.

\paragraph{Step 4: The certification bound.}

Fix any $a\in[K]$ and $\beta>0$, and consider the auxiliary instance
\[
    \nu_i^{(a,\beta)}
    :=
    \begin{cases}
        -\beta, & i=a,\\
        -\varepsilon_i, & i\ne a.
    \end{cases}
\]
All its means are negative. The construction in Step 1 and
\eqref{eq:app-conditioning-entropy} therefore give a conditional
distribution $\mathsf Q_T$ satisfying
\[
    D_T\le C_{a,\beta}.
\]
Passing from $\nu^{(a,\beta)}$ to $\mu^{(a)}$ changes only the mean
of arm $a$, from $-\beta$ to $\Delta$. Thus,
\[
    g=\frac{\Delta+\beta}{\sigma},
    \qquad
    n_a=\mathbb E_{\mathsf Q_T}[N_a(T)]\le T.
\]
Substituting these bounds into
\eqref{eq:app-one-arm-change-of-measure} proves
\eqref{eq:app-finite-certification-bound}.
Unlike the search bound, this bound holds for every fixed
location $a$.
\end{proof}

\subsection{Proof of the asymptotic lower bound}
\label{app:asymptotic-lower-bound}

\begin{proof}[Proof of Theorem~\ref{thm:one-positive-lower-bound}]
The constants in
Proposition~\ref{prop:finite-budget-lower-bounds}
are independent of the algorithm and the budget. We may therefore
take the infimum over algorithms in the finite-budget inequalities
before taking any limit. In particular, the argument allows the
algorithm to depend on $T$.

The search bound~\eqref{eq:app-finite-search-bound} gives
\[
    r_T^\star(\Delta,\varepsilon)
    \ge
    \exp\left[
        -\left(
            \sqrt{\frac{T}{2S}}+\sqrt{C_0}
        \right)^2
    \right].
\]
It follows that
\[
    -\frac{1}{T}\log r_T^\star(\Delta,\varepsilon)
    \le
    \left(
        \frac{1}{\sqrt{2S}}+\sqrt{\frac{C_0}{T}}
    \right)^2.
\]
Since $C_0$ is fixed as $T\to\infty$,
\[
    \limsup_{T\to\infty}
    -\frac{1}{T}\log r_T^\star(\Delta,\varepsilon)
    \le
    \frac{1}{2S}.
\]

Next, fix any $a\in[K]$ and $\beta>0$.
For every algorithm, its maximum failure probability over the
family is at least its failure probability at $\mu^{(a)}$.
Taking the infimum over algorithms in
\eqref{eq:app-finite-certification-bound} therefore yields
\[
    r_T^\star(\Delta,\varepsilon)
    \ge
    \exp\left[
        -\left(
            \frac{\Delta+\beta}{\sigma}\sqrt{\frac{T}{2}}
            +
            \sqrt{C_{a,\beta}}
        \right)^2
    \right].
\]
For this fixed $\beta$, the constant $C_{a,\beta}$ is finite and
independent of $T$. Consequently,
\[
    \limsup_{T\to\infty}
    -\frac{1}{T}\log r_T^\star(\Delta,\varepsilon)
    \le
    \frac{(\Delta+\beta)^2}{2\sigma^2}.
\]
This inequality holds for every $\beta>0$. Letting
$\beta\downarrow0$ after taking the limit in $T$ gives
\[
    \limsup_{T\to\infty}
    -\frac{1}{T}\log r_T^\star(\Delta,\varepsilon)
    \le
    \frac{\Delta^2}{2\sigma^2}
    =
    \frac{1}{2h}.
\]
The order of limits ensures that no bound on $C_{a,\beta}$ uniform
in $\beta$ is needed.

Combining the search and certification bounds, we conclude that
\[
    \limsup_{T\to\infty}
    -\frac{1}{T}\log r_T^\star(\Delta,\varepsilon)
    \le
    \min\left\{\frac{1}{2S},\frac{1}{2h}\right\}
    =
    \frac{1}{2\max\{h,S\}}.
\]
This proves~\eqref{eq:one-positive-lower-bound}.
\end{proof}

\subsection{Sample-complexity consequence}
\label{app:lower-bound-sample-complexity}

We now prove~\eqref{eq:one-positive-sample-lower-bound}.
For each target failure probability $\eta$, suppose that an
algorithm $\mathcal A_\eta$ uses a budget $T_\eta\ge K$ and satisfies
\[
    \max_{a\in[K]}
    \mathbb P_{\mu^{(a)}}^{\mathcal A_\eta}
    \bigl(\mathrm{Succ}_{T_\eta}^c\bigr)
    \le
    \eta.
\]
Both the algorithm and its budget may depend on $\eta$.
By the definition of the minimax failure probability,
\[
    r_{T_\eta}^\star(\Delta,\varepsilon)\le\eta.
\]

First, this guarantee requires $T_\eta\to\infty$ as
$\eta\downarrow0$.
Indeed, for any fixed integer $M\ge K$,
\eqref{eq:app-finite-search-bound} implies
\[
    r_T^\star(\Delta,\varepsilon)
    \ge
    \exp\left[
        -\left(
            \sqrt{\frac{M}{2S}}+\sqrt{C_0}
        \right)^2
    \right]
    >0,
    \qquad K\le T\le M.
\]
Thus, for sufficiently small $\eta$, no budget at most $M$ can
achieve failure probability at most $\eta$ on every instance in
the family.

Fix $\xi>0$.
Theorem~\ref{thm:one-positive-lower-bound} implies that, for all
sufficiently large $T$,
\[
    -\frac{1}{T}\log r_T^\star(\Delta,\varepsilon)
    \le
    \frac{1+\xi}{2\max\{h,S\}}.
\]
Since $T_\eta\to\infty$, for all sufficiently small $\eta$ we have
\[
    \log\frac{1}{\eta}
    \le
    -\log r_{T_\eta}^\star(\Delta,\varepsilon)
    \le
    \frac{1+\xi}{2\max\{h,S\}}\,T_\eta.
\]
Therefore,
\[
    \liminf_{\eta\downarrow0}
    \frac{T_\eta}{\log(1/\eta)}
    \ge
    \frac{2\max\{h,S\}}{1+\xi}.
\]
Letting $\xi\downarrow0$ yields
\[
    \liminf_{\eta\downarrow0}
    \frac{T_\eta}{\log(1/\eta)}
    \ge
    2\max\{h,S\}.
\]
Equivalently, any such budget satisfies
\[
    T_\eta
    \ge
    \bigl(2-o_\eta(1)\bigr)\max\{h,S\}
    \log\frac{1}{\eta},
    \qquad \eta\downarrow0.
\]
Here $o_\eta(1)\to0$ as $\eta\downarrow0$, with $K$, $\Delta$,
$\varepsilon$, $\sigma^2$, $(\mu_0,\sigma_0^2)$, and $\delta$ held
fixed. This proves~\eqref{eq:one-positive-sample-lower-bound}.

\subsection{Proof of Lemma~\ref{lem:one-positive-complexity}}
\label{app:one-positive-complexity}

\begin{proof}
Fix $a\in[K]$. For the instance $\mu^{(a)}$, we have
$\mu_\star=\Delta$, and arm $a$ is the unique positive arm. Therefore,
\[
    H_{\mathrm{pos}}(\mu^{(a)})
    =
    \frac{K\sigma^2}{\Delta^2}
    =
    Kh.
\]

For $H_{\mathrm{cap}}$, the contribution of arm $a$ is
$\sigma^2/\Delta^2=h$. For every $i\ne a$,
\[
    \max\{\mu_\star,\mu_\star-\mu_i^{(a)}\}
    =
    \max\{\Delta,\Delta+\varepsilon_i\}
    =
    \Delta+\varepsilon_i.
\]
Hence
\[
    H_{\mathrm{cap}}(\mu^{(a)})
    =
    h
    +
    \sum_{i\ne a}
    \frac{\sigma^2}{(\Delta+\varepsilon_i)^2}
    =
    h+S-\frac{\sigma^2}{(\Delta+\varepsilon_a)^2}.
\]
Since $\Delta+\varepsilon_i>\Delta$ for every $i$,
each term in $H_{\mathrm{cap}}(\mu^{(a)})$ is at most $h$.
Thus
\[
    H_{\mathrm{cap}}(\mu^{(a)})
    \le
    Kh
    =
    H_{\mathrm{pos}}(\mu^{(a)}),
\]
and therefore
\[
    H_{\mathrm{cert}}(\mu^{(a)})
    =
    H_{\mathrm{cap}}(\mu^{(a)}).
\]

Now,
\[
    \frac{\sigma^2}{(\Delta+\varepsilon_a)^2}
    \le h
    \qquad\text{and}\qquad
    \frac{\sigma^2}{(\Delta+\varepsilon_a)^2}
    \le S.
\]
It follows that
\[
    H_{\mathrm{cert}}(\mu^{(a)})
    =
    h+S-\frac{\sigma^2}{(\Delta+\varepsilon_a)^2}
    \ge h
\]
and
\[
    H_{\mathrm{cert}}(\mu^{(a)})
    \ge S.
\]
Therefore,
\[
    H_{\mathrm{cert}}(\mu^{(a)})
    \ge
    \max\{h,S\}.
\]
The upper bound follows directly from
\[
    H_{\mathrm{cert}}(\mu^{(a)})
    \le
    h+S
    \le
    2\max\{h,S\}.
\]
This proves~\eqref{eq:one-positive-complexity-comparison}.
\end{proof}

\section{Equal positive means}
\label{app:equal-positive-exponent}

We complement the lower bound for one-positive-arm families with
a result for the ``many positive arms'' example from
Section~\ref{sec:htg-interpretation}, in which
\[
    \mu=(a,\ldots,a)\in\mathbb R^K,
    \qquad a>0.
\]
We show that every admissible algorithm, as defined in
Section~\ref{sec:one-positive-families}, has the same exact
asymptotic posterior-failure exponent on this instance.
This includes both HTG and uniform sampling.

Throughout this appendix, $K$, $a$, the observation variance
$\sigma^2$, the prior parameters $(\mu_0,\sigma_0^2)$, and the
posterior credibility tolerance $\delta$ are fixed.
Recall that
\[
    \lambda=\frac{\sigma^2}{\sigma_0^2},
    \qquad
    z=\Phi^{-1}(1-\delta)>0.
\]

\begin{proposition}[Exact exponent for equal positive means]
\label{prop:equal-positive-exponent}
Let $\mu=(a,\ldots,a)$ with $a>0$.
For every sequence of admissible algorithms
$(\mathcal A_T)_{T\ge K}$,
\begin{equation}
    \lim_{T\to\infty}
    -\frac{1}{T}
    \log
    \mathbb P_\mu^{\mathcal A_T}
    \bigl(\mathrm{Succ}_T^c\bigr)
    =
    \frac{a^2}{2\sigma^2}
    =
    \frac{1}{2H_{\mathrm{pos}}(\mu)}.
    \label{eq:equal-positive-exponent}
\end{equation}
The algorithms may be randomized and may depend on $T$.
\end{proposition}

\begin{proof}
Fix a budget $T\ge K$ and an admissible algorithm $\mathcal A_T$.
Write $\mathbb P_\theta$ for the probability measure under mean
vector $\theta$, and $P_\theta^T$ for the distribution of the
observed history up to time $T$, including the algorithm's internal
randomization. We suppress the dependence on $\mathcal A_T$
in this notation. Let $Y_t$ denote the observation collected at
round $t$.

We establish finite-budget lower and upper bounds that hold
uniformly over the choice of $\mathcal A_T$.

\paragraph{Lower bound on the failure probability.}

Fix $\beta>0$ and consider the auxiliary instance
\[
    \nu=(-\beta,\ldots,-\beta).
\]
The same algorithm, with the same parameter settings, is used
under $\nu$ and $\mu$.

Apply the event construction from the proof of
Proposition~\ref{prop:finite-budget-lower-bounds} to $\nu$.
It gives an event $E$, defined using the full observation sequences,
such that
\[
    \mathbb P_\nu(E)\ge p(\beta)^K>0
\]
and, on $E$,
\[
    Z_{i,n}<z
    \qquad
    \text{for every }i\in[K]\text{ and }n\ge1.
\]
Here $p$ is the function defined in
Appendix~\ref{app:finite-budget-lower-bounds}.
The event and its probability do not depend on $T$ or the algorithm.
Since every arm is sampled at least once,
\eqref{eq:success-zmax} implies that $E$ guarantees posterior failure.

Let $\mathsf Q_T$ be the distribution of the observed history
under $\nu$, conditional on $E$. Specifically, for any event $B$
determined by the observed history, define
\[
    \mathsf Q_T(B)
    :=
    \mathbb P_\nu(B\mid E).
\]
Thus, we restrict attention to realizations in which $E$ occurs,
without changing the sampling rule. The conditioning argument in
Appendix~\ref{app:finite-budget-lower-bounds} gives
\[
    \mathsf Q_T(\mathrm{Succ}_T^c)=1,
    \qquad
    D_T
    :=
    \operatorname{KL}(\mathsf Q_T\Vert P_\nu^T)
    \le C_\beta,
\]
where
\[
    C_\beta
    :=
    K\log\frac{1}{p(\beta)}
    <\infty.
\]
For each fixed $\beta>0$, the constant $C_\beta$ is independent
of $T$ and the algorithm. We also have the density bound
\[
    \frac{d\mathsf Q_T}{dP_\nu^T}
    \le
    \frac{1}{\mathbb P_\nu(E)}.
\]

Passing from $\nu$ to $\mu$ changes every arm mean by the same
amount $a+\beta$. Define
\[
    g:=\frac{a+\beta}{\sigma},
    \qquad
    M_T:=\sum_{t=1}^{T}\frac{Y_t+\beta}{\sigma}.
\]
Because the algorithm selects each arm using only previous
observations and its internal randomization, its sampling
probabilities cancel in the likelihood ratio. The Gaussian
observation densities give
\[
    \log\frac{dP_\nu^T}{dP_\mu^T}
    =
    \frac{g^2T}{2}-gM_T.
\]
Unlike the one-arm comparison in
Appendix~\ref{app:finite-budget-lower-bounds}, every observation
contributes the same mean difference, regardless of which arm
is selected.

Under $\mathbb P_\nu$, conditional on the preceding history and
the selected arm, each $(Y_t+\beta)/\sigma$ is standard Gaussian.
Consequently, for every $\theta>0$,
\[
    \mathbb E_{P_\nu^T}\left[
        \exp\left(
            -\theta M_T-\frac{\theta^2T}{2}
        \right)
    \right]
    =1.
\]
Also, $\mathbb E_{P_\nu^T}[M_T^2]=T$.
The density bound above therefore implies
\[
    \mathbb E_{\mathsf Q_T}[M_T^2]
    \le
    \frac{T}{\mathbb P_\nu(E)}
    <\infty.
\]
Thus, the expectations used below are finite.
The Gaussian identities are used under $P_\nu^T$, not under
the conditional distribution $\mathsf Q_T$.

Applying the entropy inequality from the proof of
Proposition~\ref{prop:finite-budget-lower-bounds} with
\[
    W=-\theta M_T-\frac{\theta^2T}{2}
\]
gives
\[
    -\mathbb E_{\mathsf Q_T}[M_T]
    \le
    \frac{D_T}{\theta}+\frac{\theta T}{2}.
\]
Optimizing over $\theta>0$ yields
\[
    -\mathbb E_{\mathsf Q_T}[M_T]
    \le
    \sqrt{2TD_T}.
\]
If $D_T=0$, the same conclusion follows by letting
$\theta\downarrow0$.

It follows that
\begin{align*}
    \operatorname{KL}(\mathsf Q_T\Vert P_\mu^T)
    &=
    D_T+
    \mathbb E_{\mathsf Q_T}\left[
        \log\frac{dP_\nu^T}{dP_\mu^T}
    \right]\\
    &=
    D_T+\frac{g^2T}{2}
    -g\mathbb E_{\mathsf Q_T}[M_T]\\
    &\le
    D_T+\frac{g^2T}{2}+g\sqrt{2TD_T}\\
    &=
    \left(
        g\sqrt{\frac{T}{2}}+\sqrt{D_T}
    \right)^2\\
    &\le
    \left(
        g\sqrt{\frac{T}{2}}+\sqrt{C_\beta}
    \right)^2.
\end{align*}
Since $\mathsf Q_T(\mathrm{Succ}_T^c)=1$, the event inequality
for relative entropy used in the preceding appendix gives
\[
    \log\frac{1}{P_\mu^T(\mathrm{Succ}_T^c)}
    \le
    \operatorname{KL}(\mathsf Q_T\Vert P_\mu^T).
\]
Therefore,
\begin{equation}
    \mathbb P_\mu^{\mathcal A_T}
    \bigl(\mathrm{Succ}_T^c\bigr)
    \ge
    \exp\left[
        -\left(
            \frac{a+\beta}{\sigma}\sqrt{\frac{T}{2}}
            +
            \sqrt{C_\beta}
        \right)^2
    \right].
    \label{eq:equal-positive-finite-lower-bound}
\end{equation}

For each fixed $\beta>0$, dividing the negative logarithm by $T$
and taking the limit superior gives
\[
    \limsup_{T\to\infty}
    -\frac{1}{T}
    \log
    \mathbb P_\mu^{\mathcal A_T}
    \bigl(\mathrm{Succ}_T^c\bigr)
    \le
    \frac{(a+\beta)^2}{2\sigma^2}.
\]
Letting $\beta\downarrow0$ afterwards yields
\[
    \limsup_{T\to\infty}
    -\frac{1}{T}
    \log
    \mathbb P_\mu^{\mathcal A_T}
    \bigl(\mathrm{Succ}_T^c\bigr)
    \le
    \frac{a^2}{2\sigma^2}.
\]
No bound on $C_\beta$ uniform in $\beta$ is needed.

\paragraph{Upper bound on the failure probability.}

Under $\mu=(a,\ldots,a)$, the observations collected in chronological
order satisfy
\[
    Y_1,\ldots,Y_T
    \stackrel{\mathrm{iid}}{\sim}
    \mathcal N(a,\sigma^2).
\]
Indeed, conditional on the preceding history and the selected arm,
the next observation always has distribution
$\mathcal N(a,\sigma^2)$. Thus, adaptive arm selection does not change
the joint distribution of the observations in chronological order.

On $\mathrm{Succ}_T^c$, equation~\eqref{eq:success-zmax} implies
\[
    \lambda\mu_0
    +
    \sum_{s=1}^{N_i(T)}X_{i,s}
    <
    z\sigma\sqrt{N_i(T)+\lambda}
    \qquad
    \text{for every }i\in[K].
\]
Summing over arms gives
\[
    \sum_{t=1}^{T}Y_t
    <
    z\sigma\sum_{i=1}^{K}\sqrt{N_i(T)+\lambda}
    -
    K\lambda\mu_0.
\]
By the Cauchy--Schwarz inequality and $\sum_iN_i(T)=T$,
\[
    \sum_{i=1}^{K}\sqrt{N_i(T)+\lambda}
    \le
    \sqrt{
        K\sum_{i=1}^{K}(N_i(T)+\lambda)
    }
    =
    \sqrt{K(T+K\lambda)}.
\]
Define
\[
    B_T
    :=
    z\sigma\sqrt{K(T+K\lambda)}-K\lambda\mu_0.
\]
We have shown that
\[
    \mathrm{Succ}_T^c
    \subseteq
    \left\{
        \sum_{t=1}^{T}Y_t<B_T
    \right\}.
\]
Since $\sum_{t=1}^{T}Y_t\sim\mathcal N(aT,\sigma^2T)$,
\begin{equation}
    \mathbb P_\mu^{\mathcal A_T}
    \bigl(\mathrm{Succ}_T^c\bigr)
    \le
    \Phi\left(
        \frac{B_T-aT}{\sigma\sqrt T}
    \right)
    \le
    \exp\left(
        -\frac{(aT-B_T)_+^2}{2\sigma^2T}
    \right).
    \label{eq:equal-positive-finite-upper-bound}
\end{equation}
With all parameters fixed, $B_T/T\to0$. Therefore,
\[
    \liminf_{T\to\infty}
    -\frac{1}{T}
    \log
    \mathbb P_\mu^{\mathcal A_T}
    \bigl(\mathrm{Succ}_T^c\bigr)
    \ge
    \frac{a^2}{2\sigma^2}.
\]
Combining this with the preceding upper bound on the limit superior
proves~\eqref{eq:equal-positive-exponent}.
\end{proof}

\section{Uniform Sampling, Sequential Halving, and Separations}
\label{app:baselines}

This appendix specifies the baseline algorithms and proves the guarantees
in Section~\ref{sec:baseline-comparisons}. We then establish
sample-complexity separations from HTG on a family with one positive arm.

Throughout, we use the Gaussian observation model and prior from
Section~\ref{sec:problem}, with $\sigma,\sigma_0>0$ and
$\delta\in(0,1/2)$. Recall
$\lambda=\sigma^2/\sigma_0^2$, $z=\Phi^{-1}(1-\delta)$, and
$c=\sqrt{2}\,z+\sqrt{\lambda}|\mu_0|/\sigma$.
Every algorithm uses all observations collected from each arm to update
its posterior and makes the terminal recommendation
$J_T\in\arg\max_{i\in[K]}Z_{i,N_i(T)}$.
Thus, by~\eqref{eq:success-zmax},
\[
    \mathrm{Succ}_T^c
    =
    \left\{
        Z_{i,N_i(T)}<z
        \text{ for every }i\in[K]
    \right\}.
\]
All ties are resolved by a fixed deterministic rule. The guarantees below
assume $\mu_\star>0$, and we use $\sigma^2/0^2=+\infty$.

For the proofs, regard each arm as having an infinite sequence of potential
observations, and define
\[
    W_{i,n}
    :=
    \frac{1}{\sigma}
    \sum_{s=1}^{n}(X_{i,s}-\mu_i).
\]
At every deterministic $n\ge1$, $W_{i,n}\sim\mathcal N(0,n)$, independently
across arms. In particular,
\begin{equation}
\label{eq:app-score-noise}
    Z_{i,n}
    =
    \frac{
        n\mu_i/\sigma+\lambda\mu_0/\sigma+W_{i,n}
    }{\sqrt{n+\lambda}}.
\end{equation}

\subsection{Uniform sampling}
\label{app:uniform-sampling}

Algorithm~\ref{alg:uniform-baseline} allocates observations round-robin.
Its sample counts are deterministic: writing $n=\lfloor T/K\rfloor$,
\[
    n_i(T)
    =
    n+\mathbf{1}\{i\le T-Kn\},
    \qquad i\in[K].
\]

\begin{algorithm}[t]
\caption{Uniform sampling with posterior-based recommendation}
\label{alg:uniform-baseline}
\begin{algorithmic}[1]
\REQUIRE Integer budget $T\ge K$
\STATE Initialize the sample count of every arm to zero.
\FOR{$t=0,\ldots,T-1$}
    \STATE Select $A_{t+1}=1+(t\bmod K)$.
    \STATE Sample arm $A_{t+1}$ and update its sample count and posterior
    using all its observations.
\ENDFOR
\STATE \textbf{Output}
$J_T\in\arg\max_{i\in[K]}Z_{i,N_i(T)}$.
\end{algorithmic}
\end{algorithm}

\begin{proposition}[Uniform-sampling guarantees]
\label{prop:app-uniform-sampling}
Fix a mean vector $\mu$ with $\mu_\star>0$.
For every integer $T\ge K$,
\begin{equation}
\label{eq:app-uniform-product}
    \mathbb P_\mu^{\mathrm{Uniform}}
    \bigl(\mathrm{Succ}_T^c\bigr)
    =
    \prod_{i=1}^{K}
    \Phi\!\left(
        z\sqrt{1+\frac{\lambda}{n_i(T)}}
        -
        \frac{\mu_i}{\sigma}\sqrt{n_i(T)}
        -
        \frac{\lambda\mu_0}{\sigma\sqrt{n_i(T)}}
    \right).
\end{equation}
If $n=\lfloor T/K\rfloor\ge\lambda$, then
\begin{equation}
\label{eq:app-uniform-finite-budget}
    \mathbb P_\mu^{\mathrm{Uniform}}
    \bigl(\mathrm{Succ}_T^c\bigr)
    \le
    \exp\left[
        -\frac12
        \left(
            \frac{\mu_\star}{\sigma}\sqrt n-c
        \right)_+^2
    \right].
\end{equation}
Moreover, with $K$, $\mu$, $\sigma^2$, the prior parameters, and $\delta$
held fixed,
\begin{equation}
\label{eq:app-uniform-exponent}
    \lim_{T\to\infty}
    -\frac1T
    \log
    \mathbb P_\mu^{\mathrm{Uniform}}
    \bigl(\mathrm{Succ}_T^c\bigr)
    =
    \frac{\sum_{i=1}^{K}(\mu_i)_+^2}{2K\sigma^2}
    =
    \frac{1}{2H_{\mathrm{pos}}(\mu)}.
\end{equation}
\end{proposition}

\begin{proof}
By~\eqref{eq:app-score-noise}, for every deterministic count $n_i(T)$,
\[
    \mathbb P_\mu\!\left(Z_{i,n_i(T)}<z\right)
    =
    \Phi\!\left(
        z\sqrt{1+\frac{\lambda}{n_i(T)}}
        -
        \frac{\mu_i}{\sigma}\sqrt{n_i(T)}
        -
        \frac{\lambda\mu_0}{\sigma\sqrt{n_i(T)}}
    \right).
\]
The counts are deterministic and the observations are independent across
arms. Multiplying these probabilities proves
\eqref{eq:app-uniform-product}.

For the finite-budget bound, choose an arm $i_\star$ with
$\mu_{i_\star}=\mu_\star$. For every $m\ge\lambda$,
\[
    z\sqrt{1+\frac{\lambda}{m}}
    -
    \frac{\lambda\mu_0}{\sigma\sqrt m}
    \le
    \sqrt{2}\,z+\frac{\sqrt{\lambda}|\mu_0|}{\sigma}
    =
    c.
\]
Since $n_{i_\star}(T)\ge n\ge\lambda$, failure of every arm implies
\[
\begin{aligned}
    \mathbb P_\mu^{\mathrm{Uniform}}
    \bigl(\mathrm{Succ}_T^c\bigr)
    &\le
    \mathbb P_\mu\!\left(
        Z_{i_\star,n_{i_\star}(T)}<z
    \right)
    \\
    &\le
    \Phi\!\left(
        c-\frac{\mu_\star}{\sigma}\sqrt n
    \right).
\end{aligned}
\]
The inequality $\Phi(-x)\le\exp(-x_+^2/2)$ gives
\eqref{eq:app-uniform-finite-budget}.

For the asymptotic exponent, let $a_{i,T}$ denote the argument of $\Phi$
in the $i$th factor of~\eqref{eq:app-uniform-product}. Since
$n_i(T)/T\to1/K$,
\[
    a_{i,T}
    =
    -\frac{\mu_i}{\sigma}\sqrt{n_i(T)}
    +z+o(1).
\]
If $\mu_i>0$, Gaussian tail asymptotics give
\[
    -\log\Phi(a_{i,T})
    =
    \frac{\mu_i^2}{2\sigma^2}n_i(T)+o(T).
\]
If $\mu_i<0$, then $\Phi(a_{i,T})\to1$, whereas if $\mu_i=0$, then
$\Phi(a_{i,T})\to\Phi(z)>0$. Hence the nonpositive arms contribute
$o(T)$ to the negative logarithm of the product. Summing over the
positive arms and dividing by $T$ proves
\eqref{eq:app-uniform-exponent}.
\end{proof}

In particular, for any $\eta\in(0,1)$, a sufficient budget for
$\mathbb P_\mu^{\mathrm{Uniform}}(\mathrm{Succ}_T^c)\le\eta$ is
\begin{equation}
\label{eq:app-uniform-sufficient-budget}
    T
    \ge
    K\left\lceil
        \max\left\{
            1,\lambda,
            \frac{\sigma^2}{\mu_\star^2}
            \left(
                c+\sqrt{2\log\frac1\eta}
            \right)^2
        \right\}
    \right\rceil.
\end{equation}
This follows by substituting the resulting lower bound on
$n=\lfloor T/K\rfloor$ into
\eqref{eq:app-uniform-finite-budget}.

\subsection{Sequential halving}
\label{app:sequential-halving}

For $K\ge2$, we use the equal-stage-budget SH allocation of
\citet{karnin13almost}, with elimination based on cumulative posterior
$Z$-scores and the common all-arm terminal recommendation.
Recall that $Q-1=\lceil\log_2K\rceil$.
At each elimination stage, all active arms have equal cumulative sample
counts, so ranking their posterior $Z$-scores is equivalent to ranking
their cumulative empirical means.

Algorithm~\ref{alg:sh-baseline} specifies the allocation, including
rounding and leftover observations. For $K\le T<K(Q-1)$, it uses
round-robin sampling to ensure that every arm is observed at least once.
An arm eliminated during SH receives no further observations but remains
eligible for the terminal recommendation.

\begin{algorithm}[t]
\caption{Sequential halving with posterior-based elimination and recommendation}
\label{alg:sh-baseline}
\begin{algorithmic}[1]
\REQUIRE Integers $K\ge2$ and $T\ge K$
\STATE Initialize the sample count of every arm to zero.
\IF{$T<K(Q-1)$}
    \STATE Allocate all $T$ observations round-robin, updating each
    sampled arm's posterior using all its observations.
\ELSE
    \STATE Set $S_1=[K]$ and $t=0$.
    \FOR{$r=1,\ldots,Q-1$}
        \STATE Set $k_r=|S_r|$ and
        $q_r=\left\lfloor T/((Q-1)k_r)\right\rfloor$.
        \STATE Sample each arm in $S_r$ an additional $q_r$ times,
        in increasing arm order, updating its posterior after every
        observation.
        \STATE Set $t\leftarrow t+k_rq_r$.
        \STATE Let $S_{r+1}$ contain the $\lceil k_r/2\rceil$ arms in
        $S_r$ with the largest $Z_{i,N_i(t)}$.
    \ENDFOR
    \STATE Let $I$ be the unique arm in $S_Q$.
    \STATE Allocate the remaining $T-t$ observations to $I$, updating
    its posterior using all its observations.
\ENDIF
\STATE \textbf{Output}
$J_T\in\arg\max_{i\in[K]}Z_{i,N_i(T)}$.
\end{algorithmic}
\end{algorithm}

\paragraph{Budget accounting and possible sample counts.}
In the SH branch, the active-set sizes are deterministic:
$k_r=\lceil K/2^{r-1}\rceil$.
In particular, $k_{Q-1}=2$ and $k_Q=1$.
Since $k_rq_r\le T/(Q-1)$, the elimination stages use at most $T$
observations. Define
\begin{equation}
\label{eq:app-sh-counts}
    n_r:=\sum_{s=1}^{r}q_s,
    \qquad
    D:=T-\sum_{r=1}^{Q-1}k_rq_r.
\end{equation}
Here $D\ge0$ is the deterministic leftover budget, and $n_r$ is the
cumulative count of every arm active through stage $r$.
Consequently, every terminal count belongs to the deterministic set
\begin{equation}
\label{eq:app-sh-terminal-counts}
    \mathcal C_T
    :=
    \{n_1,\ldots,n_{Q-1},n_{Q-1}+D\},
    \qquad
    |\mathcal C_T|\le Q.
\end{equation}
The identity of the final survivor is random, but its terminal count
is always $n_{Q-1}+D$.

\begin{theorem}[Finite-budget guarantee for sequential halving]
\label{thm:app-sh-guarantee}
Let $K\ge2$ and $\mu_\star>0$, and recall $H_{\mathrm{SH}}(\mu)$
from~\eqref{eq:H-SH}. Suppose the integer budget satisfies
\begin{equation}
\label{eq:app-sh-initialization}
    T\ge2K(Q-1)\max\{1,\lambda\}.
\end{equation}
For every $R>0$ such that
\begin{equation}
\label{eq:app-sh-radius-budget}
    T>16H_{\mathrm{SH}}(\mu)(c+R)^2,
\end{equation}
Algorithm~\ref{alg:sh-baseline} satisfies
\begin{equation}
\label{eq:app-sh-failure-bound}
    \mathbb P_\mu^{\mathrm{SH}}
    \bigl(\mathrm{Succ}_T^c\bigr)
    \le
    \min\left\{
        1,\,
        6K(Q-1)e^{-R^2/2}
    \right\}.
\end{equation}
Consequently, for $\eta\in(0,1)$, conditions
\eqref{eq:app-sh-initialization} and
\begin{equation}
\label{eq:app-sh-confidence-budget}
    T
    >
    16H_{\mathrm{SH}}(\mu)
    \left(
        c+\sqrt{2\log\frac{6K(Q-1)}{\eta}}
    \right)^2
\end{equation}
imply
$\mathbb P_\mu^{\mathrm{SH}}(\mathrm{Succ}_T^c)\le\eta$.
\end{theorem}

\begin{proof}
Condition~\eqref{eq:app-sh-initialization} ensures that the SH branch
is used and that
\begin{equation}
\label{eq:app-sh-batch-lower-bound}
    q_r
    \ge
    \frac{T}{2(Q-1)k_r},
    \qquad
    q_1
    \ge
    \frac{T}{2K(Q-1)}
    \ge
    \max\{1,\lambda\}.
\end{equation}
In particular, every count in $\mathcal C_T$ is at least $\lambda$.

Consider the event
\[
    \mathcal E_R
    :=
    \bigcap_{i=1}^{K}
    \bigcap_{n\in\mathcal C_T}
    \left\{
        |W_{i,n}|\le R\sqrt n
    \right\}.
\]
The set $\mathcal C_T$ is deterministic. Gaussian tail bounds and a
union bound therefore give
\begin{equation}
\label{eq:app-sh-concentration}
    \mathbb P_\mu(\mathcal E_R^c)
    \le
    2KQe^{-R^2/2}
    \le
    6K(Q-1)e^{-R^2/2}.
\end{equation}
The event is defined on the potential observation sequences, so this
argument does not condition on an adaptively selected sample count.

Work on $\mathcal E_R$. For $n\in\mathcal C_T$, we have $n\ge\lambda$
and
\[
    z\sqrt{1+\frac{\lambda}{n}}
    -
    \frac{\lambda\mu_0}{\sigma\sqrt n}
    \le c.
\]
Together with~\eqref{eq:app-score-noise} and
$W_{i,n}\ge-R\sqrt n$, this yields
\begin{equation}
\label{eq:app-sh-certification}
    \frac{\mu_i}{\sigma}\sqrt n\ge c+R
    \quad\Longrightarrow\quad
    Z_{i,n}\ge z.
\end{equation}
We show that at least one arm satisfies this condition at its terminal
count.

\emph{Case 1:}
$H_{\mathrm{SH}}(\mu)=(Q-1)K\sigma^2/\mu_\star^2$.
Choose $i_\star$ with mean $\mu_\star$.
Regardless of whether it is eliminated, $N_{i_\star}(T)\ge q_1$.
By~\eqref{eq:app-sh-batch-lower-bound},
\[
    \frac{\mu_\star}{\sigma}
    \sqrt{N_{i_\star}(T)}
    \ge
    \sqrt{
        \frac{T\mu_\star^2}{2K(Q-1)\sigma^2}
    }
    =
    \sqrt{\frac{T}{2H_{\mathrm{SH}}(\mu)}}
    >
    c+R.
\]
Its terminal count belongs to $\mathcal C_T$, so
\eqref{eq:app-sh-certification} proves posterior success.

\emph{Case 2:}
$H_{\mathrm{SH}}(\mu)
=(Q-1)\sum_{i=1}^{K}\sigma^2/\mu_i^2<\infty$.
In this case, no arm has mean zero.
We first show that a positive arm survives every elimination stage.

Suppose, to the contrary, that stage $r$ eliminates all remaining
positive arms. Choose a positive arm $i\in S_r$.
Every retained arm $j\in S_{r+1}$ is negative and has posterior
score at least that of $i$. Since both arms have count $n_r$,
$\overline X_{j,n_r}\ge\overline X_{i,n_r}$.
On $\mathcal E_R$, this implies
\[
    \mu_j
    \ge
    \mu_i-\frac{2\sigma R}{\sqrt{n_r}}
    >
    -\frac{2\sigma R}{\sqrt{n_r}}.
\]
Thus each retained arm satisfies
$\sigma^2/\mu_j^2>n_r/(4R^2)\ge q_r/(4R^2)$.
There are at least $k_r/2$ such arms. Hence
\[
    \sum_{j:\mu_j<0}\frac{\sigma^2}{\mu_j^2}
    >
    \frac{k_rq_r}{8R^2}
    \ge
    \frac{T}{16(Q-1)R^2}.
\]
It follows that
$H_{\mathrm{SH}}(\mu)>T/(16R^2)$, contradicting
\eqref{eq:app-sh-radius-budget}. Therefore the final survivor $I$
has positive mean.

Since the last elimination stage has two active arms,
\[
    N_I(T)\ge q_{Q-1}\ge\frac{T}{4(Q-1)}.
\]
Moreover,
$\sigma^2/\mu_I^2
\le\sum_{i=1}^{K}\sigma^2/\mu_i^2
=H_{\mathrm{SH}}(\mu)/(Q-1)$.
Consequently,
\[
    \frac{\mu_I}{\sigma}\sqrt{N_I(T)}
    \ge
    \sqrt{\frac{T}{4H_{\mathrm{SH}}(\mu)}}
    >
    c+R.
\]
Applying~\eqref{eq:app-sh-certification} at
$N_I(T)=n_{Q-1}+D$ proves posterior success.

At least one of the two cases applies, because $\mu_\star>0$ makes
$H_{\mathrm{SH}}(\mu)$ finite. We have therefore shown
$\mathcal E_R\subseteq\mathrm{Succ}_T$.
Combining this inclusion with~\eqref{eq:app-sh-concentration} proves
\eqref{eq:app-sh-failure-bound}. Substituting
$R=\sqrt{2\log(6K(Q-1)/\eta)}$ proves the final assertion.
\end{proof}

For completeness, the comparison
in~\eqref{eq:baseline-complexity-comparison} follows directly from
the definitions:
\[
\begin{aligned}
    H_{\mathrm{HTG}}(\mu)
    &\le
    \min\left\{
        \sum_{i=1}^{K}\frac{\sigma^2}{\mu_i^2},
        H_{\mathrm{pos}}(\mu)
    \right\}
    \\
    &\le
    \min\left\{
        \sum_{i=1}^{K}\frac{\sigma^2}{\mu_i^2},
        \frac{K\sigma^2}{\mu_\star^2}
    \right\}
    =
    \frac{H_{\mathrm{SH}}(\mu)}{Q-1}.
\end{aligned}
\]
The second inequality uses
$\sum_i(\mu_i)_+^2\ge\mu_\star^2$.
The same definitions also give
$H_{\mathrm{HTG}}(\mu)\le H_{\mathrm{pos}}(\mu)$.
These inequalities compare the leading complexities in the guarantees;
they do not assert a finite-budget ordering of the actual failure
probabilities.

\subsection{Strict separations}
\label{app:baseline-separation}

We now establish separations in the actual sample budgets required for
posterior success. Fix $\sigma,\sigma_0>0$, $\delta\in(0,1/2)$, and a
centered prior $\mu_0=0$. Consider the family
\begin{equation}
\label{eq:app-separation-family}
    \mu_1=\frac{\sigma}{K},
    \qquad
    \mu_2=\cdots=\mu_K=-\sigma,
    \qquad
    K=2^\ell,\quad \ell\ge1.
\end{equation}
For a sampling algorithm $\mathcal A$, define
\begin{equation}
\label{eq:app-sample-complexity}
    \mathcal T_{\mathcal A}(\eta;\mu)
    :=
    \min\left\{
        T\in\mathbb N:
        T\ge K,\;
        \mathbb P_\mu^{\mathcal A}
        \bigl(\mathrm{Succ}_T^c\bigr)
        \le\eta
    \right\}.
\end{equation}
We first let $\eta\downarrow0$ with $K$ and all model parameters
fixed, and then compare the dependence of the resulting coefficients
on $K$.

\begin{theorem}[Separation from uniform sampling and sequential halving]
\label{thm:app-baseline-separation}
On the family~\eqref{eq:app-separation-family},
\begin{equation}
\label{eq:app-separation-budgets}
\begin{aligned}
    \lim_{\eta\downarrow0}
    \frac{\mathcal T_{\mathrm{Uniform}}(\eta;\mu)}
         {\log(1/\eta)}
    &=
    2K^3,
    \\
    \lim_{\eta\downarrow0}
    \frac{\mathcal T_{\mathrm{SH}}(\eta;\mu)}
         {\log(1/\eta)}
    &=
    \frac{2K^3(Q-1)}{K-1},
    \\
    \limsup_{\eta\downarrow0}
    \frac{\mathcal T_{\mathrm{HTG}}(\eta;\mu)}
         {\log(1/\eta)}
    &\le
    192(K^2+K-1).
\end{aligned}
\end{equation}
Consequently, in this order of limits, uniform sampling and the specified
SH allocation can require factors $\Omega(K)$ and $\Omega(\log K)$
more observations, respectively, than HTG.
\end{theorem}

\begin{proof}
\emph{Uniform sampling.}
There is one positive arm, with mean $\sigma/K$, so
$H_{\mathrm{pos}}(\mu)=K^3$.
Proposition~\ref{prop:app-uniform-sampling} gives
\begin{equation}
\label{eq:app-uniform-separation-exponent}
    \lim_{T\to\infty}
    -\frac1T
    \log
    \mathbb P_\mu^{\mathrm{Uniform}}
    \bigl(\mathrm{Succ}_T^c\bigr)
    =
    \frac{1}{2K^3}.
\end{equation}

\emph{Sequential halving.}
For all sufficiently large $T$, the SH branch is used.
Since $K=2^\ell$ and $Q-1=\ell$, its active-set sizes are
$k_r=K/2^{r-1}$.
Recall $n_r$ and $D$ from~\eqref{eq:app-sh-counts}, and set
$n_\star:=n_{Q-1}+D$, the count arm $1$ receives if it survives
every stage.

The rounding remainder satisfies
\[
    0\le D
    <
    \sum_{r=1}^{Q-1}k_r
    =
    2K-2.
\]
Therefore, for fixed $K$,
\begin{equation}
\label{eq:app-sh-count-limits}
    \frac{q_1}{T}
    \longrightarrow
    \frac{1}{K(Q-1)},
    \qquad
    \frac{n_\star}{T}
    \longrightarrow
    \frac{K-1}{K(Q-1)}.
\end{equation}

Define the deterministic-count failure probability
\[
    a_T
    :=
    \mathbb P_\mu(Z_{1,n_\star}<z)
    =
    \Phi\left(
        z\sqrt{1+\frac{\lambda}{n_\star}}
        -
        \frac{\sqrt{n_\star}}{K}
    \right).
\]
Gaussian tail asymptotics and~\eqref{eq:app-sh-count-limits} imply
\begin{equation}
\label{eq:app-sh-positive-tail-exponent}
    \lim_{T\to\infty}
    -\frac1T\log a_T
    =
    \frac{K-1}{2K^3(Q-1)}
    =:\alpha_K.
\end{equation}

We compare $a_T$ with the actual failure probability by bounding
elimination of arm $1$ and certification of a negative arm.

Let $E_T$ be the event that arm $1$ is eliminated.
If this happens at stage $r$, some negative arm $j$ has
$\overline X_{j,n_r}\ge\overline X_{1,n_r}$.
At the deterministic count $n_r$,
\[
    \overline X_{j,n_r}-\overline X_{1,n_r}
    \sim
    \mathcal N\left(
        -\sigma\left(1+\frac1K\right),
        \frac{2\sigma^2}{n_r}
    \right).
\]
A union bound over all negative arms and stages gives
\begin{equation}
\label{eq:app-sh-elimination-bound}
    \mathbb P_\mu^{\mathrm{SH}}(E_T)
    \le
    (K-1)(Q-1)
    \exp\left[
        -\frac{q_1(1+1/K)^2}{4}
    \right]
    =:d_T.
\end{equation}
This bound uses the potential observations at deterministic counts;
it does not condition on either arm being active.

Next, let $F_T$ be the event that some negative arm has a terminal
posterior score at least $z$.
For $j\ge2$ and a deterministic count $n$, the centered prior gives
\[
    Z_{j,n}\ge z
    \quad\Longrightarrow\quad
    W_{j,n}
    \ge
    n+z\sqrt{n+\lambda}
    \ge n.
\]
Hence $\mathbb P_\mu(Z_{j,n}\ge z)\le e^{-n/2}$.
Every terminal count belongs to $\mathcal C_T$, whose elements are
at least $q_1$. Thus
\begin{equation}
\label{eq:app-sh-negative-certificate-bound}
    \mathbb P_\mu^{\mathrm{SH}}(F_T)
    \le
    (K-1)Qe^{-q_1/2}
    =:b_T.
\end{equation}

On $E_T^c$, arm $1$ has terminal count $n_\star$.
Consequently,
\begin{equation}
\label{eq:app-sh-failure-sandwich}
    a_T-d_T-b_T
    \le
    \mathbb P_\mu^{\mathrm{SH}}
    \bigl(\mathrm{Succ}_T^c\bigr)
    \le
    a_T+d_T.
\end{equation}
The lower bound follows by requiring $Z_{1,n_\star}<z$ together
with $E_T^c\cap F_T^c$; the upper bound follows because, outside
$E_T$, terminal failure requires $Z_{1,n_\star}<z$.

By~\eqref{eq:app-sh-count-limits},
\[
    \lim_{T\to\infty}
    -\frac1T\log d_T
    =
    \frac{(1+1/K)^2}{4K(Q-1)},
    \qquad
    \lim_{T\to\infty}
    -\frac1T\log b_T
    =
    \frac{1}{2K(Q-1)}.
\]
Both exponents are strictly larger than $\alpha_K$ for every $K\ge2$.
Thus $d_T/a_T\to0$ and $b_T/a_T\to0$.
The sandwich~\eqref{eq:app-sh-failure-sandwich} implies
\begin{equation}
\label{eq:app-sh-separation-exponent}
    \lim_{T\to\infty}
    -\frac1T
    \log
    \mathbb P_\mu^{\mathrm{SH}}
    \bigl(\mathrm{Succ}_T^c\bigr)
    =
    \alpha_K
    =
    \frac{K-1}{2K^3(Q-1)}.
\end{equation}

\emph{Converting the baseline exponents into sample budgets.}
For any of the specified algorithms, the failure probability is strictly
positive at every finite budget. Indeed, under the centered prior, the
event that the first $T$ potential observations of every arm are negative
has positive probability and forces every sampled arm to have a negative
posterior score.

More generally, suppose $p_T>0$ for every finite $T$ and
$-T^{-1}\log p_T\to r>0$. For every $\varepsilon\in(0,r)$ and all
sufficiently large $T$,
\[
    e^{-(r+\varepsilon)T}
    \le p_T
    \le e^{-(r-\varepsilon)T}.
\]
The finitely many smaller budgets have failure probabilities bounded
away from zero. It follows that
\[
    \min\{T\ge K:p_T\le\eta\}
    \sim
    \frac{1}{r}\log\frac1\eta,
    \qquad \eta\downarrow0.
\]
This conclusion does not require $p_T$ to be monotone in $T$.
Applying it to~\eqref{eq:app-uniform-separation-exponent}
and~\eqref{eq:app-sh-separation-exponent} proves the first two
statements in~\eqref{eq:app-separation-budgets}.

\emph{HTG.}
On the family~\eqref{eq:app-separation-family},
\[
    \sum_{i=1}^{K}\frac{\sigma^2}{\mu_i^2}
    =
    K^2+K-1.
\]
The definition~\eqref{eq:H-HTG} therefore gives
$H_{\mathrm{HTG}}(\mu)\le K^2+K-1$.
For fixed $K$, the offset $T_0$ in
Theorem~\ref{thm:htg-main} is finite and independent of $\eta$.

Fix any $a>192H_{\mathrm{HTG}}(\mu)$ and choose
$T=\lceil a\log(1/\eta)\rceil$.
For sufficiently small $\eta$, Theorem~\ref{thm:htg-main} yields
\[
\begin{aligned}
    \mathbb P_\mu^{\mathrm{HTG}}
    \bigl(\mathrm{Succ}_T^c\bigr)
    &\le
    2KT
    \exp\left(
        \frac{T_0}{192H_{\mathrm{HTG}}(\mu)}
    \right)
    \eta^{a/(192H_{\mathrm{HTG}}(\mu))}
    \\
    &\le \eta.
\end{aligned}
\]
Here the exponent of $\eta$ is strictly larger than one, whereas
$T=O(\log(1/\eta))$.
Letting $a\downarrow192H_{\mathrm{HTG}}(\mu)$ gives
\[
    \limsup_{\eta\downarrow0}
    \frac{\mathcal T_{\mathrm{HTG}}(\eta;\mu)}
         {\log(1/\eta)}
    \le
    192H_{\mathrm{HTG}}(\mu)
    \le
    192(K^2+K-1).
\]
This proves the third statement.

Finally, these budget bounds imply
\[
\begin{aligned}
    \liminf_{\eta\downarrow0}
    \frac{\mathcal T_{\mathrm{Uniform}}(\eta;\mu)}
         {\mathcal T_{\mathrm{HTG}}(\eta;\mu)}
    &\ge
    \frac{K^3}{96(K^2+K-1)},
    \\
    \liminf_{\eta\downarrow0}
    \frac{\mathcal T_{\mathrm{SH}}(\eta;\mu)}
         {\mathcal T_{\mathrm{HTG}}(\eta;\mu)}
    &\ge
    \frac{K^3(Q-1)}
         {96(K-1)(K^2+K-1)}.
\end{aligned}
\]
The right-hand sides grow as $\Omega(K)$ and $\Omega(\log K)$,
respectively, as $K$ increases through powers of two.
\end{proof}

The two separations have different allocation explanations.
Uniform sampling devotes only a $1/K$ fraction of the budget to the unique
positive arm. For SH, even when this arm survives every stage, its terminal
sample count is asymptotic to $T(1-1/K)/(Q-1)$. Moreover, the probability of eliminating the positive arm is asymptotically
negligible compared with the probability that the surviving positive arm
fails to obtain a certificate. Thus, the additional $\log K$ cost of SH
arises from splitting the budget across $\Theta(\log K)$ elimination stages,
rather than from erroneous elimination.

The baseline asymptotics above are exact, whereas the HTG result is an
upper bound on its required budget.



\section{Flat-Prior Frequentist Interpretation}
\label{app:frequentist-interpretation}

We give the short derivation underlying
Section~\ref{sec:frequentist-interpretation}. Recall that
$\lambda=\sigma^2/\sigma_0^2$ and
\[
    Z_{i,n}
    =
    \frac{\lambda\mu_0+n\overline X_{i,n}}
         {\sigma\sqrt{n+\lambda}}.
\]
Letting $\sigma_0^2\to\infty$ with $\mu_0$ fixed gives
$\lambda\to0$, and hence
\[
    Z_{i,n}
    \longrightarrow
    \frac{\sqrt n\,\overline X_{i,n}}{\sigma}.
\]
This is the usual known-variance Gaussian test statistic for
\[
    H_{0,i}:\mu_i\le0
    \qquad\text{versus}\qquad
    H_{1,i}:\mu_i>0.
\]
Accordingly, define the one-sided $p$-value
\[
    p_{i,n}
    :=
    1-\Phi\left(
        \frac{\sqrt n\,\overline X_{i,n}}{\sigma}
    \right).
\]
Since the posterior positivity probability is $\Phi(Z_{i,n})$, the
flat-prior limit gives
\[
    \mathbb P(\mu_i>0\mid X_{i,1:n})
    =
    1-p_{i,n}.
\]

Because $\Phi$ is increasing, ranking arms by posterior $Z$-score is
equivalent to ranking them by increasing $p$-value. Thus HTG's halving
phase retains the active arms with the smallest current $p$-values, its
greedy phase samples an arm with the smallest current $p$-value, and its
terminal recommendation satisfies
\[
    J_T
    \in
    \arg\min_{i\in[K]}p_{i,N_i(T)}.
\]
Finally, since $z=\Phi^{-1}(1-\delta)$,
\[
    Z_{i,n}\ge z
    \quad\Longleftrightarrow\quad
    p_{i,n}\le\delta.
\]
Therefore,
\[
    \mathrm{Succ}_T
    =
    \left\{
        \min_{i\in[K]}p_{i,N_i(T)}
        \le\delta
    \right\}.
\]
Hence, in the flat-prior limit, HTG can be interpreted in terms
of one-sided Gaussian test statistics and $p$-values.

\section{Experimental Details and Additional Results}
\label{app:experiments}

We supplement Section~\ref{sec:experiments} with full parameter sweeps,
comparisons across the number of arms, and allocation trajectories.
All experiments compare HTG, uniform sampling, and posterior-based
sequential halving (SH), using observation variance $\sigma^2=1$,
working prior $\mathcal{N}(0,1)$, and credibility tolerance
$\delta=0.01$. Every posterior incorporates all observations collected
from the corresponding arm, and terminal recommendation maximizes the
posterior $Z$-score over all arms, including those removed during halving.
Success is evaluated after exactly $T$ observations; an earlier
posterior certificate does not terminate a run.

We use $10{,}000$ Monte Carlo repetitions per setting and budget.
Success-probability plots show pointwise $95\%$ Wilson confidence
intervals. The intervals for empirical budget thresholds and mean
allocation fractions are described separately below.
Throughout this appendix, the means are fixed within each setting:
the reported probabilities are fixed-instance success probabilities.

\subsection{Full parameter sweeps}
\label{app:exp-sweeps}

Figures~\ref{fig:app-negative-sweep}
and~\ref{fig:app-positive-sweep} extend the two instance families
in Section~\ref{sec:experiments} across the full tested parameter
ranges. We retain $K=32$ and positive mean $a=0.1$, and use the
same budget grid as in the main experiments: $16$ approximately
logarithmically spaced budgets from $576$ to $65{,}536$, rounded
to multiples of $K$. The value of the varied parameter is indicated
in each panel.

\paragraph{Negative competitors.}
Figure~\ref{fig:app-negative-sweep} considers
\[
    \mu=(a,-b,\ldots,-b),
    \qquad a=0.1.
\]
This family corresponds to the single-positive-arm example in
Section~\ref{sec:htg-interpretation}. Since $\sigma^2=1$,
\[
    H_{\mathrm{cap}}
    =
    \frac{1}{a^2}
    +
    \frac{K-1}{(a+b)^2},
    \qquad
    H_{\mathrm{pos}}
    =
    \frac{K}{a^2}.
\]
Increasing $b$ therefore reduces the separation component of
$H_{\mathrm{cap}}$ without changing the positive mean or
$H_{\mathrm{pos}}$. The sweep places the difficult-negative and
easy-negative settings from the main text within a broader range
of competitor means. In particular, the well-separated setting
exhibits a substantial advantage for HTG, consistent with its ability
to concentrate observations on the positive arm once the negative
competitors can be screened effectively.

\begin{figure*}[tp]
    \centering
    \includegraphics[
        width=0.90\textwidth,
        height=0.76\textheight,
        keepaspectratio
    ]{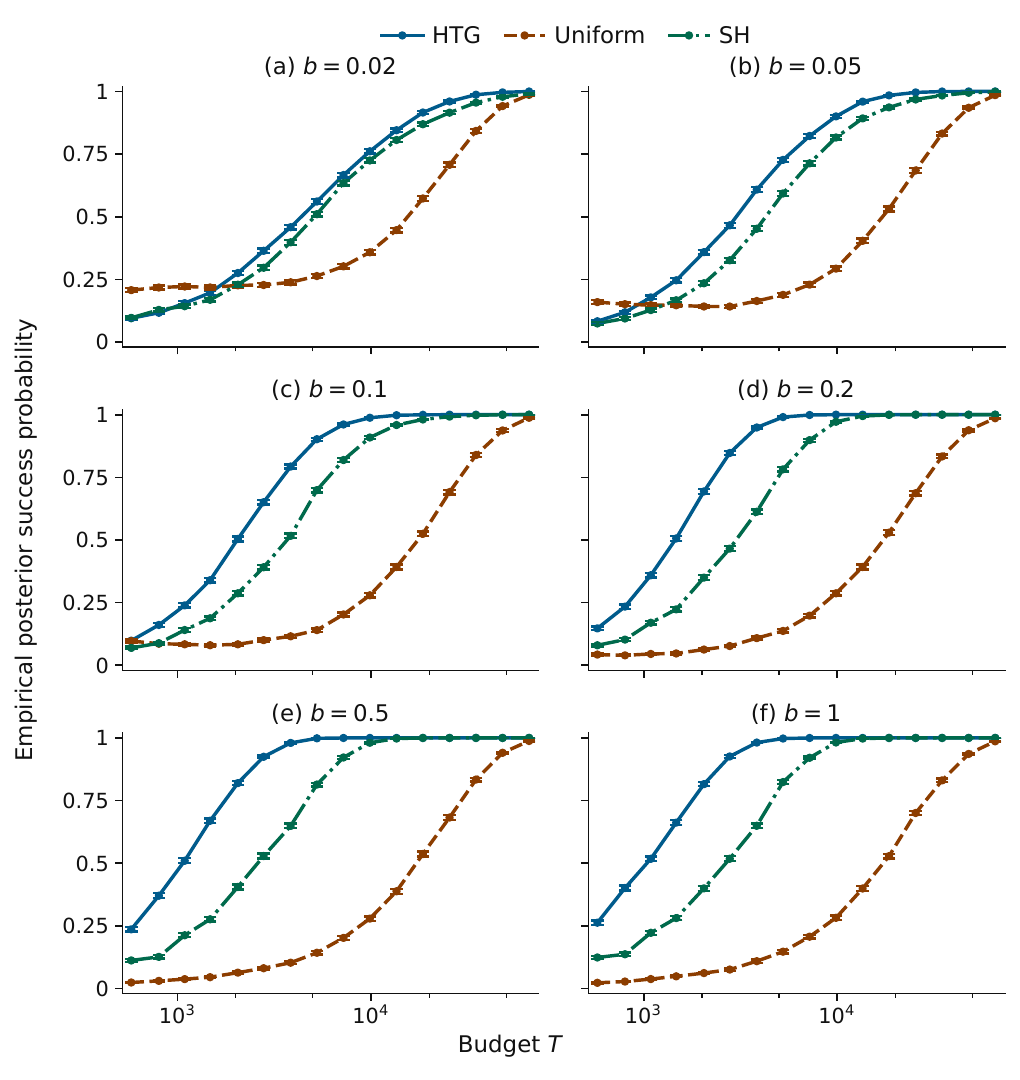}
    \caption{Full negative-competitor sweep.
    One arm has mean $0.1$, and the remaining $31$ arms have
    mean $-b$, with the value of $b$ indicated in each panel.
    Curves show terminal posterior success probability versus budget
    for HTG, uniform sampling, and SH.
    Error bars show pointwise $95\%$ Wilson Monte Carlo confidence
    intervals based on $10{,}000$ repetitions.}
    \label{fig:app-negative-sweep}
\end{figure*}

\paragraph{Certificates on negative recommendations.}
When $b$ is small, sampling noise can produce a posterior certificate
for an arm whose true mean is negative. Such an outcome is included
in the posterior success objective. More precisely,
\[
\begin{aligned}
    \mathbb{P}^{\mathcal A}_{\mu}(\mathrm{Succ}_T)
    ={}&
    \mathbb{P}^{\mathcal A}_{\mu}
    (\mathrm{Succ}_T,\mu_{J_T}>0)
    \\
    &+
    \mathbb{P}^{\mathcal A}_{\mu}
    (\mathrm{Succ}_T,\mu_{J_T}\le 0).
\end{aligned}
\]
This distinction helps interpret uniform sampling's initial lead
in the $b=0.02$ setting: a higher posterior success probability
need not correspond to a higher probability of certifying a truly
positive recommendation. The reported success curves retain both
terms, as required by the objective.

\paragraph{Multiple positive arms.}
Figure~\ref{fig:app-positive-sweep} considers
\[
    \mu=
    \bigl(
    \underbrace{a,\ldots,a}_{m\text{ arms}},
    \underbrace{-1,\ldots,-1}_{K-m\text{ arms}}
    \bigr),
    \qquad a=0.1.
\]
The sweep varies the number of positive arms without changing their
individual means. Here,
\[
    H_{\mathrm{pos}}
    =
    \frac{K}{ma^2}
    =
    \frac{3200}{m}.
\]
Additional positive arms provide more opportunities to obtain a
certificate, the same mechanism illustrated by the all-positive
example in Section~\ref{sec:htg-interpretation}.
Uniform sampling benefits substantially from this multiplicity,
and the main-text comparison between $m=4$ and $m=16$ shows a
narrowing of HTG's advantage over SH.
The full sweep documents how the methods compare across the
remaining tested values of $m$.

\begin{figure*}[tp]
    \centering
    \includegraphics[
        width=0.90\textwidth,
        height=0.76\textheight,
        keepaspectratio
    ]{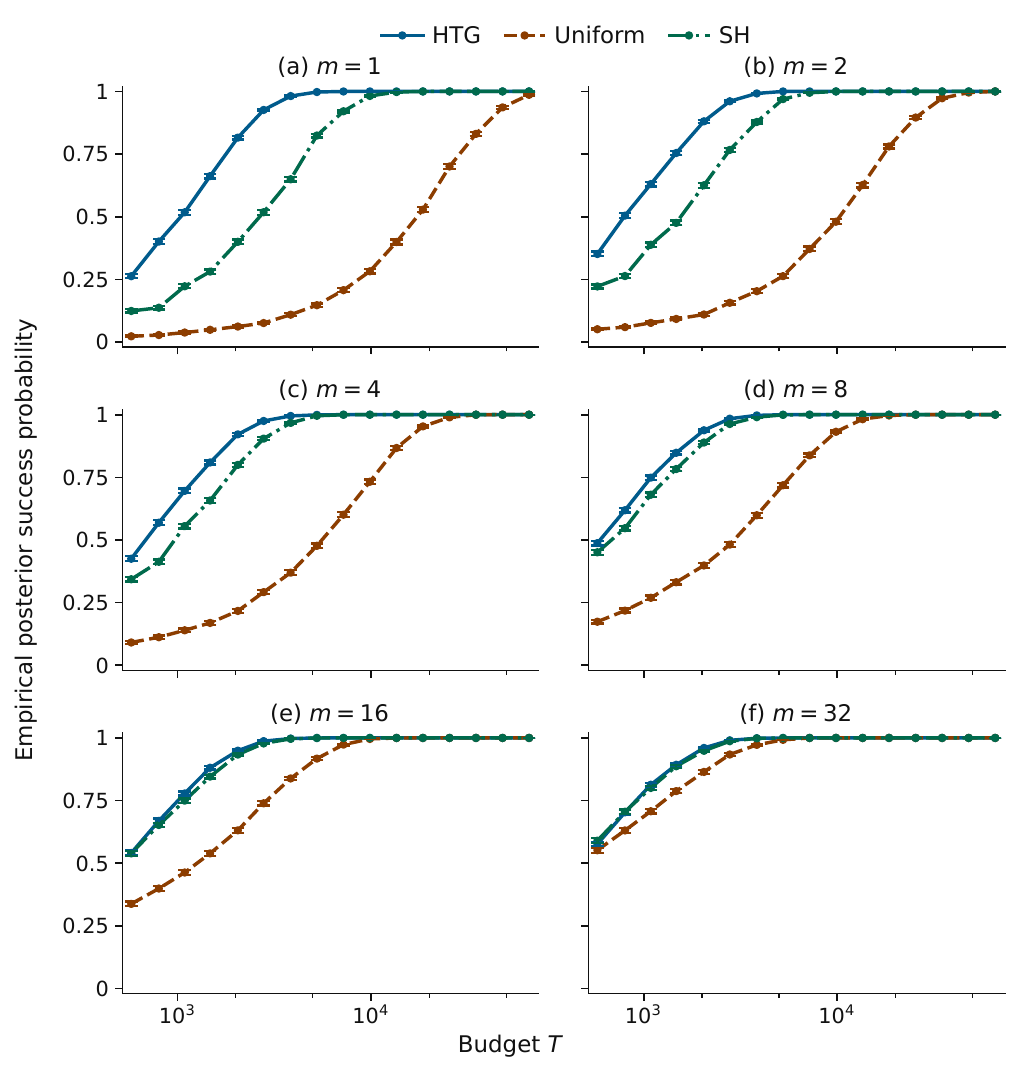}
    \caption{Full positive-arm-count sweep.
    Among $K=32$ arms, $m$ have mean $0.1$ and the remaining
    $K-m$ have mean $-1$, with $m$ indicated in each panel.
    Curves show terminal posterior success probability versus budget.
    Error bars show pointwise $95\%$ Wilson Monte Carlo confidence
    intervals based on $10{,}000$ repetitions.}
    \label{fig:app-positive-sweep}
\end{figure*}

\subsection{Scaling with the number of arms}
\label{app:exp-k-scaling}

We next examine how the budget required for successful certification
changes with $K$ on the instance family
\[
    \mu^{(K)}
    =
    (1/K,-1,\ldots,-1),
    \qquad
    K\in\{8,16,32,64\}.
\]
The positive mean decreases as the number of arms increases, while
the negative means remain fixed. This is the same family considered in
Appendix~\ref{app:baseline-separation}. In the high-confidence regime,
the analysis there gives the sample-complexity scalings
\[
    T_{\mathrm{HTG}}
    =
    O\!\left(K^2\log\frac{1}{\eta}\right),
    \qquad
    T_{\mathrm{SH}}
    =
    \Theta\!\left(K^2\log K\log\frac{1}{\eta}\right),
    \qquad
    T_{\mathrm{Uniform}}
    =
    \Theta\!\left(K^3\log\frac{1}{\eta}\right),
\]
where $1-\eta$ is the target posterior success probability. Thus,
suppressing the common dependence on $\eta$, the predicted dependence
on the number of arms is $K^2$, $K^2\log K$, and $K^3$, respectively.

Let $\widehat p_{\mathcal A,K}(T)$ denote the empirical terminal
success probability of method $\mathcal A$ at budget $T$, and let
$\mathcal G_{\mathcal A,K}$ be its tested budget grid. We report
\[
    \widehat T_{0.9}^{\mathcal A}(K)
    :=
    \min\left\{
        T\in\mathcal G_{\mathcal A,K}:
        \widehat p_{\mathcal A,K}(T)\ge 0.9
    \right\}.
\]
Thus, $\widehat T_{0.9}^{\mathcal A}(K)$ is the first tested budget
reaching $90\%$ empirical success.

Figure~\ref{fig:app-k-scaling} reports these thresholds together
with their normalized versions:
\[
    \frac{\widehat T_{0.9}^{\mathrm{HTG}}(K)}{K^2},
    \qquad
    \frac{\widehat T_{0.9}^{\mathrm{SH}}(K)}{K^2\log K},
    \qquad
    \frac{\widehat T_{0.9}^{\mathrm{Uniform}}(K)}{K^3}.
\]
These normalizations
correspond to the $K$-dependence predicted by the theoretical rates
above. Because the denominators differ across methods, the vertical
ordering in the normalized plot should not be interpreted as a direct
performance ranking.

HTG requires the smallest threshold budget at every tested value of
$K$, followed by SH and then uniform sampling. The normalized
thresholds range from $21.6$ to $27.3$ for HTG, from $19.1$ to $21.7$
for SH, and from $12.9$ to $13.0$ for uniform sampling. Their relatively
limited variation over the tested values of $K$ is consistent with the
predicted $K^2$, $K^2\log K$, and $K^3$ dependences, respectively.
The underlying success-versus-budget curves are provided in
Figure~\ref{fig:app-k-success-curves}.

\begin{figure*}[tp]
    \centering
    \includegraphics[width=0.49\textwidth]
        {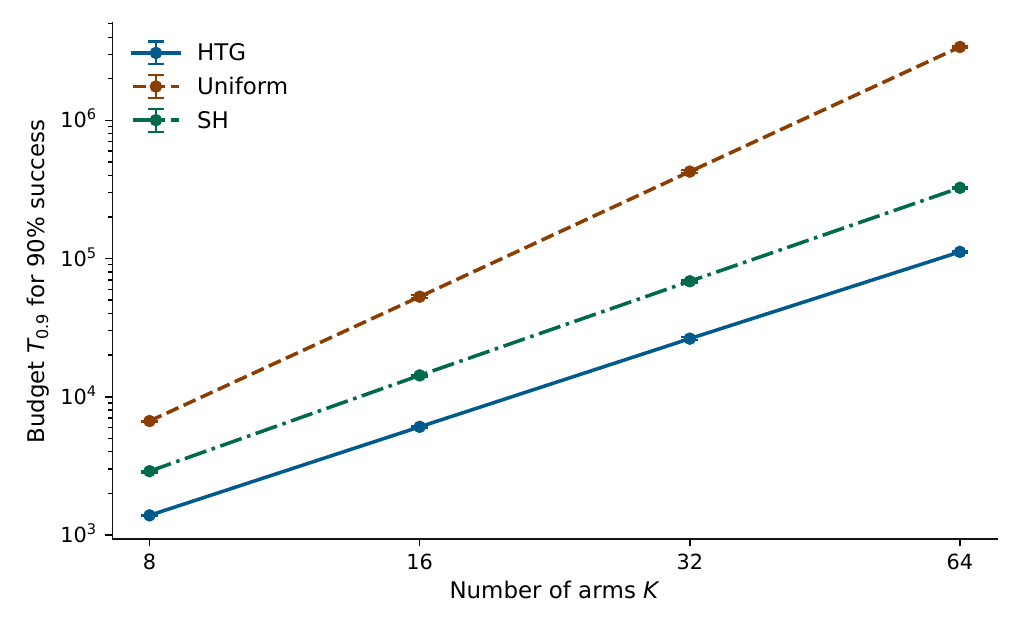}
    \hfill
    \includegraphics[width=0.49\textwidth]
        {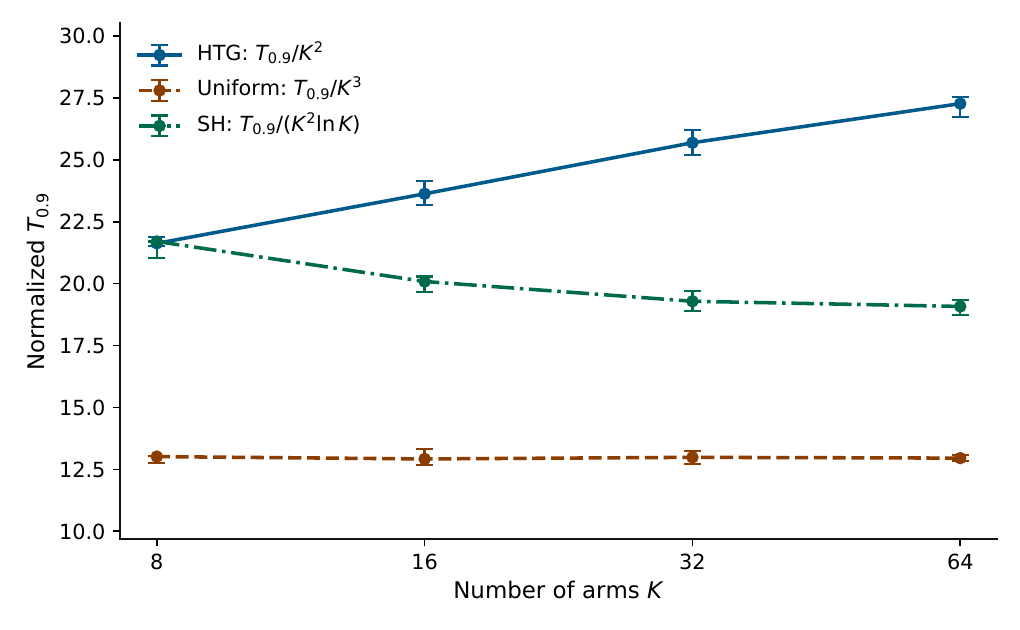}
    \caption{Scaling with the number of arms on
    $\mu^{(K)}=(1/K,-1,\ldots,-1)$.
    Left: first tested budget reaching $90\%$ empirical terminal
    success for each method.
    Right: the same thresholds divided by $K^2$ for HTG,
    $K^2\log K$ for SH, and $K^3$ for uniform sampling.
    The logarithm is natural.
    Error bars show $95\%$ bootstrap intervals.
    Thresholds are computed on the tested budget grids.}
    \label{fig:app-k-scaling}
\end{figure*}

\begin{figure*}[tp]
    \centering
    \includegraphics[
        width=0.90\textwidth,
        height=0.76\textheight,
        keepaspectratio
    ]{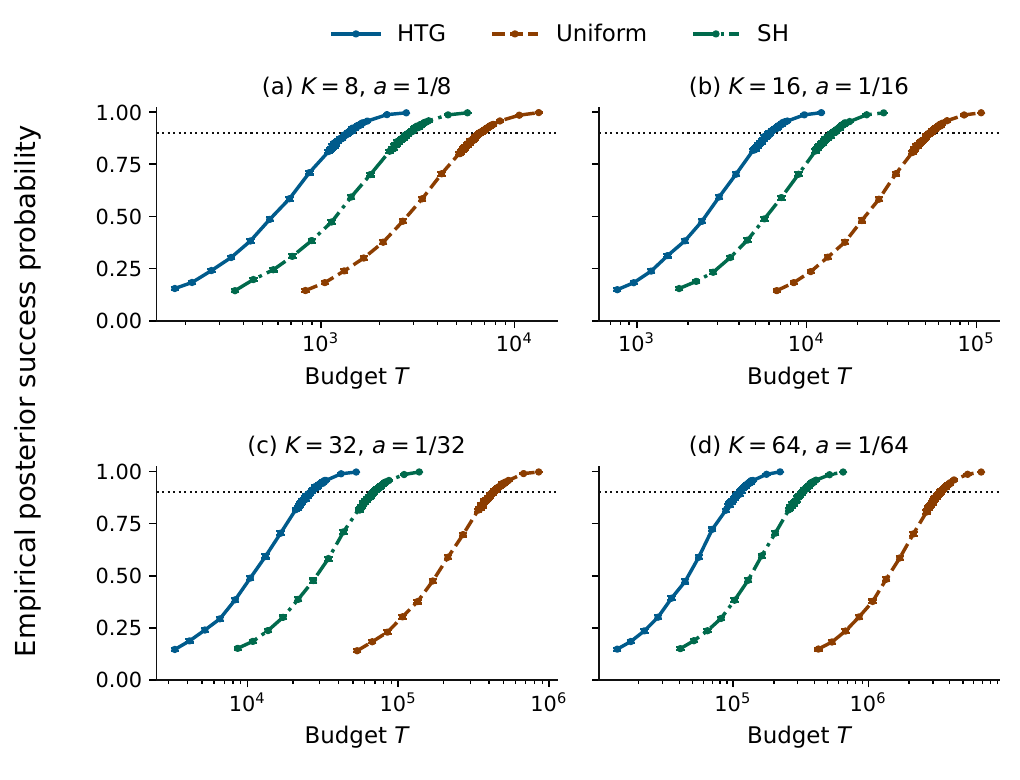}
    \caption{Terminal posterior success probability versus budget
    for the $K$-scaling study.
    Each setting uses $\mu^{(K)}=(1/K,-1,\ldots,-1)$, with
    $K\in\{8,16,32,64\}$.
    Error bars show pointwise $95\%$ Wilson Monte Carlo confidence
    intervals based on $10{,}000$ repetitions per method and budget.
    These curves underlie the threshold summaries in
    Figure~\ref{fig:app-k-scaling}.}
    \label{fig:app-k-success-curves}
\end{figure*}

\subsection{Allocation behavior}
\label{app:exp-allocation}

We examine allocation trajectories at budgets
$T=2{,}048$ and $T=7{,}200$ for the four representative settings
from the main experiments. The single-positive-arm settings use
$b\in\{0.02,1\}$, and the multiple-positive-arm settings use
$m\in\{4,16\}$. All have $K=32$ and positive mean $a=0.1$.

For each time $t\le T$, we consider the fraction of observations
allocated to positive arms collectively,
\[
    \frac{1}{t}\sum_{i:\mu_i>0}N_i(t),
\]
and the largest fraction allocated to any individual positive arm,
\[
    \frac{1}{t}\max_{i:\mu_i>0}N_i(t).
\]
With a single positive arm, both quantities reduce to $N_1(t)/t$.

The figures plot the Monte Carlo average of these allocation fractions
at each $t$. For the largest individual allocation, the maximum is
taken within each repetition before averaging, so the most-sampled
positive arm need not have the same label across repetitions.
Error bars show pointwise $95\%$ Monte Carlo confidence intervals
for the corresponding mean fractions. The dotted vertical line marks
the end of HTG's halving phase and the start of its greedy phase.

\paragraph{Negative competitors.}
Figure~\ref{fig:app-allocation-negative} shows how the separation of
negative competitors affects allocation to the single positive arm.
At $T=2{,}048$, HTG allocates an average of $51.5\%$ of observations
to the positive arm when $b=1$, compared with $15.4\%$ when $b=0.02$.
Uniform sampling allocates $1/K=3.125\%$ in either setting.
This allocation difference supports the interpretation of the
success curves: easier negative competitors allow HTG to direct
more of its budget toward collecting evidence for the positive arm.

\begin{figure*}[tp]
    \centering
    \includegraphics[
        width=0.90\textwidth,
        height=0.76\textheight,
        keepaspectratio
    ]{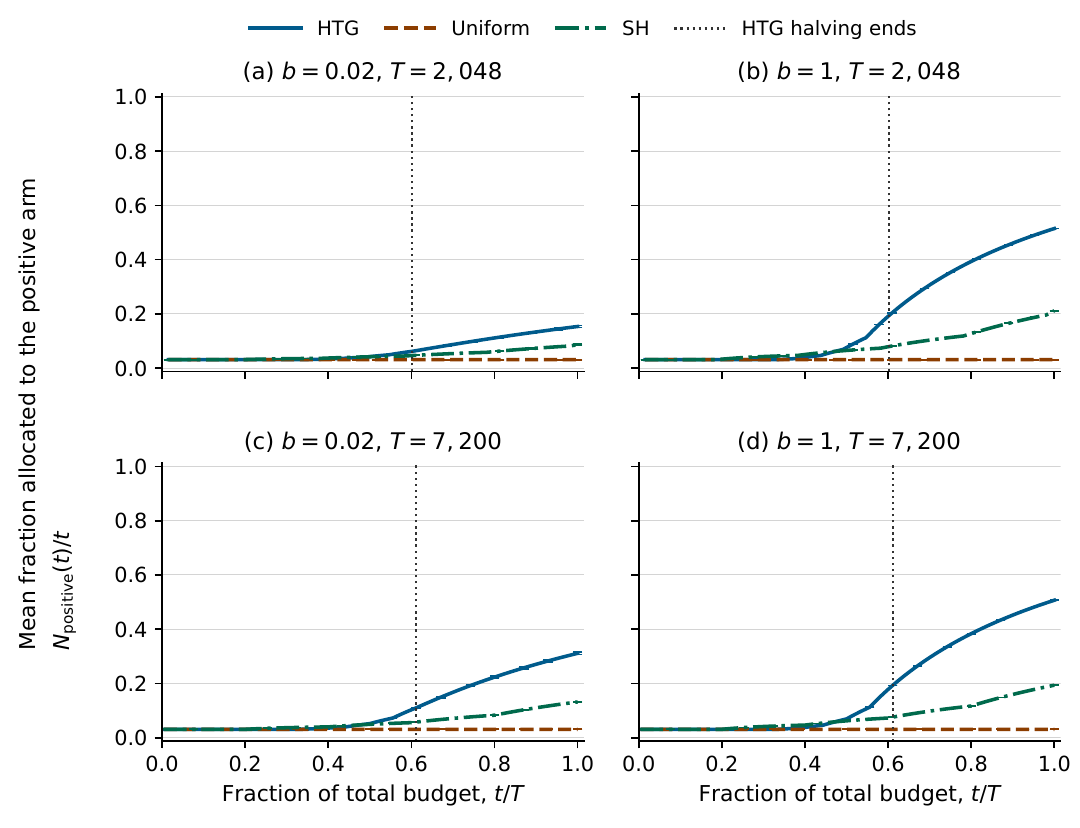}
    \caption{Allocation to the single positive arm for
    $\mu=(0.1,-b,\ldots,-b)$, with $K=32$ and
    $b\in\{0.02,1\}$.
    The two rows correspond to budgets $T=2{,}048$ and $T=7{,}200$.
    Curves show mean allocation fractions across $10{,}000$
    repetitions, with pointwise $95\%$ Monte Carlo confidence
    intervals. The dotted vertical line marks the end of HTG's
    halving phase.}
    \label{fig:app-allocation-negative}
\end{figure*}

\paragraph{Multiple positive arms.}
Figures~\ref{fig:app-allocation-positive-total}
and~\ref{fig:app-allocation-positive-max} separate total allocation
to positive arms from concentration on an individual positive arm.
These quantities can favor different methods.
For example, with $m=16$ and $T=2{,}048$, SH allocates more
observations to positive arms collectively than HTG:
$90.2\%$ versus $83.6\%$.
However, HTG's mean largest individual positive-arm allocation
is substantially greater: $48.4\%$ versus $21.2\%$ for SH.
Thus, the fraction of observations assigned to positive arms alone
does not explain the terminal success probabilities; the distribution
of those observations among the positive arms also matters.

\begin{figure*}[tp]
    \centering
    \includegraphics[
        width=0.90\textwidth,
        height=0.76\textheight,
        keepaspectratio
    ]{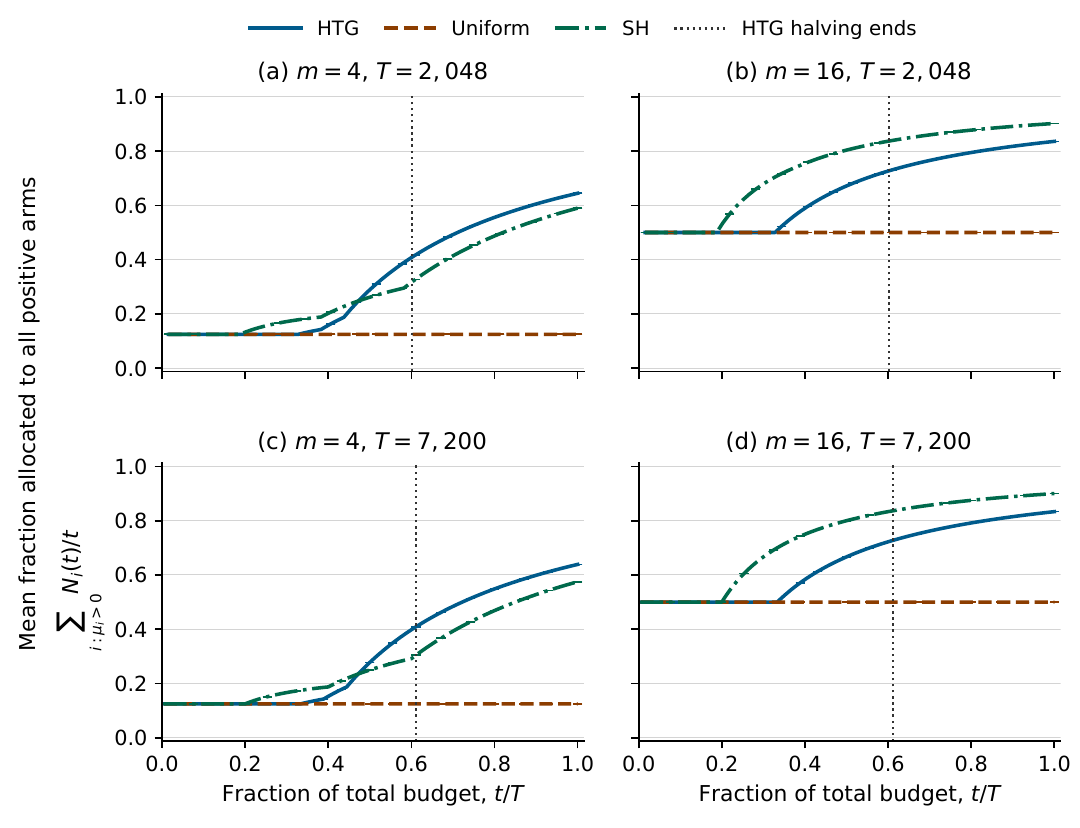}
    \caption{Total allocation to positive arms.
    Among $K=32$ arms, $m\in\{4,16\}$ have mean $0.1$ and the
    remaining arms have mean $-1$.
    The two rows correspond to budgets $T=2{,}048$ and $T=7{,}200$.
    Curves show mean total positive-arm allocation fractions across
    $10{,}000$ repetitions, with pointwise $95\%$ Monte Carlo
    confidence intervals. The dotted vertical line marks the end
    of HTG's halving phase.}
    \label{fig:app-allocation-positive-total}
\end{figure*}

\begin{figure*}[tp]
    \centering
    \includegraphics[
        width=0.90\textwidth,
        height=0.76\textheight,
        keepaspectratio
    ]{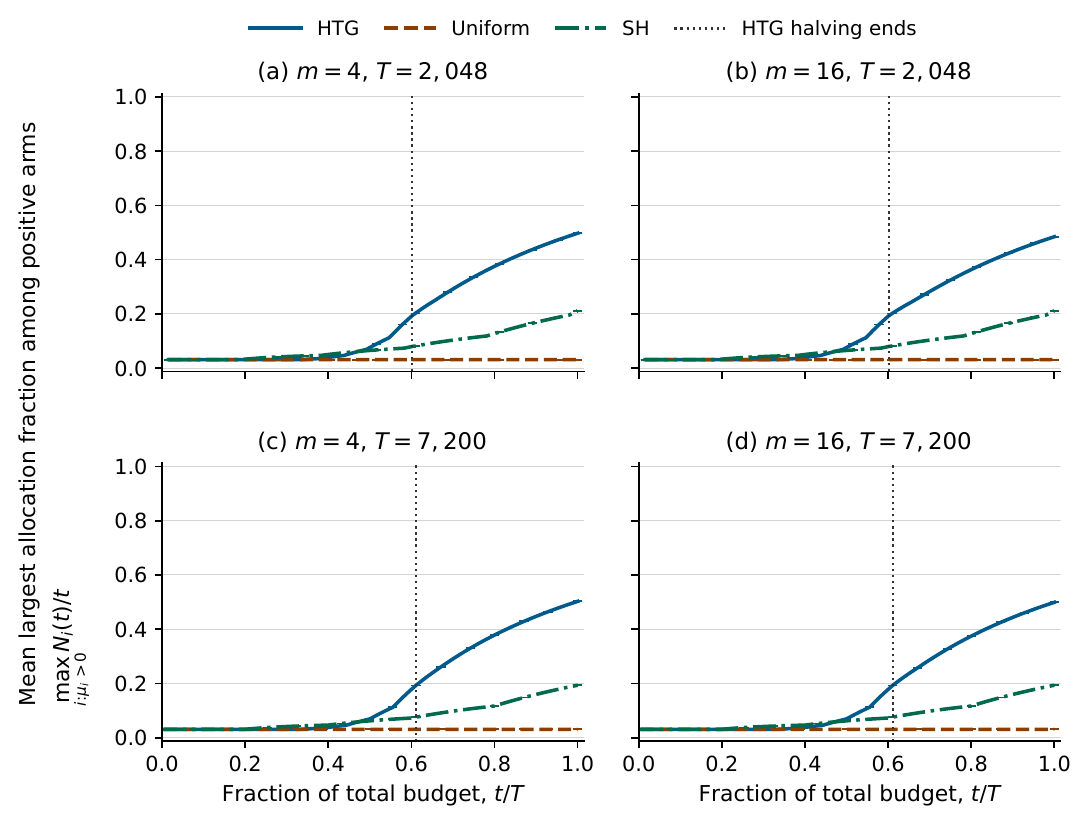}
    \caption{Largest individual positive-arm allocation.
    Settings and row organization match
    Figure~\ref{fig:app-allocation-positive-total}.
    The largest positive-arm allocation fraction is computed within
    each repetition and then averaged over $10{,}000$ repetitions.
    Error bars show pointwise $95\%$ Monte Carlo confidence intervals.
    The dotted vertical line marks the end of HTG's halving phase.}
    \label{fig:app-allocation-positive-max}
\end{figure*}



\end{document}